\documentclass[sn-standardnature]{sn-jnl}% Standard Nature Portfolio Reference Style
\usepackage{graphicx}
\usepackage{color}
\usepackage{algorithm}
\usepackage{amsmath}
\usepackage{amssymb}
\usepackage{comment}
\usepackage{graphicx}
\usepackage{bbm}
\usepackage{url}
\usepackage{xr}
\usepackage{subfigure}
\usepackage{framed}
\usepackage{float}
\usepackage{mdframed}
\usepackage{natbib}
\usepackage{listings}
\usepackage{graphicx}
\usepackage{xspace}
\usepackage{hyperref}

\usepackage{xcolor}
\usepackage{multirow}
\jyear{2021}%

\theoremstyle{thmstyleone}%
\newtheorem{theorem}{Theorem}%  meant for continuous numbers
\newtheorem{proposition}[theorem]{Proposition}% 
\newtheorem{lemma}{Lemma}

\theoremstyle{thmstyletwo}%

\theoremstyle{thmstylethree}%

\newcommand{\junkun}[1]{{\color{blue}{(Jun-Kun: #1)}}}
\newcommand{\kfir}[1]{{\color{red}{(Kfir: #1)}}}
\newcommand{\jake}[1]{{\color{green}{(Jake: #1)}}}

\newcommand{\g}{\gamma_{\K}}

\newcommand{\E}{\mathbb{E}}
\newcommand{\K}{\mathcal{K}}

\newcommand{\reals}{\mathbb{R}}

\newcommand{\argmin}{\mathop{\textnormal{argmin}}}
\newcommand{\argmax}{\mathop{\textnormal{argmax}}}

\newcommand{\balpha}{\boldsymbol{\alpha}}
\newcommand{\alg}{\text{OAlg}\xspace}

\newcommand{\BTL}{\textnormal{\textsc{FTL$^+$}}\xspace}
\newcommand{\FTL}{\textnormal{\textsc{FTL}}\xspace}
\newcommand{\FTRL}{\textnormal{\textsc{FTRL}}\xspace}

\newcommand{\BTRL}{\textnormal{\textsc{FTRL$^+$}}\xspace}

\newcommand{\OFTRL}{\textnormal{\textsc{OptimisticFTRL}}\xspace}
\newcommand{\OPTMD}{\textnormal{\textsc{OptimisticMD}}\xspace}
\newcommand{\BR}{\textnormal{\textsc{BestResp$^+$}}\xspace}
\newcommand{\OFTL}{\textnormal{\textsc{OptimisticFTL}}\xspace}
\newcommand{\MD}{\textnormal{\textsc{OMD$^+$}}\xspace}
\newcommand{\OMD}{\textnormal{\textsc{OMD}}\xspace}

\newcommand{\FW}{\textnormal{\textsc{Frank-Wolfe}}\xspace}

\newcommand{\xinit}{z_{\rm init}}

\newcommand{\regret}[1]{\balpha\textsc{-Reg}^{#1}}
\newcommand{\avgregret}[1]{\overline{\balpha\textsc{-Reg}}^{#1}}
\newcommand{\XX}{\mathcal{X}}
\newcommand{\YY}{\mathcal{Y}}
\newcommand{\ZZ}{\mathcal{Z}}
\newcommand{\pr}[1]{\left(#1\right)}
\newcommand{\yof}{y}
\newcommand{\yftl}{\hat{y}}
\newcommand{\xof}{\widetilde{x}}
\newcommand{\xav}{\bar{x}}
\newcommand{\yav}{\bar{y}}

\newcommand{\V}[1]{D_{#1}^{\phi}} %#1 pivot #2 distance generating function

\begin{document}

\title[No-Regret Dynamics in the Fenchel Game]{No-Regret Dynamics in the Fenchel Game: A Unified Framework for Algorithmic Convex Optimization
%\footnote{The work presented in this manuscript combines results from our previous publications, \cite{AW17,ALLW18,WA18}, with a significantly cleaner and more general analysis. 
%But we also include a number of new algorithms
% and convergence guarantees. The last paragraph in the introduction section provides more details. 
%}
}


\author[1]{\fnm{Jun-Kun} \sur{Wang}}\email{jun-kun.wang@yale.edu}

\author[2]{\fnm{Jacob} \sur{Abernethy}}\email{prof@gatech.edu}

\author[3]{\fnm{Kfir Y.} \sur{Levy}}\email{kfirylevy@technion.ac.il}

\affil[1]{\small \orgdiv{Department of Computer Science}, \orgname{Yale University}, \orgaddress{\city{New Haven}, \country{USA}}}

\affil[2]{\small \orgdiv{School of Computer Science}, \orgname{Georgia Institute of Technology}, \orgaddress{\city{Atlanta}, \country{USA}}}

\affil[3]{\small \orgdiv{Department of Electrical \& Computer Engineering}, \orgname{Technion - Israel Institute of Technology}, \orgaddress{\city{Haifa}, \country{Israel}}}


\abstract{We develop an algorithmic framework for solving convex optimization problems using no-regret game dynamics. By converting the problem of minimizing a convex function into an auxiliary problem of solving a min-max game in a sequential fashion, we can consider a range of strategies for each of the two-players who must select their actions one after the other. A common choice for these strategies are so-called no-regret learning algorithms, and we describe a number of such and prove bounds on their regret. We then show that many classical first-order methods for convex optimization---including average-iterate gradient descent, the Frank-Wolfe algorithm, Nesterov's acceleration methods, the accelerated proximal method---can be interpreted as special cases of our framework as long as each player makes the correct choice of no-regret strategy. Proving convergence rates in this framework becomes very straightforward, as they follow from plugging in the appropriate known regret bounds. Our framework also gives rise to a number of new first-order methods for special cases of convex optimization that were not previously known.}




\maketitle

\section{Introduction}
\label{intro}
%\citet{DFR18}
The main goal of this work is to develop a framework for solving convex optimization problems using iterative methods. Given a convex function $f : \reals^d \to \reals$, domain $\K \subseteq \reals^d$, and some tolerance $\epsilon > 0$, we want to find an approximate minimizer $w \in \K$ so that $f(w) - \min_{w' \in \K} f(w') \leq \epsilon$, using a sequence of oracle calls to $f$ and its derivatives. This foundational problem has received attention for decades, and researchers have designed numerous methods for this problem under a range of oracle query models and structural assumptions on $f(\cdot)$. What we aim to show in this paper is that a surprisingly large number of these methods---including those of Nesterov \cite{N83a,N83b,N05,N88,N04}, Frank and Wolfe \cite{frank1956algorithm}, and Beck and Teboulle \cite{BT09}---can all be described and analyzed through a single unified algorithmic framework, which we call the \emph{Fenchel game no-regret dynamics} (FGNRD). We show that several novel methods, with fast rates, emerge from FGNRD as well.

% One of the simplest and most well-studied tasks in optimization theory and practice is how to approximately find the minimum of a smooth convex function $f$ on some compact convex domain $K \subset \reals^d$ using a sequence of oracle calls to $f$ and its derivatives. Variations of this problem have been studied for over four decades, by considering restricted oracles on $f$ (0th-order, 1st-order, 2nd-order, etc.) and the type of access to the feasible set $K$ (membership, separating hyperplane, linear optimization, etc.). In many cases we have efficient algorithms with tight upper bounds on their iteration complexity.

Let us give a short overview before laying out the FGNRD framework more precisely. A family of tools, largely developed by researchers in theoretical machine learning, consider the problem of sequential prediction and decision making in non-stochastic environments, often called \emph{adversarial online learning}. This online learning setting has found numerous applications in several fields beyond machine learning---finance, for example, as well as statistics---but it has also emerged as a surprisingly useful tool in game theory. What we call \emph{no-regret online learning algorithms} are particularly well-suited for computing equilibria in two-player zero-sum games, as well as solving saddle point problems more broadly. If each agent employs a no-regret online learning algorithm to choose their action at each of a sequence of rounds, it can be shown that the agents' choices will converge to a saddle point, and at a rate that depends on their choice of learning algorithm. Thus, if we are able to simulate the two agents' sequential strategies, where each aims to minimize the ``regret'' of their chosen actions, then what emerges from the resulting \emph{no-regret dynamics} (NRD) can be implemented explicitly as an algorithm for solving min-max problems. 

Our goal is to demonstrate how
%How does 
NRD helps develop and analyze methods for minimizing a closed convex function $f$, i.e. solving $\min_{w\in\K}f(w)$, by solving a particular saddle-point problem
%. In particular, we will apply NRD to solve a particular saddle-point problem
%}.
% \kfir{Maybe explicitly state that our goal is to solve $\min_{x\in\K}f(x)$} What is our main focus in the present paper is a particular game of interest which we call the \emph{Fenchel game}: from $f$ 
%we can construct 
%Define
%with its 
whose two-input ``payoff'' function $g : \reals^d \times \reals^d \to \reals$ is defined by
\[
  g(x,y) := \langle x, y \rangle - f^*(y),
\]
where $f^*(\cdot)$ is the conjugate of $f(\cdot)$.
%Then, 
We view this as a game in the sense that if one player selects an action $x$ and a second player selects action $y$, then $g(x,y)$ is the former's ``cost'' and the latter's ``gain'' associated to their decisions. If the two players continue to update their decisions sequentially, first choosing $x_1$ and $y_1$ then $x_2$ and $y_2$, etc., and each player relies on some no-regret algorithm for this purpose, then one can show that the time-averaged iterates $\bar x, \bar y$ form an approximate equilibrium of the Fenchel game---that is, $g(\bar x, y') - \epsilon \leq g(\bar x, \bar y) \leq g(x', \bar y) + \epsilon$ for any alternative $x',y'$. But indeed, this approximate equilibrium brings us right back to where we started, since using the construction of the Fenchel game, it is easy to show that $w \leftarrow \bar x$ satisfies $f(w) - \min_{w' \in \K} f(w') \leq \epsilon$. The approximation factor $\epsilon$ is important, and we will see that it depends upon the number of iterations of the dynamic and the players' strategies.

What FGNRD gives us is a recipe book for constructing and analyzing iterative algorithms for convex optimization. To simulate a dynamic we still need to make particular choices as for both players' strategies and analyze their performance. We begin in Section~\ref{sec:onlinelearning} by giving a brief overview of tools from adversarial online learning, and we introduce a handful of simple online learning algorithms, including variants of FollowTheLeader and OnlineMirrorDescent, and prove bounds on the \emph{weighted} regret---we generalize slightly the notion of regret by introducing weights $\alpha_t > 0$ for each round. We prove a key result that relates the error $\epsilon$ of the approximate equilibrium pair $\bar x, \bar y$, which are the weighted-average of the iterates of the two players, to the weighted regret of the players' strategies. In Section~\ref{sec:ExistingAlgs} we show how several algorithms, including Frank-Wolfe's method \cite{frank1956algorithm}, several variants of Nesterov Accelerated Gradient Descent \cite{N83a,N83b,N05,N88,N04}, and the accelerated proximal method \cite{BT09}, are all special cases of the FGNRD framework, all with special choices of the learning algorithms for the $x$- and $y$-player, and the weights $\alpha_t$; see Table~\ref{tab:summary_exist} for a summary of these recipes. In addition we provide several new algorithms using FGNRD in Section~\ref{sec:NewAlgs}, summarized in Table~\ref{tab:summary_new}.

The work presented in this manuscript combines results from our previous publications, \cite{AW17,ALLW18,WA18}, with a significantly cleaner and more general analysis. 
But we also include a number of new algorithms
 and convergence guarantees, described in 
Algorithm~\ref{alg:newStoFW}, Algorithm~\ref{alg:GD}, Algorithm~\ref{alg:extra}, Algorithm~\ref{alg:gaugeFW2}, Algorithm~\ref{alg:GD2}, and Algorithm~\ref{alg:OPTACC}, and 
Theorem~\ref{thm:equivStoFW}, Theorem~\ref{thm:GD}, Theorem~\ref{thm:GD02}, Theorem~\ref{thm:extra},  Theorem~\ref{thm:gaugeFW-linear}, Theorem~\ref{thm:GD2}, Theorem~\ref{thm:OPTACC}. 
Moreover, for a family of online learning strategies listed in Algorithm~\ref{alg:oco_batch_family}, we develop a meta lemma in this work, Lemma~\ref{regret:Opt-FTRL}, such that the guarantee of each online algorithm is obtained by an easy application of the meta lemma.

%\kfir{The table can benefit from a description: say: $T$ denotes the total number of iterations, $\alpha_t$ are the weights which set the emphasis on iteration $t$, and the lats rows on the table set the specific strategies of the players in the FGNRD} 
%\kfir{ We can also put a link between the different strategies in the table to the appropriate subsections in Section 3}

%\clearpage

\section{Preliminaries}

\subsection{Convex Analysis}

We summarize some results in convex analysis that will be used in this paper. We also refer the readers to some excellent textbooks
(e.g. \cite{BT01,HL93,R96,N04,B04,BL06}).

\begin{table}[ht!]
\caption{Summary of recovering existing optimization algorithms from \emph{Fenchel Game}.
Here $T$ denotes the total number of iterations, 
$\alpha_t$ are the weights which set the emphasis on iteration $t$, the last two columns on the table indicate the specific strategies of the players in the FGNRD.}
\label{tab:summary_exist}       % Give a unique label
\centering
\footnotesize
\begin{tabular}{|c|c|c|c|c|} \hline
\multicolumn{5}{c}{ \shortstack{ Non-smooth convex optimization: $\min_{w \in \K} f(w)$
($G := \max_{w \in \K} \| \partial f(w) \|^{2}$ and $R:=\|w_0-w^*\|$)
 \\ as a game $g(x,y):= \langle x , y \rangle - f^*(y)$.} }  \\ \hline
Algorithm   &  rate &  weights &  $y$-player &  $x$-player \\ \hline
\shortstack{
Gradient Descent  \\ with averaging }
& 
\shortstack{ Thm.~\ref{thm:GD} \\
$O\left(\frac{GR}{\sqrt{T}} \right)$} & $\alpha_t=1$ & \shortstack{ Sec.~\ref{section:BR}\\ \BR } & \shortstack{ Sec.~\ref{section:OMD}\\ \OMD } \\ \hline 
\shortstack{
Cumulative \\ Gradient Descent }
& 
\shortstack{ Thm.~\ref{thm:GD2} \\
$O\left(\frac{GR}{\sqrt{T}} \right)$} & $\alpha_t=1$ & \shortstack{ Sec.~\ref{section:BTL}\\ \BTL } & \shortstack{ Sec.~\ref{section:OMD}\\ \OMD } \\ \hline \hline
\multicolumn{5}{c}{ \shortstack{ $L$-smooth convex optimization: $\min_{w\in \K} f(w)$ \\ as a game $g(x,y):= \langle x , y \rangle - f^*(y)$.} }  \\ \hline
Algorithm   &  rate &  weights &  $y$-player &  $x$-player \\ \hline
%\shortstack{
%Frank-Wolfe  \\ method \cite{frank1956algorithm}}
%& 
%\shortstack{ Prop.~\ref{thm:equivFW} and~\ref{thm:fwconvergence}   \\
%$O(\frac{L \log T}{T})$ } & $\alpha_t=1$ & \shortstack{ Sec.~\ref{section:FTL}\\ \FTL } & \shortstack{ Sec.~\ref{section:BR}\\ \BR }  \\ \hline
\shortstack{
Frank-Wolfe  \\ method \cite{frank1956algorithm}}
& 
\shortstack{Prop.~\ref{thm:equivFW} and~\ref{thm:fwconvergence}   \\
$O\left(\frac{L }{T}\right)$ } & $\alpha_t=t$ & \shortstack{ Sec.~\ref{section:FTL}\\ \FTL } & \shortstack{ Sec.~\ref{section:BR}\\ \BR } \\ \hline
\shortstack{
Linear rate \\ FW \cite{LP66} }
& 
\shortstack{ Thm.~\ref{thm:linearFW} \\
$O\left(\exp(- \frac{\lambda T}{L} )\right)$} & $\alpha_t=\frac{1}{\| \ell_t(x_t) \|^2}$ & \shortstack{ Sec.~\ref{section:FTL}\\ \FTL }  &  \shortstack{ Sec.~\ref{section:BR}\\ \BR } \\ \hline
%\shortstack{
%Nesterov's  \\ method \cite{N83b}}
%& 
%\shortstack{ Thm.~\ref{thm:metaAcc} and ~\ref{thm:accUn} \\
%$O(\frac{L}{T^2})$} & $\alpha_t = t$ & \shortstack{ Sec.~\ref{section:OFTL}\\ \OFTL } & \shortstack{ Sec.~\ref{section:BTRL} and ~\ref{section:MD}\\  \BTRL/\MD}
%\junkun{I am referring to the unconstrained version}
%\\ \hline
\shortstack{
Gradient Descent  \\ with averaging }
& 
\shortstack{ Thm.~\ref{thm:GD02} \\
$O\left(\frac{L}{T} \right)$} & $\alpha_t=1$ & \shortstack{ Sec.~\ref{section:BR}\\ \BR } & \shortstack{ Sec.~\ref{section:OMD}\\ \OMD } \\ \hline 
\shortstack{
Single-call \\ extra-gradient \\ with averaging }
& \shortstack{Thm.~\ref{thm:extra}\\
$O\left(\frac{L}{T}\right)$} & $\alpha_t=1$ & \shortstack{ Sec.~\ref{section:BR}\\ \BR } & \shortstack{ Sec.~\ref{section:OPTMD} \\ \OPTMD } \\ \hline 
\shortstack{
Nesterov's  \\  $1$-memory \\ method \cite{N88}}
& \shortstack{ Thm.~\ref{thm:metaAcc}  \\
$O\left(\frac{L}{T^2}\right)$ }& $\alpha_t=t$ & \shortstack{ Sec.~\ref{section:OFTL}\\ \OFTL } &  \shortstack{ Sec.~\ref{section:MD}\\ \MD } \\ \hline
\shortstack{
Nesterov's  \\  $\infty$-memory \\ method \cite{N05}}
& \shortstack{ Thm.~\ref{thm:smoothAcc}\\
$O\left(\frac{L}{T^2}\right)$ } & $\alpha_t=t$ & \shortstack{ Sec.~\ref{section:OFTL}\\ \OFTL } & \shortstack{ Sec.~\ref{section:BTRL}\\ \BTRL } \\ \hline
\shortstack{
Nesterov's  \\ first acceleration \\ method \cite{N83b}}
& 
\shortstack{ Thm.~\ref{thm:metaAcc} \\
$O\left(\frac{L}{T^2}\right)$} & $\alpha_t = t$ & \shortstack{ Sec.~\ref{section:OFTL}\\ \OFTL } & \shortstack{ Sec.~\ref{section:MD}\\  \MD with \\ $\phi(x)=\frac{1}{2} \| x \|^2_2$} \\ \hline
%\shortstack{
%Heavy Ball  \\ method \cite{P64} }
%& 
%\shortstack{ Thm.~\ref{thm:Heavy} \\
%$O\left(\frac{L}{T}\right)$} & $\alpha_t=t$ & \shortstack{ Sec.~\ref{section:FTL}\\ \FTL } & \shortstack{ Sec.~\ref{section:BTRL}\\ \BTRL } \\ \hline \hline
%\multicolumn{5}{c}{ \shortstack{ Non-smooth convex optimization: $\min_w f(w)$ \\ as a game $g(x,y):= \langle x , y \rangle - f^*(y)$. } }  \\ \hline \hline
%Algorithm   &  rate &  weights &  $y$-player &  $x$-player \\ \hline
%\shortstack{Smoothed \\ FW \cite{L13} }
%& 
%$O(\frac{1}{\sqrt{T}})$ & $\alpha_t=1$ & \shortstack{ Sec.~\ref{section:FTPL}\\ %FTPL }  &  \shortstack{ Sec.~\ref{section:BR}\\ \BR } \\ \hline \hline
\multicolumn{5}{c}{ \shortstack{ Composite optimization: $\min_{w\in \reals^d} f(w) + \psi(w)$, where $\psi(\cdot)$ is possibly non-differentiable, \\ as a game $g(x,y):= \langle x , y \rangle - f^*(y) + \psi(x)$.} }  \\ \hline \hline
Algorithm   &  rate &  weights &  $y$-player &  $x$-player \\ \hline
\shortstack{
Accelerated  \\  proximal \\  method \cite{BT09}}
& 
\shortstack{ Thm.~\ref{thm:proximal} \\
$O\left(\frac{L}{T^2}\right)$ } & $\alpha_t=t$ & \shortstack{ Sec.~\ref{section:OFTL}\\ \OFTL } &  \shortstack{ Sec.~\ref{section:MD}\\ \MD } \\ \hline \hline
\multicolumn{5}{c}{\shortstack{ $L$-smooth and $\mu$ strongly convex optimization: $\min_{w \in \K} f(w)$ \\ as a game $g(x,y):= \langle x , y \rangle - \tilde{f}^*(y) + \frac{\mu \| x \|^2}{2}$, where $\tilde{f}(\cdot):= f(\cdot) - \frac{\mu}{2}\| \cdot \|^2 $}. }  \\ \hline \hline
Algorithm   &  rate &  $\alpha_t$ & \shortstack{ $y$-player} &  \shortstack{$x$-player} \\ \hline
\shortstack{
Nesterov's  \\ method \cite{N04} }
& 
\shortstack{ Thm.~\ref{thm:acc_linear2} \\
$O\left(\exp\left( - \sqrt{ \frac{\mu}{L} } T \right)\right)$ } & $\alpha_t \propto \exp(t)$ & \shortstack{ Sec.~\ref{section:OFTL}\\ \OFTL } &  \shortstack{ Sec.~\ref{section:BTRL}\\ \BTRL } \\ \hline \hline
%  \\   \hline \hline
\multicolumn{5}{c}{ \shortstack{ $L$-smooth convex finite-sum optimization: $\min_{w \in \K} f(w):= \sum_{i=1}^n f_i(w)$ \\ as a game $g(x,y):= \langle x , y \rangle - f^*(y)$ }}  \\ \hline
Algorithm   & rate  &  weights &  $y$-player's alg. &  $x$-player's alg. \\ \hline
\shortstack{
Incremental \\FW \cite{Netal20} }
& \shortstack{Thm.~\ref{thm:equivStoFW} \\
$O\left(\frac{n}{T}\right)$ } & $\alpha_t=t$ & \shortstack{ Appendix.~\ref{app:equivStoFW} \\ Lazy FTL } & \shortstack{ Sec.~\ref{section:BR} \\ \BR}  \\ \hline
\end{tabular}
\end{table}

\begin{table}[t]
\caption{Summary of \emph{new} optimization algorithms from \emph{Fenchel Game}.
Here $T$ denotes the total number of iterations, 
$\alpha_t$ are the weights which set the emphasis on iteration $t$, the last two columns on the table indicate the specific strategies of the players in the FGNRD.
}
\label{tab:summary_new}       % Give a unique label
\footnotesize
\begin{tabular}{|c|c|c|c|c|} \hline
\multicolumn{5}{c}{ \shortstack{ Non-smooth convex optimization: $\min_{w \in \K} f(w)$, where $\K$ is a $\lambda$-strongly convex set
 \\ as a game $g(x,y):= \langle x , y \rangle - f^*(y)$, 
 \\ Assume that the norm of cumulative gradient does not vanish, $\| \frac{1}{t} \sum_{s=1}^t \partial f( x_s ) \| \geq \rho$.} }  \\ \hline
Algorithm   &  rate &  weights &  $y$-player's alg. &  $x$-player's alg. \\ \hline
\shortstack{
Boundary \\ FW }
& 
\shortstack{
Thm.~\ref{thm:equiv2}\\
$O\left(\frac{1}{\lambda \rho T}\right)$ }& $\alpha_t= 1$ & \shortstack{ Sec.~\ref{section:FTL}\\ \FTL } &  \shortstack{ Sec.~\ref{section:BR}\\ \BR } \\ \hline \hline
\multicolumn{5}{c}{ \shortstack{ $L$-smooth convex optimization: $\min_{w \in \K} f(w)$, 
\\ where $\K$ is a $\lambda$-strongly convex set that is centrally symmetric and contains the origin,
 \\ as a game $g(x,y):= \langle x , y \rangle - f^*(y)$ }}  \\ \hline
Algorithm   &  rate &  weights &  $y$-player's alg. &  $x$-player's alg. \\ \hline
\shortstack{
Gauge \\ FW }
& \shortstack{Thm.~\ref{thm:gaugeFW}\\
$O\left(\frac{L}{\lambda T^2}\right)$} & $\alpha_t=t$ & \shortstack{ Sec.~\ref{section:OFTL}\\ \OFTL } & \shortstack{ Sec.~\ref{section:BTRL} \\ \BTRL with \\ gauge function } \\ \hline \hline
\multicolumn{5}{c}{ \shortstack{ $L$-smooth and $\mu$-strongly convex optimization: $\min_{w \in \K} f(w)$, \\ where $\K$ is a $\lambda$-strongly convex set that is centrally symmetric and contains the origin, \\
as a game $g(x,y):= \langle x , y \rangle - \tilde{f}^*(y) + \frac{\mu}{\lambda} \g^2(x)$, where $\tilde{f}(\cdot):= f(\cdot) - \frac{\mu}{\lambda} \g^2(x) $}}  \\ \hline
Algorithm   &  rate &  weights &  $y$-player's alg. &  $x$-player's alg. \\ \hline
\shortstack{
Gauge \\ FW }
& \shortstack{Thm.~\ref{thm:gaugeFW-linear}\\
$O\left(\exp\left( - \sqrt{ \frac{\mu}{L} } T \right)\right)$} & $\alpha_t \propto \exp(t)$ & \shortstack{ Sec.~\ref{section:OFTL}\\ \OFTL } & \shortstack{ Sec.~\ref{section:BTRL} \\  \BTRL with \\ gauge function } \\ \hline \hline
\multicolumn{5}{c}{ \shortstack{ $L$-smooth convex optimization: $\min_{w \in \K} f(w)$ \\ as a game $g(x,y):= \langle x , y \rangle - f^*(y)$ }}  \\ \hline
Algorithm   &  rate &  weights &  $y$-player's alg. &  $x$-player's alg. \\ \hline
\shortstack{
Optimistic \\ Mirror Descent \\ with averaging }
& \shortstack{Thm.~\ref{thm:OPTACC}\\
$O\left(\frac{L}{T^2}\right)$} & $\alpha_t=t$ & \shortstack{ Sec.~\ref{section:BTL}\\ \BTL } & \shortstack{ Sec.~\ref{section:OPTMD} \\ \OPTMD } \\ \hline 
\end{tabular}
\end{table}

\paragraph{Smoothness and strong convexity}

%\kfir{change $u,v$ to $x,z$ to in the primal space e.g., $\| \nabla f(x) - \nabla f(z) \|_* \leq L \| x - z\|$, and leave $y$ for the dual space}

A function $f(\cdot)$ on $\reals^d$ is $L$-smooth with respect to a norm $\| \cdot \|$ if $f(\cdot)$ is everywhere differentiable and
it has Lipschitz continuous gradient
$\| \nabla f(x) - \nabla f(z) \|_* \leq L \| x - z\|$,  
%\kfir{change $u,v$ to $x,z$}
where $\| \cdot \|_{*}$ denotes the dual norm.
A function $f(\cdot)$ is $\mu$-strongly convex  w.r.t. a norm $\| \cdot \|$ if the domain of $f(\cdot)$ is convex and that
$f( \theta x + (1-\theta) z ) \leq \theta f(x) + (1-\theta) f(z)  - \frac{\mu}{2} \theta (1-\theta) \| x- z \|^2$ for all $x,z \in \text{dom}(f)$ and $\theta \in [0,1]$.
If a function is $\mu$-strongly convex, then
$f(z) \geq f(x) + f_x^\top (z-x) + \frac{\mu}{2} \| z - x \|^2$ for 
all $x,z \in \text{dom}(f)$,
where  $f_{x}$ denotes a subgradient of $f$ at $x$.
%We also assume that the optimal solution of $x^* := \argmin_{x \in \K} f(x)$ has finite norm.
The subdifferential $\partial f(x)$ is the set of all subgradients of $f$ at $x$,
i.e., $\partial f(x) = \{ f_x: f(z) \geq f(x) + f_x^\top( z -x ), \forall z\}$.
When $f$ is differentiable, we have $\partial f(x) = \{ \nabla f(x) \}$.

In Section~\ref{sec:ExistingAlgs} and Section~\ref{sec:NewAlgs},
we will cover several optimization algorithms. For those that are labeled for minimizing non-smooth functions, their guarantees will not need differentiability of $f(\cdot)$; otherwise, they are labeled for minimizing smooth functions, and they need differentiability of $f(\cdot)$.

\paragraph{Convex function and conjugate}
For any convex function $f(\cdot)$, its Fenchel conjugate is 
\begin{equation}
f^*(y) := \sup_{x \in \text{dom}(f) }  \langle x, y \rangle - f(x).
\end{equation}
It is noted that the conjugate function $f^*(\cdot)$ is convex, as it is a supremum over a linear function.
Furthermore, if the function $f(\cdot)$ is closed and convex,
the following are equivalent:
(I) $y \in \partial f(x)$, (II) $x \in \partial f^*(y)$, and (III)
\begin{equation}
\langle x, y \rangle = f(x) + f^*(y),
\end{equation}
which also implies that the bi-conjugate is equal to the original function, i.e. $f^{**}(\cdot) = f(\cdot)$.
Moreover,
when the function $f(\cdot)$ is differentiable,
we have $\nabla f(x) = \displaystyle \underset{y \in \text{dom}(f^*)}{\arg\sup}  \langle x, y \rangle - f^*(y) $;
on the other hand,
when $f(\cdot)$ is not differentiable,
we have $f_x = \displaystyle \underset{y \in \text{dom}(f^*)}{\arg\sup}  \langle x, y \rangle - f^*(y)$, where $f_{x} \in \partial f(x)$ is a subgradient at $x$. 
We refer to the readers to \cite{BL12,KST09} for more details of Fenchel conjugate.
Throughout this paper, %, unless specifically mentioned, 
we assume that the underlying convex function is closed. %sand differentiable. 
%\junkun{Smoothed FW with Non-Smooth function is not the case (See Table 1). We neither have actual proof.}

An important property of a closed and convex function is that
$f(\cdot)$ is $L$-smooth w.r.t. some norm $\| \cdot \|$ if and only if its conjugate $f^*(\cdot)$ is $1/L$-strongly convex w.r.t. the dual norm $\| \cdot\|_*$ (e.g. Theorem~6 in \cite{KST09}).
%\kfir{ also mention that the Fenchel conjugate of $L$ smooth convex function is $1/L$-strongly convex and vice versa}

\paragraph{Bregman Divergence.}

We will denote the Bregman divergence $\V{z}(\cdot)$ centered at a point $z$ with respect to a $\beta$-strongly convex distance generating function $\phi(\cdot)$ as 
\begin{equation}\label{eq:BregmanDef}
\V{z}(x) := \phi(x) - \langle \nabla \phi(z), x - z \rangle  - \phi(z).
\end{equation}
%\kfir{ should be either $V_z$ or $\V{z}$, currently the notation is inconsistent}
%\kfir{should be Capital $V$ in Equation 3 above}
%\junkun{done}

\paragraph{Strongly convex sets.}
A convex set $\K \subseteq \reals^m$ is an \emph{$\lambda$-strongly convex set} w.r.t. a norm $\| \cdot \|$ 
if for any $x, z \in \K$, any $\theta \in [0,1]$, 
%\kfir{change $u,v$ to $x,z$}.
the $\| \cdot \|$ ball centered at $ \theta x + ( 1 - \theta) z$ with radius 
$\theta (1 - \theta) \frac{\lambda}{2} \| x - z \|^2$ is included in $\K$ \cite{D15}. Examples of strongly convex sets include 
$\ell_p$ balls: $\| x \|_p \leq r, \forall p \in (1,2]$,
Schatten $p$ balls: $\| \sigma(X) \|_p \leq r$ for $p \in (1,2]$,
where $\sigma(X)$ is the vector consisting of singular values of the matrix $X$,
and Group (s,p) balls: $\| X \|_{s,p}  = \| (\| X_1\|_s, \| X_2\|_s, \dots, \| X_m\|_s)  \|_p \leq r$,
$ \forall s,p \in (1,2]$.
In Appendix~\ref{app:betagauge}, we discuss more about strongly-convex sets.

%\kfir{Is the next fact  necessary for us? where do we use it?}
%\junkun{used in Lemma~\ref{lem:lip}. indeed this property is not necessary at this porint}
\iffalse
\cite{P96} provide an equivalent definition and show that a strongly convex set $\K$ can also be written as intersection of some Euclidean balls. Namely,
    \[ \K = \underset{u: \| u\|_2 = 1}{\cap} B_{\frac{1}{\lambda}} \left( x_u - \frac{u}{\lambda} \right) ,\]
where $x_u$ is defined as $x_u = \argmax_{x \in \K} \langle \frac{u}{\|u\|}, x\rangle$.
\kfir{Define $B_r(x)$ as the Euclidean ball with radius $r$ and center $x$}
\fi

\subsection{Min-max Problems and the Fenchel Game}

Throughout the rest of the paper we will be focusing on the following problem. For some natural number $d > 0$, have some convex and closed set $\K \subseteq \reals^d$ and a closed and convex function $f : \K \to \reals$. Our goal is to solve the minimization problem $\min_{x \in \K} f(x)$. While we consider both bounded and unbounded sets $\K$ (including $\K = \reals^d$), we assume throughout that a minimizer of $f$ exists in some bounded region of $\K$.

\iffalse
A large number of core problems in statistics, optimization, and machine learning, can be framed as the solution of a two-player zero-sum game. Linear programs, for example, can be viewed as a competition between a feasibility player, who selects a point in $\reals^n$, and a constraint player that aims to check for feasibility violations \cite{Adler2013}. Boosting \cite{freund1999adaptive} can be viewed as the competition between an agent that selects hard distributions and a weak learning oracle that aims to overcome such challenges \cite{freund1996game}. The hugely popular technique of Generative Adversarial Networks (GANs) \cite{goodfellow2014generative}, which produce implicit generative models from unlabeled data, has been framed in terms of a repeated game, with a distribution player aiming to produce realistic samples and a discriminative player that seeks to distinguish real from fake. 
\fi

\paragraph{Min-max problems and (approximate) Nash equilibrium}
Given a zero-sum game with \textit{payoff function} $g(x,y)$ which is convex in $x$ and concave in $y$, 
define $V^*=\inf_{x \in \XX} \sup_{y \in \YY} g(x,y)$. An $\epsilon$-\textit{equilibrium} of $g(\cdot, \cdot)$ is a pair $\hat x, \hat y$ such that
\begin{equation} \label{def:equi}
\displaystyle V^* - \epsilon \leq \inf_{x \in \XX} g(x, \hat y)  \leq V^* \leq   \sup_{y \in \YY} g(\hat x, y) \leq V^* + \epsilon,
\end{equation}
where $\XX$ and $\YY$ are convex decision spaces of the $x$-player and the $y$-player respectively.
\noindent

% In this paper, we show a deep connection between the \textit{no-regret framework} and the classical \textit{convex optimization problem} $\min_x f(x)$.

%\junkun{Maybe we move the following paragraphs of (Fenchel Game) to Section 3.2?}

\paragraph{The Fenchel Game.}

One of the core tools of this paper is as follows. In order to solve the problem 
%\begin{equation}
$  \min_{x \in \K} f(x)$,
%\end{equation}
we instead construct a saddle-point problem which we call the \emph{Fenchel Game}. We
define $g : \K \times \YY$ as follows:
\begin{equation} \label{eq:fenchelgame}
 g(x,y) := \langle x, y \rangle - f^*(y),
\end{equation}
where $\YY$ is the gradient space of $f$ that will be precisely defined later in Section~\ref{sec:fgnrd_framework}.
This payoff function is useful for solving the original optimization problem, since an equilibrium of this game provides us with a solution to $\min_{x \in \K} f(x)$. 
Let $\hat x \in \K$ and $\hat y \in \YY$ be any equilibrium pair of $g$,
which exists even when the convex sets $\K$ and $\YY$ can be unbounded, thanks to the assumption that a minimizer of a closed convex function $f(\cdot)$ is attained in a bounded region.
We have $V^* = \sup_{y \in \YY} g(\hat x, y)$. Furthermore,
%Then we have 
\begin{eqnarray*}
  \inf_{x \in \K} f(x) & = & \inf_{x \in \K} \sup_{y \in \YY} \{\langle x, y\rangle - f^*(y)\} = \inf_{x \in \K} \sup_{y \in \YY} g(x,y)\\
    & = & \sup_{y \in \YY} g(\hat{x}, y ) = \sup_{y \in \YY}  \left\{ \langle \hat{x}, y \rangle - f^*(y) \right\} = f(\hat{x}).
\end{eqnarray*}
In other words, given an equilibrium pair $\hat x, \hat y$ of $g(\cdot,\cdot)$, we immediately have a minimizer of $f(\cdot)$. This simple observation can be extended to approximate equilibria as well.
\begin{lemma} \label{lem:fenchelgame}
  If $(\hat x, \hat y)$ is an $\epsilon$-equilibrium of the Fenchel Game \eqref{eq:fenchelgame}, then $f(\hat x) - \min_{x} f(x) \leq \epsilon$.
\end{lemma}

% \paragraph{Finding an approximate equilibrium via OCO.} 

% \jake{This whole bit needs a rewrite.}

Lemma~\ref{lem:fenchelgame} sets us up for the remainder of the paper. The framework, which we lay out precisely in Section~\ref{sec:fgnrd_framework}, will consider two players sequentially playing the Fenchel game, where the $y$-player sequentially outputs iterates $y_1, y_2, \ldots$, while alongside the $x$-player returns iterates $x_1, x_2, \ldots$. Each player may use the previous sequence of actions of their opponent in order to choose their next point $x_t$ or $y_t$, and we will rely heavily on the use of no-regret online learning algorithms described in Section~\ref{sec:onlinelearning}. In addition, we need to select a sequence of weights $\alpha_1, \alpha_2, \ldots > 0$ which determine the ``strength'' of each round, and can affect the players' update rules. What we will be able to show is that the $\alpha$-weighted average iterate pair, defined as
\[
 (\hat x, \hat y) := \left(\frac{\alpha_1 x_1 + \ldots + \alpha_T x_T}{\alpha_1 + \cdots + \alpha_T}, \frac{\alpha_1 y_1 + \cdots + \alpha_Ty_T}{\alpha_1 + \cdots + \alpha_T}\right),
\]
is indeed an $\epsilon$-equilibrium of $g(\cdot, \cdot)$, and thus via Lemma~\ref{lem:fenchelgame} we have that $\hat x$ approximately minimizes $f$. To get a precise estimate of $\epsilon$ requires us to prove a family of regret bounds, which is the focus of the following section.

\section{No-regret learning algorithms} \label{sec:onlinelearning}

An algorithmic framework, often referred to as \emph{no-regret learning} or \emph{online convex optimization}, which has been developed mostly within the machine learning research community, has grown quite popular as it can be used in a broad class of sequential decision problems. As we explain in Section~\ref{sec:noregret}, one imagines an algorithm making repeated decisions by selecting a vector of parameters in a convex set, and on each round is charged according to a varying convex loss function. The algorithm's goal is to minimize an objective known as regret. In Section~\ref{sec:fgnrd_framework} we describe how online convex optimization algorithms with vanishing regret can be implemented in a two-player protocol which sequentially computes an approximate equilibria for a convex-concave payoff function. This is the core tool that allows us to describe a range of known and novel algorithms for convex optimization, by modularly combining pairs of OCO strategies. In Section~\ref{sec:oco_algs} we provide several such OCO algorithms, most of which have been proposed and analyzed over the past 10-20 years.

\renewcommand{\algorithmiccomment}[1]{\hfill \tiny//~#1\normalsize}

\begin{algorithm}[H]
\floatname{algorithm}{Protocol}
\caption{Weighted Online Convex Optimization}
\label{alg:oco_protocol}
\begin{algorithmic}[1] \normalsize
\State \textbf{Input:} convex decision set $\ZZ \subseteq \reals^d$
\State \textbf{Input:} number of rounds $T$
\State \textbf{Input:} weights $\alpha_1, \alpha_2, \ldots, \alpha_T > 0$ \Comment{Weights determined in advance}
\State \textbf{Input:} algorithm $\alg$ \Comment{This implements the learner's update strategy}
\For{$t=1, 2, \ldots, $}
\State \textbf{Return:} $z_t \gets \alg$ \Comment{Alg returns a point $z_t$}
\State \textbf{Receive:} $\alpha_t, \ell_t(\cdot) \to \alg$ \Comment{Alg receives loss fn. and round weight}
\State \textbf{Evaluate:} $\text{Loss} \gets \text{Loss} + \alpha_t \ell_t(z_t)$ \Comment{Alg suffers weighted loss for choice of $z_t$}
\EndFor
\end{algorithmic}
\end{algorithm} 

\subsection{Online Convex Optimization and Regret} \label{sec:noregret}

%\junkun{I have changed $x$ to $z$ to avoid the confusion.}

Here we describe the framework, given precisely in Protocol~\ref{alg:oco_protocol}, for online convex optimization. We assume we have some learning algorithm known as $\alg$ that is tasked with selecting ``actions'' from a convex \emph{decision set} $\ZZ \subseteq \reals^d$. On each round $t=1, \ldots, T$, \alg returns a point $z_t \in \ZZ$, and is then presented with the pair $\alpha_t, \ell_t$, where $\alpha_t > 0$ is a weight for the current round and $\ell_t : \ZZ \to \reals$ is a convex loss function that evaluates the choice $z_t$. 
While \alg is essentially forced to ``pay'' the cost $\alpha_t \ell_t(z_t)$, it can then update its state to provide better choices in future rounds. 

On each round $t$, the learner must select a point $z_t \in \ZZ$, and is then ``charged'' a loss of $\alpha_t \ell_t(z_t)$ for this choice. Typically it is assumed that, when the learner selects $z_t$ on round $t$, she has observed all loss functions $\alpha_1 \ell_1(\cdot), \ldots, \alpha_{t-1} \ell_{t-1}(\cdot)$ up to, but not including, time $t$. However, we will also consider learners that are \emph{prescient}, i.e. that can choose $z_t$ with knowledge of the loss functions up to \emph{and including} time $t$. The objective of interest in most of the online learning literature is the learner's \emph{regret}, defined as
\begin{equation}
  \regret{z}(z^*) := \sum_{t=1}^T \alpha_t \ell_t(z_t) -  \sum_{t=1}^T \alpha_t \ell_t(z^*),
  %- \min_{z \in \ZZ} \sum_{t=1}^T \alpha_t \ell_t(z).
\end{equation}
where $z^{*} \in \ZZ$ is a comparator that the online learner is compared to.
Oftentimes we will want to refer to the \emph{average regret}, or the regret normalized by the time weight $A_T := \sum_{t=1}^T \alpha_t$, which we will denote $\avgregret{z}(z^*) := \frac{\regret{z}(z^*)}{A_T}$.
Note that in online learning literature,  
what has become a cornerstone of online learning research has been the existence of \emph{no-regret algorithms}, i.e. learning strategies that guarantee $\avgregret{z}(z^*) \to 0$ as $A_T \to \infty$.

%Let us consider some very simple learning strategies that will be used in this paper,
%and we note the available guarantees for each.
%We also refer the readers to some tutorial of online learning
%for more online learning algorithms.
%(see e.g. \cite{OO19,RS16,hazan2016introduction,shalev2012online}) 

%\clearpage
\subsection{Framework: optimization as \emph{Fenchel Game}} \label{sec:fgnrd_framework}

In this paper, we consider Fenchel game (\ref{eq:fenchelgame}) with weighted losses depicted in Protocol~\ref{alg:game}. In this game, one of the players play first, and the other player would see what its opponent plays before choosing its action.
What Protocol~\ref{alg:game} describes is the case when the $y$-player plays before the $x$-player plays, but we emphasize that we can swap the order and let the $x$-player plays first, which will turn out to be helpful when designing and analyzing certain algorithms. 
%and the $x$-player sees what the $y$-player plays before choosing its action.
The $y$-player receives loss functions $\alpha_{t} \ell_{t}(\cdot)$ in round $t$, in which $\ell_{t}(y):= f^{*}(y)- \langle x_t, y \rangle $,
while the $x$-player receives its loss functions $\alpha_{t} h_{t}(\cdot)$ in round $t$, in which $h_{t}(x):= \langle x, y_t \rangle - f^{*}(y_t)$.
Consequently, we can define the \textit{weighted regret} of the $x$ and $y$ players as
\begin{eqnarray} 
  \label{eq:yregret}     \regret{y} &   := & 
     \sum_{t=1}^T  \alpha_t  \ell_t(y_t) - \min_{y \in \YY} \sum_{t=1}^T  \alpha_t  \ell_t(y)\\
  \label{eq:xregret}      \regret{x} &   := & 
     \sum_{t=1}^T  \alpha_t  h_t(x_t) - \min_{x \in \XX } \sum_{t=1}^T  \alpha_t  h_t(x),
     %\sum_{t=1}^T  \alpha_t  h_t(x^*).
\end{eqnarray}
where the decision space of the $x$-player $\XX$ is
that of the underlying optimization problem $\min_{{x \in \K}} f(x)$, 
i.e. $\XX = \K \subseteq \reals^{d}$, while the decision space of the $y$-player is the gradient space, i.e. $\YY := \bigcup_{x \in \XX} \partial f(x)$,
% $\{y : y \in \partial f(x), x \in \XX \}$
which is the union of %the sets of subgradients 
subdifferentials over the domain $\XX$.
One can check that the closure of the set $\YY$ is a convex set. Appendix~\ref{app:show_cvx} describes the proof. 
\begin{theorem} \label{th:cvx}
The closure of the gradient space 
$\bigcup_{x \in \XX} \partial f(x)$ is a convex set.
\end{theorem}

%Notice that the $x$-player's regret is computed relative to $x^*$, a minimizer $\argmin_{x \in {\K} } f(x)$, rather than the minimizer of $\sum_{t=1}^T  \alpha_t  h_t(\cdot)$. 
%Although slightly non-standard, this allows us to handle the unconstrained setting while Theorem~\ref{thm:convergence} still holds as desired.

\begin{algorithm}[t]
\floatname{algorithm}{Protocol}
   \caption{Fenchel Game No-Regret Dynamics} \label{alg:game}
\begin{algorithmic}[1]
\normalsize
\State \textbf{Input:} number of rounds $T$
\State \textbf{Input:} Convex decision sets $\XX, \YY \subseteq \reals^d$
\State \textbf{Input:} Convex-concave payoff function $g : \XX \times \YY \to \reals$
\State \textbf{Input:} weights $\alpha_1, \alpha_2, \ldots, \alpha_T > 0$ \Comment{Weights determined in advance}
\State \textbf{Input:} algorithms $\alg^Y, \alg^X$ \Comment{Learning algorithms for both players}
\For{$t=1, 2, \ldots, T$}
\State \textbf{Return:} $y_t \gets \alg^Y$ \Comment{$y$-player returns a point $y_t$}
\State \textbf{Update:} $\alpha_t, h_t(\cdot) \to \alg^X$ \Comment{$x$-player updates with $\alpha_t$ and loss $g(\cdot,y_t)$}
\State \quad \quad where $h_t(\cdot) := g(\cdot,y_t)$
\State \textbf{Return:} $x_t \gets \alg^X$ \Comment{$x$-player returns a point $x_t$}
\State \textbf{Update:} $\alpha_t, \ell_t(\cdot) \to \alg^Y$ \Comment{$y$-player updates with $\alpha_t$ and loss $-g(x_t,\cdot)$}
\State \quad \quad where $\ell_t(\cdot) := -g(x_t,\cdot)$
\EndFor
% \State \textbf{Evaluate:} $\text{Loss} \gets \text{Loss} + \alpha_t \ell_t(x_t)$ \Comment{Alg suffers weighted loss for choice of $x_t$}
% \State Input: sequence $\alpha_1, \ldots, \alpha_T > 0$
% \For{$t= 1, 2, \dots, T$}
% \State $y$-player selects $y_t\in \YY = \reals^d$ by $\alg^y$.
% \State $x$-player selects $x_t \in \XX$ by $\alg^x$, possibly with knowledge of $y_t$.
% \State $y$-player suffers loss $\ell_{t}(y_t)$ with weight $\alpha_t$, where $\ell_t(\cdot) = -g(x_t,\cdot)$.
% \State $x$-player suffers loss $h_{t}(x_t)$ with weight $\alpha_t$, where $h_t(\cdot) = g(\cdot,y_t)$.
% \EndFor
\State Output $(\xav_T,\yav_T) := \left(\frac{ \sum_{s=1}^T \alpha_s x_s  }{ A_T }, \frac{ \sum_{s=1}^T \alpha_s y_s  }{ A_T }\right)$.
\end{algorithmic}
\end{algorithm}

At times when we want to refer to the regret on another sequence $y_1', \ldots, y_T'$ we may refer to this as $\regret{}(y_1', \ldots, y_T')$.
We also denote $A_t$ as the cumulative sum of the weights $A_t:=\sum_{s=1}^t \alpha_s$ and the weighted average regret $\avgregret{} := \frac{\regret{}}{A_T}$. 
%Finally, we will
%Finally, for offline constrained optimization (i.e. $\min_{x \in \K} f(x)$), we let the decision space of the comparator/comparator in the weighted regret definition to be $\XX=\K$; for offline unconstrained optimization, we let the decision space of the comparator/comparator to be a norm ball that contains the optimum solution of the offline problem (i.e. contains $\argmin_{x \in \reals^n} f(x)$), which means that $\XX$ of the comparator is a norm ball. 
%We let $\YY = \reals^d$ be unconstrained.
\begin{theorem}\label{thm:meta} 
  Assume a $T$-length sequence $\balpha$ are given. Suppose in Protocol~\ref{alg:game} the online learning algorithms $\alg^x$ and $\alg^y$ have the $\balpha$-weighted average regret $\avgregret{x}$ and $\avgregret{y}$ respectively. Then the output  $(\bar{x}_{T},\bar{y}_{T})$ is an $\epsilon$-equilibrium for $g(\cdot, \cdot)$, with
$    \epsilon = \avgregret{x} + \avgregret{y}$. Moreover, if $f(x) := \sup_{y \in \YY} g(x,y)$, then it follows that
\[
  f(\bar x_T) - \min_{x \in \XX} f(x) \leq \avgregret{x} + \avgregret{y}.
\]
\end{theorem} 

\begin{proof}
Suppose that the loss function of the $x$-player in round $t$ is $\alpha_t h_t(\cdot) : \XX \to \reals$, where $h_t(\cdot) := g(\cdot, y_t)$. The $y$-player, on the other hand, observes her own sequence of loss functions $\alpha_t  \ell_t(\cdot) : \YY \to \reals$, where $\ell_t(\cdot) := -  g(x_t, \cdot)$.\\

\begin{eqnarray}
\frac{1}{\sum_{s=1}^T \alpha_s}   \sum_{t=1}^T \alpha_t g(x_t, y_t) 
  & = & \frac{1}{\sum_{s=1}^T \alpha_s} \sum_{t=1}^T  - \alpha_t  \ell_t(y_t)  \notag \\
  \text{} \; & \geq & 
    - \frac{1}{\sum_{s=1}^T \alpha_s} \inf_{y \in \YY} \left( \sum_{t=1}^T  \alpha_t  \ell_t(y) \right) - \frac{ \regret{y} }{  \sum_{s=1}^T \alpha_s } \notag \\
  \; & = &
    \sup_{y \in \YY} \left( \frac{1}{\sum_{s=1}^T \alpha_s} \sum_{t=1}^T  \alpha_t g( x_t , y ) \right) - \avgregret{y}  \notag \\
  \text{(Jensen)} \;  & \geq & 
    \sup_{y \in \YY} g\left({ \frac{1}{\sum_{s=1}^T \alpha_s} \sum_{t=1}^T  \alpha_t x_t }, y \right)  - \avgregret{y} 
    \notag   \\
  \text{} \;  & = & 
    \sup_{y \in \YY} g\left({ \xav_T }, y \right)  - \avgregret{y} 
     \label{eq:ylowbound1}  \\
  & \geq & \inf_{x \in \XX} \sup_{y \in \YY} g\left( x , y \right) - \avgregret{y}  \label{eq:ylowbound2} .
\end{eqnarray}

Let us now apply the same argument on the right hand side, where we use the $x$-player's regret guarantee.
\begin{eqnarray}
\frac{1}{\sum_{s=1}^T \alpha_s}  \sum_{t=1}^T  \alpha_t g(x_t, y_t) & = & \frac{1}{\sum_{s=1}^T \alpha_s} \sum_{t=1}^T  \alpha_t h_t(x_t) \notag \\
  & \leq & \inf_{x \in \XX} \left( 
    \sum_{t=1}^T \frac{1}{\sum_{s=1}^T \alpha_s} \alpha_t g(x,y_t) \right) + \frac{ \regret{x} }{  \sum_{s=1}^T \alpha_s }  \notag \\
  & = & \inf_{x \in \XX} \left(  \sum_{t=1}^T \frac{1}{\sum_{s=1}^T \alpha_s} \alpha_t g(x, y_t) \right) + \avgregret{x} \notag \\
\text{(Jensen)}  & \leq &   \inf_x  
    g\left(x,{  \sum_{t=1}^T \frac{1}{\sum_{s=1}^T \alpha_s} \alpha_t y_t}\right)  + \avgregret{x} \notag \\
%    \label{eq:xupbound1} \\
  & = & 
  \inf_x  g\left(x,{  \yav_T}\right) + \avgregret{x} \label{eq:xupbound1}
    \\
  & \leq & \sup_{y \in \YY} \inf_{x \in \XX}  g(x,y) + \avgregret{x}  \label{eq:xupbound2}  .
% & = & \sup_{y \in \YY}  \inf_{x \in \XX} g(x,y) + \avgregret{x}  \notag,
\end{eqnarray}
By \eqref{eq:xupbound2} and \eqref{eq:ylowbound1}, we have
$\sup_{y \in \YY} g\left(\bar{x}_T, y \right) \leq V^* + \avgregret{x} + \avgregret{y}.$
By \eqref{eq:xupbound1} and \eqref{eq:ylowbound2}, we have
$V^* - \avgregret{x} - \avgregret{y} \leq \inf_{x \in \XX} g\left(x, \bar{y}_T \right)$. Thus
$(\bar{x}_{T}, \bar{y}_{T})$ is an $\epsilon =\avgregret{x} + \avgregret{y} $ equilibrium.

We also have
\begin{eqnarray}
\frac{1}{\sum_{s=1}^T \alpha_s}  \sum_{t=1}^T  \alpha_t g(x_t, y_t) & = & \frac{1}{\sum_{s=1}^T \alpha_s} \sum_{t=1}^T  \alpha_t h_t(x_t) \notag \\
  & \leq & \label{rb} \left( 
    \sum_{t=1}^T \frac{1}{\sum_{s=1}^T \alpha_s} \alpha_t h_t(x^*) \right) + \frac{ \regret{x} }{  \sum_{s=1}^T \alpha_s }   \\
  & = &   \left(  \sum_{t=1}^T \frac{1}{\sum_{s=1}^T \alpha_s} \alpha_t g(x^*, y_t) \right) + \avgregret{x} \notag \\
\text{(Jensen)}  & \leq &   
    g\left(x^*,{  \sum_{t=1}^T \frac{1}{\sum_{s=1}^T \alpha_s} \alpha_t y_t}\right) + \avgregret{x} 
    \notag \\
  & = & 
    g\left(x^*,{  \yav_T}\right) + \avgregret{x} \notag
    \\
  & \leq & \sup_{y \in \YY}  g(x^*,y) + \avgregret{x}  \label{eq:xupbound4},
% & = & \sup_{y \in \YY}  \inf_{x \in \XX} g(x,y) + \avgregret{x}  \notag,
\end{eqnarray}
where we denote $x^* \leftarrow \argmin_{x\in \K} f(x)$ and \eqref{rb} is because
$ \sum_{t=1}^T  \alpha_t h_t(x_t) - \sum_{t=1}^T  \alpha_t h_t(x^*) 
\leq \sum_{t=1}^T  \alpha_t h_t(x_t) - \min_{x \in \XX} \sum_{t=1}^T  \alpha_t h_t(x) := \regret{x}$.
%Note that $\sup_{y \in \YY}  g(x^*,y) = f(x^*)$ by Fenchel conjugacy, and hence we can conclude that $\sup_{y \in \YY} g(x^*,y)  = V^* = \sup_{y \in \YY} \inf_{x \in \XX}  g(x,y) = \inf_{x \in \XX} \sup_{y \in \YY}   g(x,y)$.
Combining (\ref{eq:ylowbound1}) and (\ref{eq:xupbound4}), we see that
$  f(\bar x_T) - \min_{x \in \XX} f(x) \leq \avgregret{x} + \avgregret{y}.$
%\begin{equation}
%    \sup_{y \in \YY} g\left({ \xav_T}, y \right)  - \avgregret{y}  \le \inf_{x \in \XX}  
%g\left(x,{  \yav_T }\right) + \avgregret{x} 
%\end{equation}
%which implies that 
%$(\bar{x}_{T}, \bar{y}_{T})$ is an $\epsilon =\avgregret{x} + \avgregret{y} $ equilibrium.
 \end{proof}

In order to utilize minimax duality, we have to define decision sets for two players, and we must produce a convex-concave payoff function. First we will assume, for convenience, that $f(x) := \infty$ for any $x \notin \XX$. That is, it takes the value $\infty$ outside $\XX$, 
%of a compact convex set $\XX$ if $\XX \neq \reals^{d}$,
 which ensures that $f(\cdot)$ is lower semi-continuous and convex. 
%We let the decision space of the $y$-player 
We will also assume that the solution(s) to the underlying problem
 $x^* \leftarrow \argmin_{x \in \K} f(x)$ has a finite size.
As we will see in Theorem~\ref{thm:metaAcc} and Theorem~\ref{thm:proximal}, the convergence rate of certain algorithms depend on the initial distance to $x^{*}$, and hence this assumption guarantees that the initial distance is finite.
%Finally, for offline constrained optimization (i.e. $\min_{x \in \K} f(x)$), we let the decision space of the comparator/comparator in the weighted regret definition to be $\XX=\K$; for offline 

\paragraph{Faster rates for min-max problems.} 
We conclude this section by noting that many of the accelerated rates that we prove in this work emerge from a new set of tools that have been developed within the past decade. The use of no-regret learning to solve zero-sum games, as described in Theorem~\ref{thm:meta}, was popularized by Freund and Schapire \cite{freund1996game} in 1996, although the idea goes back to Blackwell \cite{blackwell1956analog} and Hannan \cite{hannan1957approximation} in the 1950s. While this basic tool has been used in many settings and applications, a key trick was discovered by Rakhlin and Sridharan \cite{RS13} in 2013, showing that the rate for solving a simple matrix game could be improved from $O(1/\sqrt{T})$ to $O(1/T)$ via a more sophisticated type of algorithm which we herein refer to as \emph{optimistic}; this will be explained precisely in Section~\ref{sec:oco_algs}. The first known presentation of such an algorithm was given by \cite{CJ12}, who showed that a better regret bound is achievable if the sequence of loss functions changes slowly. As explained in \cite{RS13}, this is very useful when used in Protocol~\ref{alg:game}. This new trick is a core piece of many of the tools presented in this paper, and we emphasize that all of the \emph{accelerated} optimization algorithms we describe through the FGNRD framework rely on one of the players using an optimistic online learning algorithm.

On the other hand, classic optimization algorithms like mirror prox of Nemirovski \cite{nemirovski2004prox} and dual extrapolation of Nesterov \cite{nesterov2007dual} were invented earlier than the aforementioned no-regret learning approaches for solving min-max problems \cite{RS13,CJ12,SALS15}.
Both mirror prox and dual extrapolation can be applied to convex-concave problems and obtain a $O(1/T)$ fast rate, compared to a $O(1/\sqrt{T})$ slow rate \cite{nedic2009subgradient}.
Furthermore, both can be more broadly applied to solve variational inequalities in monotone operators \cite{auslender2005interior,eckstein1992douglas,chen2014optimal,chen2017accelerated,he2016accelerated,juditsky2011solving}, where finding a saddle point of a convex-concave function is a special case, see also \cite{tseng2000modified,malitsky2020golden,iusem2017extragradient} for related treatments.
We note that there are some recent advancements based on mirror prox and dual extrapolation, e.g., \cite{CST20},
and there are also some progress on deriving lower bounds for solving certain classes of min-max problems, e.g., \cite{ouyang2021lower,zhang2022lower}.

%We also note that solving min-max problems has drawn great interest in optimization and machine learning over the past few years, e.g., \cite{goodfellow2014generative,mertikopoulos2018optimistic,mokhtari2020convergence,BHKM15,namkoong2016stochastic}, and some notable applications include 
%training generative adversarial networks \cite{goodfellow2014generative} and optimization under uncertainties, e.g., \cite{BHKM15,namkoong2016stochastic,ho2019exploiting}.

\paragraph{Min-max formulations of a primal problem} 

The idea of casting a primal problem as a saddle point problem via Fenchel conjugate is not new.
For example, Chambolle and Pock \cite{chambolle2011first} consider solving a primal problem of the form,
$ \min_{{x \in \K}} f( M x) + \psi(x)$, 
via solving a saddle point problem
$\min_{x \in \K} \max_{y \in \Theta} \langle M x, y \rangle + \psi(x) - f^*(y) $,
where $M: \K \to \Theta$ is a continuous linear operator.
When $Mx=x$ and $\psi(\cdot)=0$, it reduces to our Fenchel Game formulation. 
Other examples include \cite{pock2009algorithm,esser2010general,chambolle2016ergodic,zhu2008efficient,drori2015simple,combettes2014forward}, to name just a few.
However, the algorithmic approaches and the focuses in these related works are quite different
from ours.
For example,
the algorithms of Chambolle and Pock \cite{chambolle2011first}
assume that both $f(\cdot)$ and $\psi(\cdot)$ are simple in the sense that the corresponding resolvent operators can be efficient computed, i.e., problems of the form $\argmin_{{x \in \K}} f(x) + \frac{ \|x - z\|^2 }{2 \eta}$ can be efficiently solved, while our work does not need this assumption.
Moreover, we focus on demonstrating how no-regret learning together with the reformulation technique help design and analyze convex optimization algorithms systematically, enable to capture many existing algorithms and results, as well as to give rise to new ones, which is very different from existing works of primal-dual approaches. 
It is also interesting to note that in our approach, the $y$-player always outputs a gradient of the underlying function, which is not necessarily the case in other works, e.g., \cite{chambolle2011first,pock2009algorithm,esser2010general,chambolle2016ergodic,zhu2008efficient,drori2015simple,combettes2014forward}. 

Other relevant works include Gutman and Pe{\~n}a \cite{gutman2022perturbed,gutman2019convergence}. In \cite{gutman2022perturbed}, they consider solving convex optimization problems of the form, 
$\min_{x\in \K} f(Mx) + \psi(x)$, where $M: \K \to \text{dom}(f)$ is a linear operator.
The design and analysis of their work focuses on minimizing the corresponding Fenchel duality gap,
$$ \Delta = f(Mx) + \psi(x) + f^*(z) + \psi^*( - M^\top z ),$$ where $z$ is the dual variable.
Their algorithmic framework also considers outputs a weighted average of points and allows different weights on each iteration. From their framework, they derive generalized Frank-Wolfe, proximal gradient method and its variants. Our work differs from \cite{gutman2022perturbed} because we consider applying no-regret learning and the corresponding analysis, which is not present in \cite{gutman2022perturbed}. In \cite{gutman2019convergence}, they provide a proof of $O(1/T^2)$ rate of accelerated proximal gradient method \cite{BT09}, where the analysis is done by weighting a certain duality gap. On the other hand, we will show a simpler proof of the accelerated proximal method by plugging in and summing the regret bounds of a pair of online learning strategies, see Theorem~\ref{thm:proximal}. Interestingly, our proof of the accelerated proximal method and that of Nesterov's $1$-memory method (Theorem~\ref{thm:metaAcc}) is the same from our Fenchel Game interpretation.

%\paragraph{Faster rates for min-max problems.} 

\paragraph{The use of the weighted regret} 

We note that the idea of the weighted regret can also be found in 
the works of Ho-Nguyen and K{\i}l{\i}n{\c{c}}-Karzan 
\cite{ho2019exploiting,ho2018primal}.
They consider applying no-regret learning algorithms to the associated Lagrangian of a constrained optimization problem, while ours is for the Fenchel Game.
%However, their focus is different from ours.
Specifically,
\cite{ho2018primal} consider solving a convex problem with functional constraints,
$$\min_{x \in \K} f(x), \qquad \text{ s.t. } h_{i}(x) \leq 0, \forall i \in [m],$$
via solving the min-max problem, $\min_{{x \in \K}} \max_{{y_i \geq 0 , \forall i \in [m]}} f(x)+ y_i h_i(x)$.
%However, their focuses of the problems are different from ours.
% In \cite{ho2018primal}, they consider solving a convex problem with functional constraints by solving an associated min-max problem via regret minimization. 
Their algorithm obtains an $O(1/T)$ rate 
 %for a class of non-smooth strongly convex optimization problems,
and one of the techniques is emphasizing the per-round regret on the later rounds.
In \cite{ho2019exploiting}, they consider optimization under uncertainty sets:
$ \min_{x \in \K} f(x), \text{ s.t.}\, \sup_{u_i \in U^i} h_i(x,u_i), \forall i \in [m],$
where each $U^i$ is a fixed uncertainty set.
Their algorithm achieves a fast $O(1/T)$ rate when the weight is $\alpha_{t}=t$, 
while other works of online saddle-point optimization have a slower $O(1/\sqrt{T})$ rate, e.g., \cite{koppel2015saddle,mahdavi2012trading}.

We also note that a monograph \cite{Nlec} has a subsection regarding an analysis of the weighted average regret of Online Mirror Descent.

\subsection{Online Convex Optimization: An Algorithmic Menu} \label{sec:oco_algs}
%\jake{This is for Kfir to shape up!}\\
In this section, we introduce and analyze several core online learning algorithms.  
We also refer the reader to some tutorial for more online learning algorithms (see e.g. \cite{OO19,RS16,hazan2016introduction,shalev2012online}). 
Later in Sections~\ref{sec:ExistingAlgs}~\&~\ref{sec:NewAlgs}, we will show how composing different online learning algorithm within the Fenchel Game No-Regret Dynamics (Protocol~\ref{alg:game}) enables to easily recover known results and methods for convex optimization (Section~\ref{sec:ExistingAlgs}), as well as to design new algorithm with novel guarantees (Section~\ref{sec:NewAlgs}).

We present two families of online learning strategies: the ``batch-style'' family of algorithms, and the ``update'' family, although we note that one can relate these two. Within each family, we have algorithms that require a regularization function $R(\cdot)$, or a ``link'' or ``distance-generating'' function $\phi(\cdot)$. In some cases, the algorithm may peek at the current loss function $\ell_t$ before selecting $z_t$ (these are denoted with a superscript $+$), or they may only have access to $\ell_1, \ldots, \ell_{t-1}$. We also present ``optimistic'' variants, where the algorithm can not see $\ell_t$ but can try to ``guess'' this function with some $m_t$, and use this guess in it's optimization to compute $z_t$. We note, importantly, that we have given different names to these algorithms, keeping up with historical precedents in the online learning community, but we observe that with the appropriate parameters every one of the batch-style algorithms is in fact a special case of \OFTRL, which is in some sense the Master Algorithm. We utilize this fact when proving regret bounds, as a Master Bound is provided in Lemma~\ref{regret:Opt-FTRL} of Subsection~\ref{sec:oftrl}, and nearly all of the remaining bounds of the ``batch-style'' family of algorithms in Algorithm~\ref{alg:oco_batch_family} follow as an easy corollary. Therefore, in the following, we describe the online learning strategies
with their regret bounds first, and we defer the proofs in Subsection~\ref{lem:various}.
It is noted that the online learning algorithms that will be introduced in this subsection do not need the loss function $\ell_{t}(\cdot)$ to be differentiable.

Before diving into the two families, we begin with one simple algorithm \BR that stands as something of an outlier.

\begin{algorithm} 
   \caption{ $\BR$ - Best Response} \label{alg:oco_br}
   \begin{align*}
    \text{Parameters:} \quad & \text{convex set $\ZZ$}\\
    \text{Receives:} \quad& \alpha_1, \ldots, \alpha_T > 0, \ell_1, \ldots, \ell_T  : \ZZ \to \reals\\
    \text{Action:} \quad&  z_t \leftarrow \argmin_{z \in \ZZ} \left\{\ell_t(z)\right\}
   \end{align*}
\end{algorithm}

\begin{algorithm} 
   \caption{A family of batch-style online learning strategies} \label{alg:oco_batch_family}
{\small
   \begin{align}
    \nonumber \text{Parameters:} \quad & \text{convex set $\ZZ$, initial point $\xinit \in \ZZ$}\\
    \nonumber \text{Receives:} \quad& \alpha_1, \ldots, \alpha_T > 0, \ell_1, \ldots, \ell_T  : \ZZ \to \reals\\
         \label{eq:alg_oco_ftl} \text{\FTL}[\xinit] \quad& z_t \leftarrow \xinit \text{ if } t=1 \text{ otherwise } \\
     \nonumber & z_t \leftarrow     \argmin_{{z \in \ZZ}} \left( \sum_{s=1}^{t-1} \alpha_{s} \ell_{s}(z) \right)
 \\
    \label{eq:alg_oco_btl} \text{\BTL:} \quad&  z_t \leftarrow 
     \argmin_{{z \in \ZZ}} \left( \sum_{s=1}^{t} \alpha_{s} \ell_{s}(z)  \right ) \\
    \label{eq:alg_oco_ftrl} \text{\FTRL}[R(\cdot), \eta]: \quad &  
    % z_t \leftarrow \xinit \text{ if } t=1 \text{ else } \\
    %  \nonumber & 
     z_t \leftarrow 
      \argmin_{{z \in \ZZ}} \left( \sum_{s=1}^{t-1} \alpha_{s} \ell_{s}(z) + \frac{1}{\eta} R(z)  \right  ) 
    \\
    \label{eq:alg_oco_btrl} \text{\BTRL}[R(\cdot), \eta]: \quad&  z_t \leftarrow 
      \argmin_{{z \in \ZZ}} \left( \sum_{s=1}^{t} \alpha_{s} \ell_{s}(z) + \frac{1}{\eta} R(z)  \right  ) 
    \\
    \label{eq:alg_oco_oftl} \text{\OFTL}: \quad&  
    % z_t \leftarrow \xinit =
    % \argmin_{z \in \ZZ} m_1(\cdot)
    %  \text{ if } t=1 \text{ else } \\
    %  \nonumber & 
     z_t \leftarrow
      \argmin_{{z \in \ZZ}} \left(\alpha_t m_t(z) +  \sum_{s=1}^{t-1} \alpha_{s} \ell_{s}(z) \right)
    \\
    \nonumber & \text{Requires ``guesses'' } m_1, \ldots, m_T : \ZZ \to \reals\\
    \label{eq:alg_oco_oftrl} \text{\OFTRL}[R(\cdot), \eta]: \quad&  
    % z_t \leftarrow \xinit =
    % \argmin_{z \in \ZZ} m_1(\cdot) \text{ if } t=1 \text{ otherwise } \\
    %  \nonumber & 
     z_t \leftarrow \argmin_{{z \in \ZZ}} \left( \alpha_t m_t(z) +  \sum_{s=1}^{t-1} \alpha_{s} \ell_{s}(z)  + \frac{R(z)}{\eta} \right)
    \\
    \nonumber & \text{Requires ``guesses'' } m_1, \ldots, m_T : \ZZ \to \reals
       \end{align}
}
\end{algorithm}
\begin{algorithm}
   \caption{A family of update-style online learning strategies} \label{alg:oco_update_family}
{\small 
   \begin{align}
     \nonumber \text{Parameters:} \quad & \text{convex set $\ZZ$, initial point $z_0 = z_{-\frac{1}{2}}  \in \ZZ$}\\
     \nonumber \text{Receives:} \quad& \alpha_1, \ldots, \alpha_T > 0, \ell_1, \ldots, \ell_T  : \ZZ \to \reals\\
     \nonumber \text{Set:} \quad& \alpha_0 := 0, \ell_0(\cdot) := 0\\
\OMD[\phi(\cdot),z_0,\gamma]: \quad
     \label{eq:alg_oco_omd}  & 
     % \text{if } t = 1 \text{ then } z_t \leftarrow \xinit \text{ else :} \\
     %  \nonumber & 
      z_t  \leftarrow  \argmin_{z \in \ZZ}  \left( \alpha_{t-1} \ell_{t-1}(z) +  \frac{1}{\gamma} \V{z_{t-1}}(z) \right)\\
\MD[\phi(\cdot),z_0,\gamma]: \quad
      \label{eq:alg_oco_md} & z_t  \leftarrow  \argmin_{z \in \ZZ}  \left( \alpha_{t} \ell_{t}(z) +  \frac{1}{\gamma} \V{z_{t-1}}(z) \right)
      \\
    \label{eq:alg_oco_optmd} \OPTMD[\phi(\cdot),z_0, \gamma]: \quad
      % & \text{if } t = 1 \text{ then } z_t \leftarrow \xinit \text{ else :} \\
      \nonumber & z_t \leftarrow \argmin_{z \in \ZZ}   \left( \alpha_{t} 
\langle z,  m_t \rangle + \frac{1}{\gamma} \V{z_{t-\frac{1}{2}}}(z)\right)
\\   \nonumber  & z_{t+\frac{1}{2}} \leftarrow   \argmin_{z \in \ZZ}   \left( \alpha_{t}
\langle z,  \nabla \ell_{t}(z_t) \rangle +  \frac{1}{\gamma} \V{z_{t-\frac{1}{2}}}(z)\right) \\
    \nonumber & \text{Requires ``guesses'' } m_1, \ldots, m_T : \ZZ \to \reals
\end{align}
}
\end{algorithm}

\subsubsection{Best Response -- \BR} \label{section:BR}
Perhaps the most trivial strategy for a prescient learner is to ignore the history of the $\ell_s$'s, and simply play the best choice of $z_t$ on the current round. We call this algorithm \BR. The ability to see the cost of your decisions in advance guarantees, quite naturally, that you will not suffer positive regret. We note that algorithms with a $+$ superscript may ``peek'' at the loss $\ell_t$ before choosing $z_t$.
\begin{lemma}{(\BR)}\label{lem:BRregret}
For any sequence of loss functions $\{ \alpha_t \ell_t(\cdot)\}_{t=1}^T$, if $z_1, \ldots, z_T$ is chosen according to \BR %\kfir{ should be \BR not \BTL} 
ensures,
 \begin{equation} \label{reg:BR}
\displaystyle 
\avgregret{z}(z^*)\leq \sum_{t=1}^T \alpha_t \ell_t(z_t) - \min_{z \in \ZZ} \sum_{t=1}^T \alpha_t \ell_t(z)  \leq 0,
\end{equation}
for any comparator $z^{*}$.
\end{lemma}

\begin{proof}
Since $ z_t = \argmin_{z \in \ZZ} \ell_t(z)$, we have 
$\ell_t(z_t) \leq \ell_t(z)$ for any $z \in \ZZ$.
The result follows by summing the inequalities from $t=1,\dots, T$, and recalling that the $\alpha_t$'s are non-negative.
\end{proof}

\subsubsection{A meta online learning algorithm: \OFTRL} \label{sec:oftrl}
Here we describe \OFTRL.
As can be seen from Equation~\eqref{eq:alg_oco_oftrl} in Algorithm~\ref{alg:oco_batch_family},  \OFTRL employs a regularization term (similar to \FTRL and \BTRL), and makes use of a hint sequence $m_t(\cdot)$ (similar to \OFTL).
In Lemma~\ref{regret:Opt-FTRL} below, we state the regret guarantees of \OFTRL. Later, we will show 
that all the other online learning strategies in Algorithm~\ref{alg:oco_batch_family} 
--- \FTL, \BTL, \FTRL, \BTRL, \OFTL are special cases of \OFTRL and will demonstrate how their guarantees follow as corollaries of this lemma in Subsection~\ref{lem:various}.

\begin{lemma}{({\OFTRL}$[R(\cdot),\eta]$)} \label{regret:Opt-FTRL}
Given $\eta>0$ and a $\beta$-strongly convex $R : \ZZ \to \reals$, assume we have $\{ \alpha_t \ell_t(\cdot)  \}_{t=1}^T$ a sequence of weights and loss functions such that each $\ell_t(\cdot)$ is $\mu_t$-strongly convex for some $\mu_t \geq 0$.
Let $m_1, \ldots, m_T : \ZZ \to \reals$ be the sequence of \emph{hint} functions given to \OFTRL, where each $m_t(\cdot)$ is $\hat{\mu}_t$-strongly convex function over $\ZZ$.
Then {\OFTRL}$[R(\cdot),\eta]$ satisfies
% is defined as follows,
% \begin{equation} \label{tmp:optftrl}
% z_t \leftarrow \argmin_{{z \in \ZZ}} \left( \sum_{s=1}^{t-1} \alpha_{s} \ell_{s}(z) \right)
% + \alpha_t m_t(z)  + \frac{1}{\eta} R(z),
% \end{equation}
 % {\OFTRL} satisfies,
% \begin{equation} 
\begin{align*}
& \regret{z}(z^*) \leq \stepcounter{equation}\tag{\theequation}\label{eq:uni}
 % :=  \sum_{t=1}^T  \alpha_t  \ell_t(x_t) - \min_{x \in \XX} \sum_{t=1}^T  \alpha_t  \ell_t(x)
\\ & \qquad \quad \sum_{t=1}^{T}  \alpha_t \left( \ell_t(z_t) - \ell_t(w_{t+1})  -    m_t(z_t) +  m_t(w_{t+1}) \right) \tag{\text{term (A)}}
\\
%\\ & \quad 
&  \qquad
+ \frac{1}{\eta} \left( R(z^*) - R(w_1) \right) \tag{\text{term (B)}}
\\
& \qquad -
   \frac{1}{2} \sum_{t=1}^T \left( \frac{\beta}{\eta} + \sum_{s=1}^{t-1} \alpha_s \mu_s  \right) \| z_t - w_t \|^2  
\tag{\text{term (C)}}  
\\
& \qquad -
\frac{1}{2}  \sum_{t=1}^T \left( \frac{\beta}{\eta} + \alpha_t \hat{\mu}_t + \sum_{s=1}^{t-1} \alpha_s \mu_s  \right) \| z_t - w_{t+1} \|^2  
\tag{\text{term (D)}},
\end{align*}
% where 
% \end{equation}
% \begin{equation} \label{eq:uni}
% \begin{aligned}
% & \regret{x}  :=  \sum_{t=1}^T  \alpha_t  \ell_t(x_t) - \min_{x \in \XX} \sum_{t=1}^T  \alpha_t  \ell_t(x)
% \\ & \leq
% \underbrace{
% \left( \sum_{t=1}^{T}  \alpha_t \left( \ell_t(x_t) - \ell_t(w_{t+1}) \right) - \alpha_t \left(  m_t(x_t) -  m_t(w_{t+1}) \right) \right)
% }_{ \text{term (A)} }
% %\\ & \quad 
% + \underbrace{ \frac{1}{\eta} \left( R(x^*) - R(w_1) \right) }_{ \text{ term (B)}}
% \\ & \qquad 
% \underbrace{
%   -
%    \frac{1}{2} \left( \sum_{t=1}^T \big( \frac{\beta}{\eta} + \sum_{s=1}^{t-1} \alpha_s \mu_s  \big) \| x_t - w_t \|^2  \right)
% }_{ \text{term (C)} }  
% \\ & \qquad
% \underbrace{
% -
%  \frac{1}{2} \left( \sum_{t=1}^T \big( \frac{\beta}{\eta} + \sum_{s=1}^{t-1} \alpha_s \mu_s+ \alpha_t \hat{\mu}_t  \big) \| x_t - w_{t+1} \|^2  \right)
% }_{ \text{term (D)}}
% , 
% \end{aligned}
% \end{equation}
where $w_1, \ldots, w_T$ is the alternative sequence chosen according to {\FTRL}$[R(\cdot),\eta]$
$w_{t}:= \argmin_{z \in \ZZ} \sum_{s=1}^{t-1} \alpha_s \ell_s(z)  + \frac{1}{\eta} R(z)$, and $z^* \in \ZZ$ is arbitrary.
 % the comparator in the regret definition  
% $\regret{z}(z^*):= \sum_{t=1}^T  \alpha_t  \ell_t(z_t) - \sum_{t=1}^T  \alpha_t  \ell_t(z^*)$ (similarly to the way we define in  Equation~\eqref{eq:xregret}).
%and $x^* = \argmin_{x \in \K} \sum_{t=1}^T \alpha_t \ell_t(x)$.
%\end{equation}
%which is called \textsc{Follow-the-Regularized-Leader} (see e.g. \cite{R09,OO19}) in the literatuere.

\end{lemma}

\begin{proof}[Proof of Lemma~\ref{regret:Opt-FTRL}]
We can re-write the regret as
\begin{equation} \label{eq:term}
\begin{aligned}
\displaystyle  \regret{z}(z^*) & := \sum_{t=1}^T  \alpha_t  \ell_t(z_t) - \sum_{t=1}^T  \alpha_t  \ell_t(z^*)
\\ & 
=
\underbrace{
\sum_{t=1}^{T} \alpha_t \left( \ell_t(z_t) - \ell_t(w_{t+1}) \right) - \alpha_t \left(  m_t(z_t) -  m_t(w_{t+1}) \right)
}_{ \text{first term} }
% \langle x_t - w_{t+1} , \theta_t - m_t \rangle 
\\& \quad + 
\underbrace{
\sum_{t=1}^T \alpha_t \left( m_t(z_t) - m_t(w_{t+1}) \right)
}_{ \text{second term} }
% \sum_{t=1}^{T}  \langle x_t - w_{t+1} , m_t \rangle 
%\\ & 
\quad + 
\underbrace{
\sum_{t=1}^T \alpha_t \left(  \ell_t(w_{t+1})  - \ell_t(z^*) \right)
}_{ \text{third term} }.%\langle w_{t+1} - x^*, \theta_t \rangle
\end{aligned}
\end{equation}
In the following, we will denote 
\[
D_T:=\frac{1}{2} \left( \sum_{t=1}^T \left( \frac{\beta}{\eta} + \sum_{s=1}^{t-1} \alpha_s \mu_s  \right) \| z_t - w_t \|^2 
+ 
\sum_{t=1}^T \left( \frac{\beta}{\eta} + \sum_{s=1}^{t-1} \alpha_s \mu_s  + \alpha_t \hat{\mu}_t \right) \| z_t - w_{t+1} \|^2 \right) 
\]
 for brevity, and define $D_{0}:=0$.
Let us first deal with the second term and the third term on \eqref{eq:term}.
We will use induction to show that
\begin{equation} \label{inq:0}
\begin{aligned}
& \underbrace{
\sum_{t=1}^T \alpha_t \left( m_t(z_t) - m_t(w_{t+1}) \right)
}_{ \text{second term} }
% \sum_{t=1}^{T}  \langle x_t - w_{t+1} , m_t \rangle 
%\\ & 
\quad + 
\underbrace{
\sum_{t=1}^T \alpha_t \left(  \ell_t(w_{t+1})  - \ell_t(z^*) \right)
}_{ \text{third term} }
\\ & \leq \frac{1}{\eta} \left( R(z^*) - R(w_1) \right) 
- D_T
%- \left( \sum_{t=1}^T \big( \frac{\beta}{\eta} + \sum_{s=1}^t \alpha_s \mu_s  \big) \big( \| z_t - w_t \|^2 + \| z_t - w_{t+1} \|^2 \big) \right),
\end{aligned}
\end{equation}
for any point $z^* \in \ZZ$. % which includes the point $x^* = x^* \in \ZZ$.
For the base case $T=0$, we have 
%\begin{equation*}
%\begin{aligned}
%& 
$\sum_{t=1}^0 \alpha_t ( m_t(z_t) - m_t(w_{t+1}) )
+
\sum_{t=1}^0 \alpha_t (  \ell_t(w_{t+1})  - \ell_t(z^*) )
%\\ & 
= 0 \leq \frac{1}{\eta} ( R(z^*) - R(w_1) ) - 0,$
%\end{aligned}
%\end{equation*}
as $w_1 := \argmin_{z \in \ZZ} R(z)$. So the base case trivially holds.

Let us assume that the inequality (\ref{inq:0}) holds for $t=0,1, \dots, T-1$.
Now consider round $T$. We have
\begin{equation*}
\begin{aligned}
& \sum_{t=1}^T \alpha_t \left( m_t(z_t) - m_t(w_{t+1}) + \ell_t(w_{t+1}) \right)
\\ &
\overset{(a)}{\leq} 
\alpha_T \left( m_T(z_T) - m_T(w_{T+1}) + \ell_T(w_{T+1}) \right)
+ \frac{1}{\eta} \left( R(w_T) - R(w_1) \right)
-  D_{T-1}
\\ & \quad +  \sum_{t=1}^{T-1} \alpha_t \ell_t(w_T).
\\ &
\overset{(b)}{\leq} 
\alpha_T \left( m_T(z_T) - m_T(w_{T+1}) + \ell_T(w_{T+1}) \right)
+ \frac{1}{\eta} \left( R(z_T) - R(w_1) \right)
-  D_{T-1}
\\ & \quad 
- \frac{1}{2} \left( \frac{\beta}{\eta} + \sum_{s=1}^{T-1} \alpha_s \mu_s \right) \| z_T - w_T \|^2
+  \sum_{t=1}^{T-1} \alpha_t \ell_t(z_T)
\\ & =
\alpha_T \left( \ell_T(w_{T+1}) - m_T(w_{T+1}) \right)
+ \frac{1}{\eta} \left( R(z_T) - R(w_1) \right)
-  D_{T-1}
\\ & \quad 
- \frac{1}{2} \left( \frac{\beta}{\eta} + \sum_{s=1}^{T-1} \alpha_s \mu_s \right) \| z_T - w_T \|^2
+  \sum_{t=1}^{T-1} \alpha_t \ell_t(z_T) + \alpha_T m_T(z_T)
\\ & \overset{(c)}{\leq}
\alpha_T \left( \ell_T(w_{T+1}) - m_T(w_{T+1}) \right)
+ \frac{1}{\eta} \left( R(w_{T+1}) - R(w_1) \right)
-  D_{T-1}
\\ & \quad 
- \frac{1}{2} \left( \frac{\beta}{\eta} + \sum_{s=1}^{T-1} \alpha_s \mu_s \right) \| z_T - w_T \|^2
- \frac{1}{2} \left( \frac{\beta}{\eta} + \sum_{s=1}^{T-1} \alpha_s \mu_s + \alpha_T \hat{\mu}_T \right) \| z_T - w_{T+1} \|^2
\\ & \quad
+  \sum_{t=1}^{T-1} \alpha_t \ell_t(w_{T+1}) + \alpha_T m_T(w_{T+1})
\\ & =
\sum_{t=1}^T \alpha_t \ell_t (w_{T+1}) 
+ \frac{1}{\eta} \left( R(w_{T+1}) - R(w_1) \right)
- D_T
\\ & \overset{(d)}{ \leq } 
\sum_{t=1}^T \alpha_t \ell_t (z^*) 
+ \frac{1}{\eta} \left( R(z^*) - R(w_1) \right)
- D_T,
\end{aligned}
\end{equation*}
where (a) we use the induction such that the inequality (\ref{inq:0}) holds for any $z^* \in \ZZ$ including $z^* = w_T$, and (b) is because
\begin{equation*}
\sum_{t=1}^{T-1} \alpha_t \ell_t(w_T) + \frac{1}{\eta} R(w_T)
\leq 
\sum_{t=1}^{T-1} \alpha_t \ell_t(z_T) + \frac{1}{\eta} R(z_T)
- \frac{1}{2} \left( \frac{\beta}{\eta} + \sum_{s=1}^{T-1} \alpha_s \mu_s \right) \| z_T - w_T \|^2,
\end{equation*}
as $w_T$ is the minimizer of a $\frac{\beta}{\eta} + \sum_{t=1}^{T-1} \alpha_t \mu_t$ strongly convex function since
$w_{T}:= \argmin_{z \in \ZZ}  \sum_{s=1}^{T-1} \alpha_s \ell_s(z) + \frac{1}{\eta} R(z)$,
and (c) is because 
\begin{equation*}
\begin{aligned}
 \sum_{t=1}^{T-1} \alpha_t \ell_t(z_T) + \alpha_T m_T(z_T)  + \frac{1}{\eta} R(z_T)
%\\ & \qquad 
& \leq 
\sum_{t=1}^{T-1} \alpha_t \ell_t(w_{T+1}) + \alpha_T m_T(w_{T+1}) + \frac{1}{\eta} R(w_{T+1})
\\ & \quad 
- \frac{1}{2} \left( \frac{\beta}{\eta} + \sum_{s=1}^{T-1} \alpha_s \mu_s 
+ \alpha_T \hat{\mu}_T
\right) \| z_T - w_{T+1} \|^2,
\end{aligned}
\end{equation*}
as $z_T$ is the minimizer of a $\frac{\beta}{\eta} + \sum_{t=1}^{T-1} \alpha_t \mu_t + \alpha_T \hat{\mu}_T$ strongly convex function 
since
$z_T := \argmin_{{z \in \ZZ}} \left( \sum_{s=1}^{T-1} \alpha_{s} \ell_{s}(z) \right)
+ \alpha_T m_T(z)  + \frac{1}{\eta} R(z)$,
%(i.e. see the Optimistic FTRL update (\ref{tmp:optftrl})),
and (d) is due to $w_{T+1}:=\argmin_{z \in \ZZ} \sum_{t=1}^{T} \alpha_t \ell_t(z) +\frac{1}{\eta} R(z).$

\end{proof}

%\noindent
%We prove Lemma~\ref{regret:Opt-FTRL} at the end of this section. But let us first
%Let we will show how several 
%of the aforementioned guarantees--for 
%algorithms \FTL, \BTL, \OFTL, \FTRL, and \BTRL have their guarantees follow from Lemma

\subsubsection{Follow The Leader -- \FTL} \label{section:FTL}

\FTL, defined in \eqref{eq:alg_oco_ftl}, is perhaps the simplest strategy in online learning, it plays the best fixed action for the cumulative (weighted) loss seen in each previous round. It has been shown in various places (e.g. \cite{OO19,shalev2012online}) that \FTL exhibits logarithmic regret, at least when the loss functions are strongly convex. Our proof will ultimately be a consequence of a larger meta-result which we lay out in Section~\ref{sec:oftrl}.
\begin{lemma}{ ({\FTL}$[\xinit]$)} \label{regret:FTL}
Let $\alpha_1, \ldots, \alpha_T > 0$ be arbitrary, and assume we have a sequence of $\mu$-strongly convex loss functions $\ell_1(\cdot), \ldots, \ell_T(\cdot)$ for some $\mu \geq 0$. Then for any initial point $\xinit\in\ZZ$ the OCO procedure
{\FTL}$[\xinit]$ 
%((\eqref{eq:alg_oco_ftl} in Algorithm~\ref{alg:oco_batch_family}) 
satisfies
% \begin{align}\label{eq:FTL}
% \displaystyle  z_1 &= \xinit  \nonumber\\
% \displaystyle z_t &= \argmin_{z \in \ZZ} \sum_{s=1}^{t-1} \alpha_s \ell_s(z) 
% \end{align}
% and satisfies the following regret bound,
\begin{equation} \label{reg_FTL}
\regret{z}(z^*) 
\leq 
\sum_{t=1}^T
\frac{2 \alpha_t^2}{ \big( \sum_{s=1}^t \alpha_s \mu \big)  }
 \| \delta_t \|_*^2,
\end{equation}
where $\delta_{t} \in \partial \ell_{t}(z_t)$ and $z^* \in \ZZ$ is arbitrary.
\end{lemma}
%Later in Section~\ref{sec:oftrl} we show that the above guarantee  can be obtained by an easy application of  Lemma~\ref{regret:FTRL}.
%\begin{proof}
%Observe that FTL is actually FTRL when $R(\cdot)$ is a zero function. 
%Therefore, let $R(\cdot)=0$ and $\beta=0$ in Lemma~\ref{regret:FTRL}, we obtain the result.
%
%\end{proof}
By Lemma~\ref{regret:FTL},
when we set the weights uniformly, i.e. $\alpha_t = 1~\forall t$, and assume a bound on the gradient norms, i.e.  $\| \delta_t \|^2_* \leq G$, 
the uniform regret is
\begin{equation} \label{reg:FTL_unit}
%\displaystyle 
\textsc{Reg:=}\sum_{t=1}^T \ell_t(z_t) - \sum_{t=1}^T \ell_t(z^*) \leq \frac{G \log(T+1)}{2\mu},
\end{equation}
for any comparator $z_{*} \in \ZZ$.

It has been observed \cite{kalai2005efficient} that following the leader does not always lead to vanishing regret, for example when the loss function is linear, $\ell_t(\cdot):= \langle \theta_t, \cdot \rangle$, the regret can grow linearly with $T$ (see e.g. Example 2.2 \cite{shalev2012online}). 
On the other hand, if the constraint set $\ZZ$ is ``suitably round'', then logarithmic regret is achievable even for linear losses.

\begin{lemma}[Theorem 3.3 in \cite{HLGS16}] \label{thm:FTL}
Let $\{ \ell_t(\cdot):= \langle \theta_t, \cdot \rangle \}_{t=1}^T$
be any sequence of linear loss functions.
Denote $G: =  \max_{t\leq T} \| \theta_t \|$ and assume that 
the support function
$\Phi( \cdot):= \max_{z \in \ZZ}(z, \cdot)$ has a unique maximizer for each cumulative loss vector $L_{t}:= \sum_{s=1}^t  \theta_s$ at round $t$.
Define $\nu_T:= \min_{1\leq t \leq T} \| L_t \|$. 
Let $\ZZ \subset \reals^d$ be a $\lambda$-strongly convex set.  
Choose $\xinit \in \text{boundary}(\ZZ)$. 
Then {\FTL}$[\xinit]$ ensures,
\begin{equation}
\textsc{Reg:=}
\sum_{t=1}^T \ell_t(z_t) - \sum_{t=1}^T \ell_t(z^*)
\leq \frac{2 G^2}{\lambda \nu_T} \left( 1 + \log (T) \right),
\end{equation}
for any comparator $z^{*} \in \ZZ$.
%\jake{Jun-Kun, Shouldn't this be an inequality?}
\end{lemma}

\subsubsection{Be The Leader -- \BTL} \label{section:BTL}
\BTL (a.k.a. Be The Leader), defined in \eqref{eq:alg_oco_btl}, is similar to \FTL, except that the learner plays the best fixed action for the weighted cumulative loss seen thus far \emph{including the current round}. \footnote{The term ``Be The Leader'' was coined by \cite{kalai2005efficient}, who also showed it exhibits non-positive regret for linear loss functions.} Of course, in practical scenarios, we are unable to see the the loss on the current round before we select our action, but it is nevertheless useful for analytical purposes to consider such ``one-step lookahead'' algorithms; in each such case we superscript the algorithm's name with a $+$ symbol. In the the FGNRD framework, we can design dynamics in where the player that acts second in the protocol can indeed observe the loss function prior to choosing its action, and thus it is quite natural to consider such one-step lookahead OCO algorithms.
% \BTL is often used as analytic tool rather than a practical  algorithm. Nevertheless, note that in FGNRD (protocol~\ref{alg:game}) the $x$-player is allowed to view the current loss prior to playing, and can therefore apply \BTL.  
% This algorithm was named by \cite{kalai2005efficient}, who also proved that it actually guarantees non-positive regret. Here we provide a tighter bound.

\begin{lemma}{{(}\BTL )} \label{regret:BTL}
Let $\alpha_1, \ldots, \alpha_T > 0$ be arbitrary, and assume we have a sequence of $\mu$-strongly convex loss functions $\ell_1(\cdot), \ldots, \ell_T(\cdot)$ for some $\mu \geq 0$. Then the regret of \BTL satisfies
\begin{equation} \label{reg_BTL} 
\regret{z}(z^*) 
\leq - \sum_{t=1}^T \frac{ \mu A_{t-1} }{2} \|z_{t-1} - z_t \|^2 \leq 0,
\end{equation}
for any comparator $z^* \in \ZZ$. 
\end{lemma}

%\begin{proof}
%Observe that $\BTL$ is actually $\BTRL$ when $R(\cdot)$ is a zero function. 
%Therefore, let $R(\cdot)=0$ and $\beta=0$ in Lemma~\ref{regret:BTRL}, we obtain the regret of $\BTL$
%\begin{equation} 
%\begin{aligned}
%\displaystyle
%\regret{x} 
%\leq - \sum_{t=1}^T \frac{ \mu A_{t-1} }{2} \|x_{t-1} - x_t \|^2 \leq 0.
%\end{aligned}
%\end{equation} 
%
%\end{proof}

\subsubsection{Follow The Regularized Leader -- \FTRL}
FTRL, also called dual averaging in optimization literature \cite{X10} is a classic algorithm in online learning (see e.g. \cite{OO19,hazan2016introduction}). 
Looking at Equation~\eqref{eq:alg_oco_ftrl} in Algorithm~\ref{alg:oco_batch_family} one can notice that \FTRL is similar to \FTL with an additional \emph{Regularization term} $R(\cdot)$ scaled by parameter $1/\eta$.
The regularization term induces \emph{stability} into the decisions of the player, ensuring that consecutive decisions are close to each other, and this property is often crucial in order to ensure regret guarantees.
For example, in the case of linear loss functions, \FTRL, with appropriate choices of $\eta$ and $R(\cdot)$, can guarantee sublinear regret guarantees, while \FTL cannot.
In what follows we assume that $R(\cdot)$ is a $\beta$-strongly-convex function over $\ZZ$.

%Its update and corresponding guarantee can be easily derived from $\OFTRL$ too.

\begin{lemma}{({\FTRL}$[R(\cdot),\eta]$)} \label{regret:FTRL}
Given $\eta>0$ and a $\beta$-strongly convex $R : \ZZ \to \reals$, assume we have $\{ \alpha_t \ell_t(\cdot)  \}_{t=1}^T$ be a  sequence of loss functions  such that each $\ell_t(\cdot)$ is $\mu$-strongly convex for some $\mu \geq 0$.
Then {\FTRL}$[R(\cdot),\eta]$, which generates the sequence $z_1, \ldots, z_T \in \ZZ$, satisfies the following regret bound for any comparator $z^* \in \ZZ$,
 \begin{equation} \label{reg_FTRL}
\displaystyle 
\regret{z}(z^*) 
\leq 
\sum_{t=1}^T
\frac{2 \alpha_t^2}{ \big( \sum_{s=1}^t \alpha_s \mu \big) + \beta }
 \| \delta_t\|_*^2
 + 
 \frac{1}{\eta} \left( R(z^*) - R(z_1) \right),
\end{equation}
where $\delta_{t} \in \partial \ell_{t}(z_t)$.

\end{lemma}

\subsubsection{Be The Regularized Leader -- \BTRL} \label{section:BTRL}
\BTRL (\eqref{eq:alg_oco_btrl}) is very similar to \FTRL, with the difference that the former has an access to all past loss functions up to and \emph{including} the current round. 
Recall that in our FGNRD template, one of the players is allowed to viewed the current loss prior to playing, 
%Recall that in our FGNRD template (protocol~\ref{alg:game}) the $x$-player is allowed to view the current loss prior to playing, 
and can therefore apply \BTRL.

\begin{lemma}{({\BTRL}$[R(\cdot),1/\eta]$)} \label{regret:BTRL}
Given $\eta>0$ and a $\beta$-strongly convex $R : \ZZ \to \reals$, assume we have $\{ \alpha_t \ell_t(\cdot)  \}_{t=1}^T$ be a  sequence of loss functions such that each $\ell_t(\cdot)$ is $\mu$-strongly convex for some $\mu \geq 0$. 
Then {\BTRL}$[R(\cdot),\eta]$ 
% is defined as follows,
% \begin{equation} \label{update:BTRL}
% z_t \leftarrow \argmin_{{z \in \ZZ}} \sum_{s=1}^t \alpha_{s} \ell_{s}(z) + \frac{1}{\eta} R(z),
% \end{equation}
satisfies the following regret bound for any comparator $z^* \in \ZZ$,
\begin{equation} \label{regBTRL}
\begin{aligned}
\displaystyle
\regret{z}(z^*) 
\leq \frac{R(z^*) - R(z_0)}{\eta} - \sum_{t=1}^T \left(  \frac{ \mu A_{t-1} }{2} + \frac{\beta}{2 \eta} \right) \|z_{t-1} - z_t \|^2,
%\leq R(x^*) - R(x_0) - \sum_{t=1}^T \frac{ \mu A_{t-1} + \beta}{2} \|x_{t-1} - x_t \|^2,
\end{aligned}
\end{equation} 
where $z_0 = \min_{z \in \ZZ} R(z)$.% and $z^*$ is any comparator in $\ZZ$.
\end{lemma}

%\begin{proof}
%Observer that the \BTRL update is exactly equivalent to
%\OFTL with $m_t(\cdot) = \ell_t(\cdot)$.
%Furthermore, 
%$z_{t+1}$ in Lemma~\ref{regret:Opt-FTRL} is actually $x_t$ of \BTRL
%shown on (\ref{update:BTRL}).
%So term (A) and term (D) on (\ref{eq:uni})
%in Lemma~\ref{regret:Opt-FTRL} is $0$,
%Therefore, \BTRL regret satisfies
%\begin{equation} \label{reg_x_acc0}
%\begin{aligned}
%\displaystyle
%\regret{x} 
%\leq R(x^*) - R(x_0) - \sum_{t=1}^T \big(  \frac{ \mu A_{t-1} }{2} + \frac{\beta}{2 \eta} \big) \|x_{t-1} - x_t \|^2.
%\end{aligned}
%\end{equation} 
%
%\end{proof}
%
%

\subsubsection{Optimistic Follow the Leader -- \OFTL} \label{section:OFTL}
In the previous subsection, we have seen that \BTL uses the knowledge of the loss function at rounds $t$ in order to ensure  negative regret. While this knowledge is oftentimes unavailable, one can often access a ``hint" or ``guess''  function $m_t(\cdot)$  that approximates $\ell_t(\cdot)$ prior to choosing an action $z_t$. As can be seen from Equation~\eqref{eq:alg_oco_oftl} in Algorithm~\ref{alg:oco_batch_family} and Lemma~\ref{regret:Opt-FTL}, \OFTL makes use of the availability of such hints in order to provide better guarantees.
The next statement shows that when we have ``good" hints, in the sense that $m_t(\cdot)\approx \ell_t(\cdot)$, then \OFTL obtains improved guarantees compared to standard \FTL.

\begin{lemma}{({\OFTL})} \label{regret:Opt-FTL}
% \junkun{I think $z_1 \leftarrow \xinit$ here has an issue, as its causes a conflict in the second line. Perhaps {\OFTL}$[m_t(\cdot)]$ is better?} 
Let $\{ \alpha_t \ell_t(\cdot)  \}_{t=1}^T$ be a  sequence of loss functions  such that each $\ell_t(\cdot)$ is $\mu_t$-strongly convex, let $m_1(\cdot), \ldots, m_T(\cdot): \ZZ \to \reals$ be an arbitrary sequence of convex functions.
Assume that we select our sequence of actions $z_1, \ldots, z_T$ according to OCO algorithm {\OFTL}, then for any $z^{*} \in \ZZ$, we have the following regret bound 
% Given an initial point $\xinit = \argmin_{z \in \ZZ} m_1(\cdot)$, 
% {\OFTL}$[\xinit]$ is defined as follows,
% \begin{align} \label{tmp:oftl}
% z_1& \leftarrow \xinit \nonumber \\
% z_t &\leftarrow \argmin_{{z \in \ZZ}} \left( \sum_{s=1}^{t-1} \alpha_{s} \ell_{s}(z) \right)
% + \alpha_t m_t(z),
% \end{align}
% where  $m_t(\cdot)$ is the \emph{hint (or the guess)} for the loss function $\ell_t(\cdot)$.
% %\junkun{Otherwise, we should add $\xinit = \argmin_{z \in \ZZ} m_t(\cdot)$. }
% \\
% {\OFTL} satisfies,
\begin{equation} \label{eq:OTFL}
\begin{aligned}
& \regret{z}(z^*)   \leq
 \sum_{t=1}^{T}  \alpha_t \left( \ell_t(z_t) - \ell_t(w_{t+1}) \right) - \alpha_t \left(  m_t(z_t) -  m_t(w_{t+1}) \right),
\end{aligned}
\end{equation}
where we assume the sequence $w_2, \ldots, w_T$ satisfies
$w_{t} \gets \argmin_{z \in \ZZ} \sum_{s=1}^{t-1} \ell_s(z)$.
%{\FTL}$[\xinit]$, i.e. $w_{t}:= \argmin_{z \in \ZZ} \sum_{s=1}^{t-1} \ell_s(z)$ and $w_{1} \gets \xinit = \argmin_{z \in \ZZ} m_1(z)$.
\end{lemma}

%\begin{proof}
%Observe that \OFTL is actually \OFTRL when $R(\cdot)$ is a zero function. 
%Therefore, let $R(\cdot)=0$ in Lemma~\ref{regret:Opt-FTRL}
%and drop term (C) and (D) in (\ref{eq:uni}) as they are non-positive, we obtain the result.
%
%\end{proof}

\subsubsection{Proof of \FTL, \BTL, \FTRL, \BTRL, \OFTL} \label{lem:various}

Having described all the online learning algorithms in Algorithm~\ref{alg:oco_batch_family},
we are ready to show how their guarantees can be naturally obtained from that of \OFTRL.
The way that we derive them, i.e., showing the regret bound of \OFTRL and then obtaining the rest of them as special cases, to our knowledge, is new, while we acknowledge that various versions of proofs of them can be found in the literature, e.g., \cite{S07,OO19,hazan2016introduction}.

\begin{proof}[Proof of Lemma~\ref{regret:Opt-FTL} on  \OFTL]
Observe that \OFTL is actually \OFTRL when $R(\cdot)$ is a zero function. 
Therefore, let $R(\cdot)=0$ in Lemma~\ref{regret:Opt-FTRL}
and drop term (C) and (D) in (\ref{eq:uni}) as they are non-positive, we obtain the result.

\end{proof}

\begin{proof}[Proof of Lemma~\ref{regret:BTRL} on  \BTRL]
Observe that the \BTRL update is exactly equivalent to
\OFTRL with $m_t(\cdot) = \ell_t(\cdot)$.
Furthermore, 
$w_{t+1}$ in Lemma~\ref{regret:Opt-FTRL} is actually $z_t$ of \BTRL
shown on \eqref{eq:alg_oco_btrl}.
So term (A) and term (D) on (\ref{eq:uni})
in Lemma~\ref{regret:Opt-FTRL} is $0$,
Therefore, \BTRL regret satisfies
\begin{equation} \label{reg_x_acc0}
\begin{aligned}
\displaystyle
\regret{z}(z^*) 
\leq \frac{ R(z^*) - R(z_0) }{\eta} - \sum_{t=1}^T \left(  \frac{ \mu A_{t-1} }{2} + \frac{\beta}{2 \eta} \right) \|z_{t-1} - z_t \|^2.
\end{aligned}
\end{equation} 
\end{proof}

\begin{proof}[Proof of Lemma~\ref{regret:BTL} on  \BTL]
Observe that $\BTL$ is actually $\BTRL$ with $R(\cdot)=0$. 
Therefore, let $R(\cdot)=0$ and $\beta=0$ in Equation~\eqref{reg_x_acc0}, we obtain the regret of $\BTL$
\begin{equation*} 
\begin{aligned}
\displaystyle
\regret{z}(z^*) 
\leq - \sum_{t=1}^T \frac{ \mu A_{t-1} }{2} \|z_{t-1} - z_t \|^2 \leq 0.
\end{aligned}
\end{equation*} 
\end{proof}

For the regret guarantee of \FTRL, we will need the following supporting lemma.
\begin{lemma} [Lemma 5 in \cite{koren2015fast}] \label{lem:diff}
Let $\psi_1(\cdot), \psi_2(\cdot): \ZZ \rightarrow \reals$ be two convex functions defined over a closed and convex domain. 
Denote $u_1 := \argmin_{z \in \ZZ} \psi_1(z)$ and 
$u_2 := \argmin_{z \in \ZZ} \psi_2(z)$.
Assume that $\psi_2$ is $\sigma$-strongly convex with respect to a norm $\| \cdot \|$.
Define $\phi(\cdot):= \psi_2(\cdot) - \psi_1(\cdot)$.
Then,
\begin{equation}
\| u_1 - u_2 \| \leq \frac{2}{\sigma} \| \delta \|_*, 
\text{ where } \delta \in \partial \phi(u_1).
\end{equation}
%where .
Furthermore, if $\phi(\cdot)$ is convex, then,
\begin{equation}
0 \leq \phi(u_1) - \phi(u_2) \leq \frac{2}{\sigma} \| \delta \|_*^2, 
\text{ where } \delta \in \partial \phi(u_1).
\end{equation}

\end{lemma}

\begin{proof}[Proof of Lemma~\ref{regret:FTRL} on  \FTRL]
Observe that \FTRL is actually \OFTRL where 
$m_t(\cdot)=0~\forall t$.
Therefore, let $m_t(\cdot)=0$ in Lemma~\ref{regret:Opt-FTRL}, we obtain the regret of \FTRL,
\begin{equation} \label{tmpeq:FTRL0}
\begin{aligned}
%\\ &
\regret{z}(z^*) 
 \leq
 \sum_{t=1}^{T}  \alpha_t \left( \ell_t(z_t) - \ell_t(z_{t+1}) \right)  + \frac{1}{\eta} \left( R(z^*) - R(z_1) \right), 
\end{aligned}
\end{equation}
where we have dropped term (C) and term (D) on (\ref{eq:uni}) since they are non-positive, and we also note that $w_t$ in Lemma~\ref{regret:Opt-FTRL} is the same as $z_t$ here.

To continue, we use Lemma~\ref{lem:diff}.
Specifically, in Lemma~\ref{lem:diff}, we let $\psi_1(\cdot) \leftarrow \sum_{s=1}^{t-1} \alpha_s \ell_s(\cdot) + \frac{1}{\eta} R(\cdot)$ and $\psi_2(\cdot) \leftarrow \sum_{s=1}^t \alpha_s \ell_s(\cdot) + \frac{1}{\eta} R(\cdot)$.
Then, we have $\phi(\cdot) = \alpha_t \ell_t(\cdot)$,
$u_1 = z_t$, $u_2 = z_{t+1}$ and that $\sigma = \sum_{s=1}^t \alpha_s \mu + \beta$. So by Lemma~\ref{lem:diff} below, we have 
\begin{equation} \label{tmpeq:FTRL1}
 \alpha_t \left( \ell_t(z_t) - \ell_t(z_{t+1}) \right)  
 \leq \frac{2 \alpha_t^2}{ \big( \sum_{s=1}^t \alpha_s \mu \big) + \beta }
 \| \delta_t \|_*^2,
\end{equation}
where $\delta_{t} \in \partial \ell_t(z_t ) $.
Combining (\ref{tmpeq:FTRL0}) and (\ref{tmpeq:FTRL1}) leads to the result.
\end{proof}

\begin{proof}[Proof of Lemma~\ref{regret:FTL} on  \FTL]
 Observe that FTL is actually FTRL with $R(\cdot)=0$ . 
Therefore, let $R(\cdot)=0$ and $\beta=0$ in Lemma~\ref{regret:FTRL}, we obtain the result.

\end{proof}

% Next, in Subsections~\ref{section:BR} and~\ref{section:MD}, we go on by presenting two additional online learning algorithms that cannot be captured by the  \OFTRL Meta-algorithm.

%We now complete this section by proving the meta result for \OFTRL.

\subsubsection{Mirror Descent -- \OMD} \label{section:OMD}

%When the learner is not able to see the loss function $\ell_{t}(\cdot)$ beforehand, they can use 
\OMD (\eqref{eq:alg_oco_omd} in Algorithm~\ref{alg:oco_update_family}) 
is another popular algorithm in online learning.
Compared to those in Algorithm~\ref{alg:oco_batch_family},
\OMD updates the current iterate using only the latest loss function instead of the cumulative loss functions seen so far. Yet, it still has a nice regret guarantee.
%for updating its action,
% is defined as follows. 
%\begin{equation} \label{eq:OMDupdate}
%\displaystyle  z_t  :=  \displaystyle \argmin_{z \in \ZZ}   \alpha_{t-1} \ell_{t-1}(z) +  \frac{1}{\gamma} \V{z_{t-1}}(z)~;\qquad \textbf{(\OMD$[\phi(\cdot),\gamma]$)}
%   =  \argmin_{x \in \K}   \gamma \langle x, \alpha_t \theta_t \rangle  + V_{x_{t-1}}(x),
%\end{equation}

\begin{lemma}{(\OMD$[\phi(\cdot),z_0, \gamma]$)}\label{lem:OMD}
Assume that $\phi(\cdot)$ is $\beta$-strongly convex w.r.t $\| \cdot \|$.
% and $D := \sup_{z_0,z^* \in \ZZ} \V{z_0}(z^*) < \infty$.
%denotes the minimizer of $f(\cdot)$. 
For any sequence of proper lower semi-continuous convex 
loss functions $\{ \alpha_t \ell_t(\cdot) \}_{t=1}^T$,
%For any sequence of linear loss functions $\{ \ell_t(\cdot):= \alpha_t \langle \theta_t, \cdot \rangle \}_{t=1}^T$,
\OMD$[\phi(\cdot),z_0,\gamma]$ satisfies the following regret bound
for any comparator $z^* \in \ZZ$,
\begin{align}
   \regret{z}(z^*) 
   %& \leq \frac{1}{\gamma} \V{z_0}(z^*)
   %- \frac{1}{\gamma} \sum_{t=1}^T \V{z_{t+1}}(z_t)
%+ \sum_{t=1}^T \langle \alpha_t  \nabla \ell_t(z_t), z_t - z_{t+1} \rangle
%\\ & 
\leq 
\frac{1}{\gamma} \V{z_1}(z^*) 
+ \frac{\gamma}{2 \beta}   \sum_{t=1}^T \| \alpha_t \delta_t \|^2_*,
\end{align}
where $\delta_{t} \in  \partial \ell_t(z_t )$.
\end{lemma}

The result in Lemma~\ref{lem:OMD} is quite well-known, 
see e.g. \cite{OO19,hazan2016introduction},
nevertheless, we replicate the proof in Appendix~\ref{app:lem:MD} for completeness.

\subsubsection{Prescient Mirror Descent -- \MD} \label{section:MD}

%For any sequence of proper lower semi-continuous convex functions $\{ \alpha_t \ell_t(\cdot) \}_{t=1}^T$,
%For any sequence of linear loss functions $\{ \ell_t(\cdot):= \alpha_t \langle \theta_t, \cdot \rangle \}_{t=1}^T$,
%consider that the player uses 
%\MD for updating its action, which is defined as follows. 
%\begin{equation} \label{eq:MDupdate}
%\displaystyle  z_t  :=  \displaystyle \argmin_{z \in \ZZ}   \alpha_t \ell_t(z) +  %\frac{1}{\gamma} \V{z_{t-1}}(z)~;\qquad \textbf{(\MD$[\phi(\cdot),\gamma]$)}
%   =  \argmin_{x \in \K}   \gamma \langle x, \alpha_t \theta_t \rangle  + V_{x_{t-1}}(x),
%\end{equation}
%where we recall that the Bregman divergence $\V{z}(\cdot)$ is with respect to a strongly convex distance generating function $\phi(\cdot)$(see Equation~\eqref{eq:BregmanDef}). 
%Note that in the above definition of  
\MD (\eqref{eq:alg_oco_md} in Algorithm~\ref{alg:oco_update_family}) assumes that the online player is \emph{prescient}, i.e., it knows the loss functions $\ell_t$ prior to choosing $z_t$. Recall that in FGNRD (protocol~\ref{alg:game}) one of the players is allowed to view the current loss prior to playing, and can therefore apply \MD.  

%Also, we note that the $x$-player has an advantage in this  dynamics, since $x_t$ is chosen \emph{with knowledge of} %$y_t$ and hence has knowledge of 
%the incoming loss $\ell_t(\cdot)$ (again recall that our FGNRD enables the $x$-player 

\begin{lemma}{(\MD$[\phi(\cdot),z_0,\gamma]$)}\label{lem:MD}
Assume that $\phi(\cdot)$ is $\beta$-strongly convex.
% and $D := \sup_{z_0,z^* \in \ZZ} \V{z_0}(z^*) < \infty$.
%denotes the minimizer of $f(\cdot)$. 
For any sequence of proper lower semi-continuous convex 
loss functions $\{ \alpha_t \ell_t(\cdot) \}_{t=1}^T$,
%For any sequence of linear loss functions $\{ \ell_t(\cdot):= \alpha_t \langle \theta_t, \cdot \rangle \}_{t=1}^T$,
\MD$[\phi(\cdot),z_0,\gamma]$
satisfies the following regret bound
for any comparator $z^* \in \ZZ$,
% (Equation~\eqref{eq:MDupdate})
%is bounded as follows,
$$
   \regret{z}(z^*) \leq  
   \frac{1}{\gamma} \V{z_{0}}(z^*) - \sum_{t=1}^{T} \frac{\beta}{2 \gamma} \| z_{t-1} - z_t \|^2
%\leq   \frac {D }{\gamma} - \sum_{t=1}^{T} \frac{\beta}{2 \gamma} \| z_{t-1} - z_t \|^2.
 $$
\end{lemma}

We will need the following supporting lemma for proving Lemma~\ref{lem:MD} below.
\begin{lemma}[Property 1 in \cite{T08}]  \label{lem:newMD}
For any proper lower semi-continuous convex function $\theta(z)$,
let $z^+ =  \argmin_{z \in \ZZ} \theta(z) + \V{c}(z)$. 
Then, it satisfies that for any $z^* \in \ZZ$,
\begin{equation} 
\displaystyle \theta(z^+) - \theta(z^*) \leq \V{c}(z^*) - \V{z^+}(z^*) - \V{c}(z^+).
\end{equation}
\end{lemma}

\begin{proof}[Proof of Lemma~\ref{lem:MD}]
The key inequality we need is Lemma~\ref{lem:newMD};
%Lemma~\ref{lem:rs13};
using the lemma with $\theta(z) = \gamma \alpha_t \ell_t(z) $, $z^+= z_t$ and $c=z_{t-1}$  we have 
\begin{equation} \label{ttb1}
\displaystyle \gamma \alpha_t \ell_t(z_t)  - \gamma  \alpha_t \ell_t(z^*)  =  \theta(z_t) - \theta(z^*) \leq \V{z_{t-1}}(z^*) - \V{z_t}(z^*) - \V{z_{t-1}}(z_t).
\end{equation}
Therefore, the regret with respect to any comparator $z^* \in \ZZ$ can be bounded as
\begin{equation*}
\begin{aligned}
 \displaystyle  \regret{z}(z^*) &:= \sum_{t=1}^T  \alpha_t  \ell_t(z_t) -  \sum_{t=1}^T  \alpha_t  \ell_t(z^*)
\\& \displaystyle  \overset{(\ref{ttb1})}{\leq}  \sum_{t=1}^{T} \frac{1}{\gamma} \big( \V{z_{t-1}}(z^*) - \V{z_t}(z^*) - \V{z_{t-1}}(z_t) \big) 
\\ & \displaystyle
=  \frac{1}{\gamma} \V{z_{0}}(z^*) - \frac{1}{\gamma} \V{z_T}(z^*)  + \sum_{t=1}^{T-1} \left( \frac{1}{\gamma} - \frac{1}{\gamma} \right) \V{z_t}(z^*) - \frac{1}{\gamma} \V{z_{t-1}}(z_t) 
\\ & \displaystyle \leq \frac{1}{\gamma} \V{z_{0}}(z^*) - \sum_{t=1}^{T} \frac{\beta}{2 \gamma} \| z_{t-1} - z_t \|^2,
\end{aligned}
\end{equation*}
where the last inequality uses 
%the definition of $D$ and 
the strong convexity of $\phi$, which grants $\V{z_{t-1}}(z_t) \geq \frac \beta 2 \| z_t - z_{t-1}\|^2$. 
\end{proof}

\iffalse
\begin{proof}
The result is quite well-known and also appeared in (e.g.
\cite{CT93} and \cite{LLM11}).
For completeness, we replicate the proof here.
Recall that the Bregman divergence with respect to the distance generating function $\phi(\cdot)$ at 
a point $c$ is: $\V{c}(z):= \phi(z) - \langle \nabla \phi(c), z - c  \rangle  - \phi(c).$ 

Denote $F(z):= \theta(z) + \V{c}(z)$.
Since $z^+$ is the optimal point of 
$\min_{z \in \ZZ} F(z)\displaystyle,$
by optimality,
\begin{equation} \label{m1}
\displaystyle \langle z^* - z^+ ,  \nabla F(z^+) \rangle =  
\langle z^* - z^+ , \partial \theta(z^+) + \nabla \phi (z^+) - \nabla \phi(c) \rangle  \geq 0,
\end{equation}
for any $z^{*} \in \ZZ$.
Now using the definition of subgradient, we have 
\begin{equation} \label{m2}
\displaystyle \theta(z^*) \geq \theta(z^+) + \langle \partial \theta(z^+) , z^* - z^+ \rangle.
\end{equation}
By combining (\ref{m1}) and (\ref{m2}), we have 
\begin{equation} 
\begin{aligned}
 \displaystyle   \theta(z^*) & \displaystyle \geq \theta(z^+) + \langle \partial \theta(z^+) , z^* - z^+ \rangle.
\\ &  \displaystyle \geq \theta(z^+) + \langle z^* - z^+ , \nabla \phi (c) - \nabla \phi(z^+) \rangle.
\\  & \displaystyle = \theta(z^+) - \{ \phi(z^*) - \langle \nabla \phi(c), z^* - c  \rangle  - \phi(c) \}
\\&
+ \{ \phi(z^*) - \langle \nabla \phi(z^+), z^* - z^+  \rangle  - \phi(z^+) \}
\\  & \displaystyle + \{ \phi(z^+) - \langle \nabla \phi(c), z^+ - c  \rangle  - \phi(c) \}
\\  & \displaystyle  = \theta(z^+) - \V{c}(z^*) + \V{z^+}(z^*) +  \V{c}(z^+).
\end{aligned}
\end{equation}
\end{proof}
%}
\fi

\subsubsection{Optimistic Mirror Descent -- \OPTMD} \label{section:OPTMD}

When a hint of the loss function at the beginning of each round is available,
the online learner can exploit the hint in the hope to get a better regret.
Denote $m_{t}$ a vector that represents the hint before the learner outputs a point at $t$. \OPTMD in Algorithm~\ref{alg:oco_update_family} has the following regret guarantee.
%\cite{RS13} propose \OPTMD, defined by the following updates
%\begin{align} %\label{eq:OPTMDupdate}
%\displaystyle  z_t  & :=  \displaystyle \argmin_{z \in \ZZ}   \alpha_{t} 
%\langle z,  m_t \rangle + \frac{1}{\gamma} \V{z_{t-\frac{1}{2}}}(z)~ \label{opt-1}
%\qquad \textbf{(\OMD$[\phi(\cdot),\gamma]$)} 
%\\ \displaystyle  z_{t+\frac{1}{2}} & :=  \displaystyle \argmin_{z \in \ZZ}   \alpha_{t}
%\langle z,  \nabla \ell_{t}(z_t) \rangle +  \frac{1}{\gamma} \V{z_{t-\frac{1}{2}}}(z)~.
%\label{opt-2}
%\qquad \textbf{(\OMD$[\phi(\cdot),\gamma]$)}
%\end{align}
 
\begin{lemma}{(\OPTMD$[\phi(\cdot),z_0,\gamma]$)}\label{lem:OPTMD}
Assume that $\phi(\cdot)$ is $\beta$-strongly convex w.r.t $\| \cdot \|$.
% and $D := \sup_{z_0,z^* \in \ZZ} \V{z_0}(z^*) < \infty$.
For any sequence of proper lower semi-continuous convex 
loss functions $\{ \alpha_t \ell_t(\cdot) \}_{t=1}^T$,
%the weighted regret of 
\OPTMD$[\phi(\cdot),z_0,\gamma]$ 
satisfies the following regret bound
for any comparator $z^* \in \ZZ$,
%(Equation~\eqref{eq:OPTMDupdate})
%is bounded as follows,
\begin{align}
   \regret{z}(z^*) 
\leq 
\frac{1}{\gamma} \V{z_0}(z^*) 
+ \frac{\gamma}{2 \beta}   \sum_{t=1}^T \alpha_t^2 \| \delta_t -  m_t \|^2_*,
\end{align}
where $\delta_{t} \in \partial \ell_{t}(z_t)$.
\end{lemma}
%We refer the reader to \cite{RS13} for the applications of \OPTMD. 
\OPTMD were proposed by \cite{CJ12,RK13}, to our knowledge.
We replicate the proof of Lemma~\ref{lem:OPTMD} in Appendix~\ref{app:lem:MD} for completeness.

\iffalse
\begin{algorithm}[h] 
   \caption{ Frank-Wolfe \cite{frank1956algorithm} } \label{alg:fw}
Given: $L$-smooth $f(\cdot)$, convex domain $\K$, arbitrary $w_0$, iterations $T$.
\begin{center} 
\begin{tabular}{c c}
    $
      \boxed{\small
      \begin{array}{rl}
        \gamma_t & \gets \frac{2}{t+1}\\
%        (B): \gamma_t & \gets \frac{1}{t}\\        
        v_t  & \gets \displaystyle \argmin_{v \in \K} \langle v, \nabla f(w_{t-1})  \rangle \\
        w_{t} & \gets (1 - \gamma_t) w_{t-1} + \gamma_t v_t
      \end{array}
      }
    $
    & \small
    $\boxed{\small
    \begin{array}{rl}
      g(x,y) & := \langle x, y \rangle - f^*(y)\\
      \alpha_t & \gets t\\
%      (B): \alpha_t & \gets 1\\
      \alg^Y & := \FTL[\nabla f(w_0)] \\
      \alg^X & := \BR \\
    \end{array}
    }$
  \\
    \small Iterative Description & 
    \small FGNRD Equivalence
\end{tabular}
\end{center}
Output: $w_T = \bar x_T$
\end{algorithm}
\fi

\subsubsection{Discussion about the decision space}

Before closing this section, we would like to discuss whether the decision space $\ZZ$ of the aforementioned online learning algorithms can indeed be unconstrained without incurring a potentially vacuous regret.
For those algorithms that are prescient (\BR,\BTL,\BTRL,\MD), their decision spaces can be unconstrained, i.e. $\ZZ = \reals^{d}$, while enjoying either a non-positive regret or a constant regret that is independent of the horizon $T$.
%This is the property that allows us to get
On the other hand, for those that are unable to see their loss functions before outputting an action, e.g, \FTL, \FTRL, or \OMD, their decision spaces might need to be a compact convex set to avoid a vacuous regret bound. Take Lemma~\ref{lem:OMD} of \OMD as an example, its regret bound consists of a sum of square gradient norms, and hence one might require the decision space $\ZZ$ to be a compact convex set instead of the unconstrained one $\reals^{d}$ to further bound the gradient norms.

%the decision space 
%Having introduced the online learning algorithms and their regret guarantees, we would like

\section{Recovery of existing algorithms}
\label{sec:ExistingAlgs}
% \begin{center}
% \begin{tabular}{c c}
%   \hline
%   \multicolumn{2}{c}{
%       Frank-Wolfe - $\epsilon_t = O\left(\frac{LD}{T}\right)$
%   }
%   \\\hline
%     Iterative Description & 
%     FGNRD Equivalence
%   \\ \small
%     $
%       % \left[
%       \boxed{\small\begin{array}{rl}
%         \gamma_t & \gets \frac{1}{t+2}\\
%         v_t  & \gets \displaystyle \argmin_{v \in \K} \langle v, \nabla f(w_{t-1})  \rangle \\
%         w_{t} & \gets (1 - \gamma_t) w_{t-1} + \gamma_t v_t
%       \end{array}}
%       % \right.
%     $
%     & \small
%     $ \boxed{\small
%     % \left\{
%     \begin{array}{rl}
%       \alg^Y & := \FTL \\
%       \alg^X & := \BR \\
%       \alpha_t & \gets t\\
%     \end{array}
%     % \right\}
%     }$
% \end{tabular} 
% \end{center}

What we are now able to establish, using the tools developed above, is that several iterative first order methods to minimize a convex function can be cast as simple instantiations of the Fenchel game no-regret dynamics. But more importantly, using this framework and the various regret bounds stated above, we able to establish a convergence rate for each via a unified analysis.

For every one of the optimization methods we explore below we provide the following:
\begin{enumerate}
  \item We state the update method described in its standard iterative form, alongside an equivalent formulation given as a no-regret dynamic. To provide the FGNRD form, we must specify the payoff function $g(\cdot,\cdot)$--typically the Fenchel game, with some variants---as well as the sequence of weights $\alpha_t$, and the no-regret algorithms $\alg^Y,\alg^X$ for the two players.
  \item We provide a proof of this equivalence, showing that the FGNRD formulation does indeed produce the same sequence of iterates as the iterative form; this is often deferred to the appendix.
  \item Leaning on Theorem~\ref{thm:meta}, we prove a convergence rate for the method.
\end{enumerate}

\subsection{Frank-Wolfe method and its variants}

The \emph{Frank-Wolfe method} (FW) \cite{frank1956algorithm}, also known as \emph{conditional gradient}, is known for solving constrained optimization problems. FW is entirely first-order, while requiring access to a linear optimization oracle.
Specifically,
given a compact and convex constraint set $\K \subset \reals^d$, FW relies on the ability to (quickly) answer queries of the form $\argmin_{x \in \K} \langle x, v\rangle$, for any vector $v \in \reals^d$.
In many cases this linear optimization problem is much faster for well-behaved constraint sets; e.g. simple convex polytopes, the PSD cone, and various balls defined by vector and matrix norms 
\cite{D16a,CP21}. 
%\cite{D16a,D16b,BPZ19}.
When the constraint set is the nuclear norm ball, which arises in matrix completion problems,
then the linear optimization oracle corresponds to computing a top singular vector, which requires time roughly linear in the size of the matrix 
%  which takes time $O(d_1 d_2)$ for approximating a top singular vector in practice;
% while projecting to the nuclear norm ball requires a singular value decomposition (SVD), which takes time $O(d_1 d_2 \max\{d_1,d_2\})$
% for approximating solving SVD
%  in practice 
\cite{CP21}.

\begin{algorithm}[H] 
   \caption{ Frank-Wolfe \cite{frank1956algorithm} } \label{alg:fw}
Given: $L$-smooth $f(\cdot)$, convex domain $\K \subset \reals^d $, arbitrary $w_0$, iterations $T$.
\begin{center} 
\begin{tabular}{c c}
    $
      \boxed{\small
      \begin{array}{rl}
        \gamma_t & \gets \frac{2}{t+1}\\
%        (B): \gamma_t & \gets \frac{1}{t}\\        
        v_t  & \gets \displaystyle \argmin_{v \in \K} \langle v, \nabla f(w_{t-1})  \rangle \\
        w_{t} & \gets (1 - \gamma_t) w_{t-1} + \gamma_t v_t
      \end{array}
      }
    $
    & \small
    $\boxed{\small
    \begin{array}{rl}
      g(x,y) & := \langle x, y \rangle - f^*(y)\\
      \alpha_t & \gets t\\
%      (B): \alpha_t & \gets 1\\
      \alg^Y & := \FTL[\nabla f(w_0)] \\
      \alg^X & := \BR \\
    \end{array}
    }$
  \\
    \small Iterative Description & 
    \small FGNRD Equivalence
\end{tabular}
\end{center}
Output: $w_T = \bar x_T$
\end{algorithm}

We describe the Frank-Wolfe method precisely in Algorithm~\ref{alg:fw}, in both its iterative form and its FGNRD interpretation. We begin by showing that these two representations are equivalent.
% There are two choices of the sequence of learning rate $\{\gamma_t\}$ in the literature (see e.g. \cite{B20}): one is $\gamma_t = \frac{1}{t}$, which leads to $O(\frac{\log T}{T})$ convergence rate for smooth convex problem; the other is $\gamma_t = \frac{2}{t+2}$, which leads to $O(\frac{1}{T})$ rate.

% \begin{algorithm}[H] 
%    \begin{algorithmic}[1]
%    \caption{Standard Frank-Wolfe algorithm}\label{alg:fw}
% \State \textbf{Input:} learning rate $(A) \text{ } \gamma_t = \frac{1}{t}$ or $(B) \text{ } \gamma_t = \frac{2}{t+1} $.
% \State Initialize $w_{0} \in \K$.
% \For{$t=1,2, 3 \dots , T$}
% \State $\displaystyle v_t = \argmin_{v \in \K} \langle v, \nabla f(w_{t-1})  \rangle $.
% \State $w_{t} = (1 - \gamma_t) w_{t-1} + \gamma_t v_t$.
% \EndFor
% \State Output: $w_T$
% \end{algorithmic}
% \end{algorithm}

% Now we are going show that the Frank-Wolfe method can be recovered from the zero-sum game framework, Algorithm~\ref{alg:game} with specific online learning strategies.

\begin{proposition} \label{thm:equivFW}
  The two interpretations of Frank-Wolfe, as described in Algorithm~\ref{alg:fw}, are equivalent. That is, for every $t$, the iterate $w_t$ computed iteratively on the left hand side is identically the weighted-average point $\bar x_t$ produced by the dynamic on the right hand side.
 % When both are run for exactly $T$ rounds, the output $\bar x_T$ of Algorithm~\ref{alg:game} with the weighting scheme $\{\alpha_t=1\}$ (respectively, $\{\alpha_t = t \}$) is identically the output $w_T$ of Algorithm~\ref{alg:fw} with learning rate $\gamma_t = \frac{1}{t}$ (respectively, $\gamma_t = \frac{2}{t+1}$)
 % as long as: \textbf{(I)} Initialize $y_1$ in Alg.~\ref{alg:game} equals $\nabla f(w_0)$ in Alg.~\ref{alg:fw}; \textbf{(II)} Alg.~\ref{alg:game} sets $\alg^Y := \FTL$; \textbf{(III)}  Alg.~\ref{alg:game} sets $\alg^X := \BR$.
\end{proposition}

\begin{proof}
% Let us first prove that the output $\bar x_T$ of Algorithm~\ref{alg:game} with the weighting scheme $\{\alpha_t=1\}$ is identically the output $w_T$ of Algorithm~\ref{alg:fw} with learning rate $\gamma_t = \frac{1}{t}$.
We show, via induction, that the following three equalities are maintained for every $t$. Note that three objects on the left correspond to the iterative description given in Algorithm~\ref{alg:fw} whereas the three on the right correspond to the FGNRD description.
 % We emphasize that the objects on the left correspond to Alg.~\ref{alg:game} and those on the right to Alg.~\ref{alg:fw}.
  \begin{eqnarray}
     \nabla f(w_{t-1})  & = & y_t \label{eq:ygradfw}\\
     v_t  & = & x_t \label{eq:yv}\\
     w_t & = &\bar{x}_t  \label{eq:xfw}.
  \end{eqnarray}
  To start, we observe that since the $\alg^Y$ is set as $\FTL[\nabla f(w_0)]$, we have the base case for \eqref{eq:ygradfw}, $y_1 = \nabla f(w_0)$, holds by definition. Furthermore, we observe that for any $t$ we have \eqref{eq:ygradfw} $\implies$ \eqref{eq:yv}. This is because, if $y_t = \nabla f(w_{t-1})$, the definition of \BR implies that 
  \begin{equation}\label{get_x}
    x_t =  \argmin_{x \in \XX} \alpha_t \left( \langle x, y_t \rangle - f^*(y_t) \right)
     = \argmin_{x \in \XX} \langle x, \nabla f(w_{t-1}) \rangle = v_t
  \end{equation}
  Next, we can show that \eqref{eq:yv} $\implies$ \eqref{eq:xfw} for any $t$ as well using induction. Assuming that $w_{t-1} = \frac{\sum_{s=1}^{t-1} \alpha_s x_s}{\sum_{s=1}^{t-1} \alpha_s} = \frac{\sum_{s=1}^{t-1} s v_s}{\sum_{s=1}^{t-1} s}$, a bit of algebra verifies
  \begin{eqnarray*}
    w_t & := & (1 - \gamma_t) w_{t-1} + \gamma_t v_t  = \left(\frac{t-1}{t+1}\right)\frac{\sum_{s=1}^{t-1} s v_s}{\sum_{s=1}^{t-1} s} + \left(\frac{2}{t+1}\right) v_t \\
    & = & \frac{\sum_{s=1}^{t} s v_s}{\sum_{s=1}^{t} s} = \frac{\sum_{s=1}^{t} \alpha_s x_s}{\sum_{s=1}^{t} \alpha_s} =: \bar x_t
  \end{eqnarray*}
%   We first note that the first condition of the theorem ensures that \eqref{eq:ygradfw} holds for $t=1$. Second, the choices of learning rate $\gamma_t$ and the weighting scheme $\{\alpha_t\}$ leads to
%   \begin{eqnarray} \label{eq:w_mix}
% w_t  =  \frac{1}{t-1} \sum_{s=1}^t v_s & = & \frac{1}{A_t} \sum_{s=1}^t \alpha_s v_s,  \quad \text{ if } (\gamma_t=\frac{1}{t}, \alpha_t = 1 ) \\
% w_t  =  \frac{1}{ \sum_{s=1}^t s } \sum_{s=1}^t s v_s  &=& \frac{1}{A_t} \sum_{s=1}^t \alpha_s v_s,  \quad \text{ if } (\gamma_t=\frac{2}{t+1}, \alpha_t = t).
%   \end{eqnarray}
% From \eqref{eq:w_mix}, we see that \eqref{eq:yv} implies \eqref{eq:xfw}, as $w_t$ is always an average of the updates $v_t$. It remains to establish \eqref{eq:ygradfw} and \eqref{eq:yv} via induction. We begin with the former.
  Finally, we show that \eqref{eq:ygradfw} holds for $t > 1$ via induction.
  Recall that $y_t$ is selected via \FTL against the sequence of loss functions $\alpha_t \ell_t(\cdot) := - \alpha_t g(x_t, \cdot)$
   Precisely this means that, for $t > 1$, 
  \begin{eqnarray*} 
    y_t & := & \displaystyle \argmin_{y \in \YY} \left\{ \frac{ 1}{ A_{t-1} } \sum_{s=1}^{t-1} \alpha_s \ell_s(y) \right\} \label{h}\\
   &  = & \argmin_{y \in \YY} \left\{ \frac {1}{ A_{t-1} } \sum_{s=1}^{t-1} (\alpha_s \langle -x_s, y\rangle + f^*(y) ) \right\} \label{g}\\ 
    & = & \argmax_{y \in \YY} \left\{ \langle \bar x_{t-1}, y \rangle - f^*(y)  \right\} 
     =  \nabla f ({\bar x_{t-1}}). \label{get_y}
  \end{eqnarray*}  
The final line follows as a result of the Legendre transform \cite{B04}. Finally, by induction, we have ${\bar x_{t-1}} = w_{t-1}$, and hence we have established \eqref{eq:ygradfw}.
% Finally, let us consider how $x_t$ is chosen according to \BR. Recall that sequence of loss functions presented to the $x$-player is $\alpha_t h_t(\cdot) := \alpha_t g(\cdot,y_t) $. Utilizing \BR for this sequence implies that
  % \begin{equation}
  % \begin{aligned}\label{get_x}
  % x_t & = \argmin_{x \in \XX} \alpha_t h_t(x) = \argmin_{x \in \XX} \alpha_t \left( x^\top y_t - f^*(y_t) \right)
  %  = \argmin_{x \in \XX} \left( x^\top y_t \right)\\ 
  % \text{(\eqref{eq:ygradfw} by induc.)}\quad
  % & = \argmin_{x \in \XX}  x^\top \nabla f ({\bar x_{t-1}}) \quad  = \quad \argmin_{x \in \XX} x^\top \nabla f (w_{t-1}) = v_t.
  % %\quad ( \text{ which is } v_t ).
  % \end{aligned}
  % \end{equation}
This completes the proof.

\end{proof}

Now that we have established Frank-Wolfe as an instance of Protocol~\ref{alg:game}, we can now prove a bound on convergence using the tools established in Section~\ref{sec:onlinelearning}.

\begin{theorem}\label{thm:fwconvergence}
  Let $w_T$ be the output of Algorithm~\ref{alg:fw}. Let $f$ be $L$-smooth and let $\K$ have square $\ell_2$ diameter no more than $D$. Then we have
  \[
    f(w_T) - \min_{w \in \K} f(w) \leq \frac{8LD}{T+1}.
  \]
  % where $L$ is the smoothness 
  % Assume that $f(\cdot)$ is $L$-smooth convex. Then Algorithm~\ref{alg:fw}, with learning rate $\gamma_t := \frac 1 t$, outputs $w_T$ with approximation error $O\left( \frac{L D \log T}{T} \right)$;
  % on the other hand,
  % with learning rate $\gamma_t := \frac {2}{ t+1 }$, it outputs $w_T$ with approximation error $O\left( \frac{L D }{T} \right)$,
  % where $D$ is a number satisfying $\| v_t - w_{t-1} \|^2 \leq D $ for all $t$.
\end{theorem}
\begin{proof}
%  Now that we have established that Algorithm~\ref{alg:fw} is an instance of Protocol~\ref{alg:game}, we can appeal directly to
By Proposition~\ref{thm:equivFW} and Theorem~\ref{thm:meta}, 
we obtain 
  \[
    f(w_T) - \min_{w \in \K} f(w) \leq \avgregret{x}[\BR] + \avgregret{y}[\FTL].
  \]
  Recall that we have $\avgregret{x}[\BR] \leq 0$ by Lemma~\ref{lem:BRregret}. Let us then turn our attention to the regret of $\alg^{Y}$.

  First note that, since $f(\cdot)$ is $L$-smooth, its conjugate $f^*(\cdot)$ is $\frac 1 L$-strongly convex, and thus the function $-g(x,\cdot)$ is also $\frac 1 L$-strongly convex in its second argument. Next, if we define $\ell_t(\cdot) := -g(x_t, \cdot)$, then we can bound the norm of the gradient as
  \[
    \| \nabla \ell_t(y_t) \|^2 = \| x_t - \nabla f^*(y_t) \|^2 = \| x_t - \bar{x}_{t-1} \|^2 \leq D.
  \]
  Combining with Lemma~\ref{regret:FTL} we see that
  \begin{align*} \label{eq:FWy2}
    \avgregret{y}[\FTL] 
      & \leq \frac{1}{A_T}  \sum_{t=1}^T \frac{2 \alpha_t^2 \| \nabla \ell_t(y_t) \|^2}{\sum_{s=1}^{t} \alpha_s (1/L)} 
   = \frac{8L}{T(T+1)}
  \sum_{t=1}^T \frac{t^2 D}{t(t+1)} \leq \frac{8LD}{T+1}.
  \end{align*}
This completes the proof.

\end{proof}

\subsubsection{Variant 1: a linear rate Frank-Wolfe over strongly convex set}

\cite{LP66,DR70,D79} show that Frank-Wolfe for smooth convex function (\emph{not necessarily a strongly convex function}) for strongly convex sets has linear rate under certain conditions.
Algorithm~\ref{alg:Ada-fw} shows that a similar result can be derived from the game framework, in which the $y$-player uses a variant of \FTL called \textsc{AFTL} (defined in Algorithm~\ref{alg:SC-AFTL} of Appendix~\ref{app:linearFW}).

%\jake{We need to {}say what AFTL is here.}

\begin{algorithm}[H] 
   \caption{ Adaptive Frank-Wolfe } \label{alg:Ada-fw}
Given: $L$-smooth convex $f(\cdot)$, arbitrary $\bar{x}_0=x_0=w_0 \in \K \subset \reals^d  $, iterations $T$, $A_0=0$.
\begin{center} 
\begin{tabular}{c c}
    $
      \boxed{\small
      \begin{array}{rl}
        x_t  & \gets \argmin_{x \in \K} \langle x, \nabla f(w_{t-1})  \rangle\\
        \alpha_t &  \gets \frac{1}{ \| x_t - \bar{x}_{t-1} \| }\\
        A_t  & \gets A_{t-1} + \alpha_t \\
        w_t  & \gets \frac{1}{A_t} \sum_{s=1}^{t} \alpha_s x_s \\
      \end{array}
      }
    $
    & \small
    $\boxed{\small
    \begin{array}{rl}
      g(x,y) & := \langle x, y \rangle - f^*(y)\\
      \alg^Y & := \text{AFTL}[\nabla f(x_0)] \\
      \alg^X & := \BR \\
    \end{array}
    }$
  \\
    \small Iterative Description & 
    \small FGNRD Equivalence
\end{tabular}
\end{center}
Output: $w_T = \bar x_T$
\end{algorithm}

\begin{theorem}\label{thm:linearFW} 
Suppose that $f(\cdot)$ is $L$-smooth convex.
and that $\K$ is a $\lambda$-strongly convex set. Also assume that the gradients of the $f(\cdot)$ in $\K$ are bounded away from $0$, i.e., $\max_{w \in \K}\|\nabla f(w)\|\geq B$.
%Then, Algorithm~\ref{alg:Ada-fw} in Appendix~\ref{app:linearFW} has
Then, Algorithm~\ref{alg:Ada-fw} has
\[
f(w_T) - \min_{w \in \K} f(w) = O\left( \exp\left(-\frac{\lambda B }{L} T\right)  \right).
\]
%Then, there exists a FW-like algorithm that has $O(\exp(-\frac{\lambda B }{L} T))$ rate which is an instance of Algorithm~\ref{alg:game}
%with the weighting scheme $\alpha_t:= \frac{1}{ \| \nabla \ell_t(y_t) \|^2 }$ if Alg.~\ref{alg:game} sets  $\alg^Y := \FTL$ and $\alg^X := \BR$.
\end{theorem}
Note that the weights $\alpha_t$ are not predefined but rather depend on the queries of the algorithm. The proof of Theorem~\ref{thm:linearFW} is described in full detail in Appendix~\ref{app:linearFW}. 

\subsubsection{Variant 2: an incremental Frank-Wolfe}

Recently, \cite{Netal20} and \cite{LF20} propose stochastic Frank-Wolfe algorithms
for optimizing smooth convex finite-sum functions, i.e. $\min_{x \in \K} f(x) := \frac{1}{n}\sum_{i=1}^n f_i(x)$,
where each $f_i(x) := \phi( x^\top z_i)$ represents a loss function $\phi(\cdot)$ associated with sample $z_i$. 
In each iteration the algorithms only require a gradient computation of a single component. 
%, see option (A) of Algorithm~\ref{alg:newStoFW}.
\cite{Netal20} show that their algorithm has $O(\frac{ c_{\kappa}}{T})$ expected convergence rate, where $c_{\kappa}$ is a number that depends on the underlying data matrix $z$ and in worst case is bounded by the number of components $n$. 
%We show that a similar algorithm, %option (B) of
Algorithm~\ref{alg:newStoFW}
shows a similar algorithm which picks a sample in each iteration by cycling through the data points. It is obtained when the $y$-player uses a variant of \FTL called 
\textsc{LazyFTL} defined in Appendix~\ref{app:equivStoFW}. 
We have the following theorem 
that shows an $\tilde{O}(\frac{n}{T})$ \emph{deterministic} convergence rate.
Its proof is in Appendix~\ref{app:equivStoFW}.

\begin{algorithm}[H] 
   \caption{ Incremental Frank-Wolfe } \label{alg:newStoFW}
Given: $L$-smooth $f(\cdot)$, convex domain $\K$, arbitrary $w_0 \in \K \subset \reals^d $, iterations $T$. \\ 
Init: For each sample $i$, compute $g_{i,0} := \frac{1}{n} \nabla f_{i}( w_0 ) \in \reals^d$ so that $\nabla f(w_0) = \sum_{i=1} g_{i,0}$.
\begin{center} 
\begin{tabular}{c c}
    $
      \boxed{\small
      \begin{array}{rl}
 \text{Select} & i_t \text{ via by cycling through } [n]\\
g_{i_t,t} &\gets \frac{1}{n} \nabla f_{{i_t}}( w_t ) \text{ and } \\
g_{j,t} & \gets  g_{j,t-1}, \text{ for } j \neq i_t \\
%\State Compute $\nabla f_{{i_t}}( w_t ) $ 
%and set $g_{i_t,t} := \frac{1}{n} \nabla f_{{i_t}}( w_t )  $.
%For other $j \neq i \in [n]$,  $g_{j,t} =  g_{j,t-1}$.
%\State 
g_t & \gets \sum_{i=1}^n g_{i,t}\\
v_t &\gets \argmin_{x \in \K} \langle x ,  g_t \rangle\\
w_{t} & \gets (1 - \gamma_t) w_{t-1} + \gamma_t v_t
      \end{array}
      }
    $
    & \small
    $\boxed{\small
    \begin{array}{rl}
      g(x,y) & := \langle x, y \rangle - f^*(y)\\
      \alpha_t & \gets 1\\
%      (B): \alpha_t & \gets 1\\
      \alg^Y & := \textsc{LazyFTL}[\nabla f(w_0)] \\
      \alg^X & := \BR \\
    \end{array}
    }$
  \\
    \small Iterative Description & 
    \small FGNRD Equivalence
\end{tabular}
\end{center}
Output: $w_T = \bar x_T$
\end{algorithm}

\iffalse
\begin{algorithm} 
   \caption{Stochastic Frank-Wolfe algorithm.
   %\footnotesize(option (A) is the algorithm of \cite{Netal20}, while option (B) is the algorithm analyzed in this paper.)
   } \label{alg:newStoFW}
\begin{algorithmic}[1]
%\State \textbf{Input:} Data matrix $Z = [z_1; z_2; \dots, z_n] \in \reals^{n \times d}$.
\State \textbf{Init:} $w_0 \in \K$.
\State For each sample $i$, compute $g_{i,0} := \frac{1}{n} \nabla f_{i}( w_0 ) \in \reals^d$.
\For{$t= 1, 2, \dots, T$}
%\State Option (A): Sample a $i_t \in [n]$ uniformly at random.
%\State Option (B): 
\State Select a sample $i_t \in [n]$ by cycling through the samples.
\State Compute $\nabla f_{{i_t}}( w_t ) $ 
and set $g_{i_t,t} := \frac{1}{n} \nabla f_{{i_t}}( w_t )  $.
For other $j \neq i \in [n]$,  $g_{j,t} =  g_{j,t-1}$.
\State $g_t = \sum_{i=1}^n g_{i,t}$.
\State $v_t = \argmin_{x \in \K} \langle x ,  g_t \rangle.$
%\State Option (A): $w_t = (1- \frac{2}{t+1}) w_{t-1} + \frac{2}{t+1} v_t$.
%\State Option (B): $w_t = (1- \frac{1}{t}) w_{t-1} + \frac{1}{t} v_t $.
\State $w_t = (1- \frac{1}{t}) w_{t-1} + \frac{1}{t} v_t $.
\EndFor
\State Output $w_t$.
\end{algorithmic}
\end{algorithm}
\fi

\begin{theorem} \label{thm:equivStoFW}
 %When both are run for exactly $T$ rounds, the output $\bar x_T$ of Algorithm~\ref{alg:game} with the weighting scheme $\{\alpha_t=t\}$ is identically the output $w_T$ of Algorithm~\ref{alg:newStoFW} with learning rate $\gamma_t = \frac{1}{t}$
 %as long as: 
 %\textbf{(I)} Alg.~\ref{alg:game} sets $\alg^Y := g_t$ (line 7 of Algorithm~\ref{alg:newStoFW}); \textbf{(II)}  Alg.~\ref{alg:game} sets $\alg^X := \BR$.
 %Furthermore, 
   Assume that $f(\cdot)$ is $L$-smooth convex
and that its conjugate is $L_0$-Lipschitz.
    %Then option (B) of 
    Algorithm~\ref{alg:newStoFW} has
    \[
f(w_T) - \min_{w \in \K} f(w) = O\left( \frac{ \max\{ L R, L(L_0+r) n r \}  \log T }{T} \right),
    \]
    %outputs $w_T$ with approximation error $O\left( \frac{ \max\{ L D, L(L_0+r) n r \}  \log T }{T} \right)$,
  where $r$ is a bound of the length of any point $x$ in the constraint set $\K$, i.e. $\max_{x \in \K} \| x \| \leq r$, and $R$ is the square of the diameter of $\K$.
\end{theorem}

\subsubsection{More related works}

There has been growing interest in Frank-Wolfe in recent years, e.g.
\cite{J13,K08}.
%\cite{D15} shows that for strongly convex and smooth objective functions, FW can achieve $O(1/T^2)$ convergence rate over strongly convex set .
For the constraint set that belongs to a certain class of convex polytopes, there
are FW-like algorithms achieving linear convergence rate, see e.g. \cite{D16a,D16b,W70,S15,FG16,BPTW19}. 
%\cite{D13,D16b} show that exponential convergence for strongly convex and smooth objectives over some polytopes can be achieved by a projection-free algorithm. Their algorithms require a stronger oracle by using the standard one, but can be efficiently implemented for certain polytopes like simplex. Other linear rate of FW-like algrorithms for certain convex polytopes includes \cite{D16a,W70,GJL16,S15,FG16}.
There are also works that study Frank-Wolfe on various aspects, e.g.
%online learning setting \cite{HK12}, 
minimizing some structural norms \cite{H13},
reducing the number of gradient evaluations \cite{LZ16}.
Bach \cite{B15} shows that for certain types of objectives, subgradient descent applied to the primal domain is equivalent to FW applied to the dual domain. 
%block-wise update for structural SVM \cite{SJM13,O16,W16}.
%Finally, we note that Frank-Wolfe has a nice property that it tends to produce sparse solution (see e.g. \cite{J13,K08}), as it adds one component at a time. 
%\junkun{todo: cite more Frank-Wolfe papers}

\subsection{Vanilla gradient descent, with averaging} \label{thm:GD}

In Algorithm~\ref{alg:GD} we describe a variant of gradient descent. It outputs a weighted average of iterates of the gradient descent steps.
The algorithm is obtained by swapping the ordering of the players. The $x$-player plays first according to $\OMD$ and then the $y$-player plays $\BR$.
Theorem~\ref{thm:GD} below shows that the algorithm has a guarantee even for non-smooth convex functions.
It is noted that Algorithm~\ref{alg:GD} with the convergence rate guarantee shown in Theorem~\ref{thm:GD} can also be obtained from the classical online-to-batch conversion (e.g., \cite{cesa2004generalization} or Theorem 9.5 in \cite{hazan2016introduction}). Nevertheless, we provide an analysis based on our framework.

\begin{algorithm}[h] 
   \caption{ Vanilla gradient descent, with averaging  } \label{alg:GD}
Given: Convex $f(\cdot)$ and iterations $T$.\\
Init: $w_{0} = x_{0} \in \K \subseteq \reals^d$.
Set:  $\eta = \gamma = \begin{cases} \frac{R}{G \sqrt{T} },  &  \text{ if } f(\cdot) \text{ is non-smooth,}  \\ \frac{1}{2L},  & \text{ if } f(\cdot) \text{ is $L$-smooth}.  \end{cases}$
\begin{center} 
\begin{tabular}{c c}
    $
      \boxed{\small
      \begin{array}{rl} 
%        \eta & \gets \frac{R}{G\sqrt T} \\
        w_t & \gets w_{t-1} - \eta \delta_{t-1} , \text{where } \delta_{t-1} \in \partial f(w_{t-1}) \\
        \bar w_t & \gets \frac 1 t \sum_{s=1}^t w_s
      \end{array}
      }
    $
    & \small
    $\boxed{\small
    \begin{array}{rl}
      g(x,y) & := \langle x, y \rangle - f^*(y)\\
      \alpha_t & := 1\text{ for } t=1, \ldots, T\\
      \alg^X & := \OMD[\frac 1 2 \| \cdot \|^2_2, x_0, \gamma ] \\ %\text{, or, } 
      \alg^Y & := \BR
%      \alg^X & := \BTRL[\frac 1 2 \| \cdot \|^2_2, \frac{1}{4L}] \\
    \end{array}
    }$
  \\
    \small Iterative Description & 
    \small FGNRD Equivalence
\end{tabular}
\end{center}
Output: $\bar w_T = \bar x_T$
\end{algorithm}

\begin{theorem} \label{thm:GD}
Denote the constant $G = \max_{y \in \partial f(w), w \in \K} \| y \|^{2}$, and let $R$ be an upper bound on $\|w_0-w^*\|$, where $w^{*}:= \argmin_{w \in \K} f(w)$.
The output $\bar w_T = \bar{x}_{T}$ of Algorithm~\ref{alg:GD} satisfies 
\[
  f(\bar w_T) - \min_{w \in \K} f(w) =
  O\left( \frac{ GR  }{\sqrt{T}} \right).
\]
\end{theorem}

\begin{proof}
The equivalence $w_{t} = x_{t}$ between the two displays can be easily shown by noting that
the update of \OMD is $x_{{t}} = x_{{t-1}} - \alpha_t \gamma y_{t-1}$,
where $\gamma = \frac{R}{G\sqrt T}$ and $\alpha_{t}=1$,
and the update of $\BR$
is $y_{t} = \argmin_{{y \in \YY}} f^{*}(y) - \langle x_t, y \rangle \in \partial f(x_t)$.

By Lemma~\ref{lem:OMD}, the regret of the $x$-player is,
%\begin{align}
$   \regret{x} 
 \leq 
\frac{1}{\gamma} \V{x_0}(x^*) 
+ \frac{\gamma}{2}   \sum_{t=1}^T \| \alpha_t y_t \|^2_2.
$
%\end{align}
%where $\ell_{t}(x_t):= \langle y_{t}, x_{t} \rangle$.
On the other hand, by Lemma~\ref{lem:BRregret}, the regret of the $y$-player is $0$.
By setting $\alpha_{t} = 1$, adding average regrets of both players, and plugging in $\gamma = \frac{R}{G\sqrt T}$,
\begin{align}
 \avgregret{x}[\OMD] + \avgregret{y}[\BR]
%\\ & 
& \leq \frac{1}{A_t}
\left (  
\frac{1}{\gamma} \V{x_0}(x^*)
+ \sum_{t=1}^T \frac{\gamma}{2}  \| \alpha_t y_t \|^2
\right ) \label{ygb}
\\ & \leq 
\frac 1 T \left(\frac{R^2}{\gamma} + \frac{\gamma TG^2}{2} \right) \notag
\\ &=
O\left( \frac{ GR  }{\sqrt{T}} \right).  \notag
\end{align}
\end{proof}

We should note that \BR is not necessarily a well-defined algorithm, when the required $\argmin$ doesn't have a unique solution. This is relevant in the case when $f(\cdot)$ is not a differentiable function, which is why we have set $y_t = \delta_{t}$ for any $\delta_t \in \partial f(x_t)$. On the other hand, the convergence in Theorem~\ref{thm:GD} holds for any choice of subgradients.

We can also recover the known $O\left(\frac{1}{T}\right)$ convergence rate of GD when the function $f(\cdot)$ is smooth.
To achieve this, we will use the following known result to upper-bound the r.h.s. of \eqref{ygb}.
\begin{lemma}(see e.g., Lemma~A.2 in \cite{L17}) \label{lem:smooth}
Assume $f(\cdot)$ is $L$-smooth. Denote $x^* = \argmin_{x \in \reals^d} f(x)$. Then,
$$ \| \nabla f(x) \|^2 \leq 2 L \left( f(x) - f(x^*)  \right), \forall x \in \reals^d. $$
\end{lemma}
%Equipped with Lemma~\ref{lem:smooth}, we are ready to state the
\begin{theorem} \label{thm:GD02}
Assume $f(\cdot)$ is $L$-smooth.
% Let $\eta = \frac{1}{2L}$. 
Then,
the output $\bar w_T = \bar{x}_{T}$ of Algorithm~\ref{alg:GD} satisfies 
\[
  f(\bar w_T) - \min_{w \in \reals^d} f(w) \leq
\frac{2L \| w_0 - w^*\|^2 }{  T }.
\]
\end{theorem}

\begin{proof}
Let $\alpha_{t}= 1, \forall t$ and hence $A_{t} = T$.
We start by upper-bounding the last term on the r.h.s.~of \eqref{ygb}.
We have
\begin{align}
 \avgregret{x}[\OMD] + \avgregret{y}[\BR]
%\\ & 
& \leq \frac{1}{A_t}
\left (  
\frac{1}{\gamma} \V{x_0}(x^*)
+ \sum_{t=1}^T \frac{\gamma}{2}  \| \alpha_t y_t \|^2
\right ) \notag
\\ & \overset{(a)}{\leq} 
\frac 1 T \frac{\V{x_0}(x^*)}{\gamma} + \gamma L
 \frac 1 T \sum_{t=1}^T ( f(x_t) - f(x^*) ) \notag
\\ & \overset{(b)}{\leq} 
\frac 1 T \frac{\V{x_0}(x^*)}{\gamma} + \gamma L 
\underbrace{ \frac 1 T \sum_{t=1}^T \langle \nabla f(x_t), x_t - x^* \rangle }_{  =\avgregret{x}[\OMD] } \notag,
\end{align}
where (a) uses Lemma~\ref{lem:smooth} and that
the update of $\BR$
is $y_{t} = \argmin_{y \in \YY} f^{*}(y) - \langle x_t, y \rangle = \nabla f(x_t)$,
and (b) uses convexity.

Rearranging the above inequality and noting that $\avgregret{y}[\BR]=0$, we obtain
$$(1-\gamma L)\avgregret{x}[\OMD] = \frac{\V{x_0}(x^*)}{ \gamma T }.$$
Setting $\gamma = \frac{1}{2L}$ and recognizing that $\V{x_0}(x^*) = \frac{1}{2} \| x_0 - x^*\|^{2}
= \frac{1}{2} \| w_0 - w^*\|^{2}$, we obtain the result.

\end{proof}

\subsection{Single-gradient-call extra-gradient, with averaging }

Extra-gradient method \cite{korpelevich1976extragradient,mokhtari2020convergence,mertikopoulos2018optimistic}
and its single-gradient-call variants \cite{hsieh2019convergence,popov1980modification,chambolle2011first,gidel2018variational,daskalakis2017training,peng2020training} have drawn significant interest in recent years.
In Algorithm~\ref{alg:extra} below, we show how a single-gradient-call variant 
%called past extra-gradient
%\cite{popov1980modification} 
can be obtained from the Fenchel game by pitting \OPTMD against \BR.
More precisely, we obtain an $O(1/T)$ convergence rate in the unconstrained setting.
We remark that an $O(1/T)$ convergence rate of single-call extra-gradient has been established in the literature, see e.g., \cite{hsieh2019convergence}, without restricting to the unconstrained setting $\K= \reals^{d}$.

\begin{algorithm}[h] 
   \caption{Single-gradient-call extra-gradient, with averaging  } \label{alg:extra}
Given: $L$-smooth $f(\cdot)$, a $1$-strongly convex $\phi(\cdot)$, iterations $T$ \\
Init: arbitrary $w_{-\frac{1}{2}} = w_{0}  = x_0 = x_{-\frac{1}{2} }\in \reals^d  $.
\begin{center} 
\begin{tabular}{c c}
    $
      \boxed{\small
      \begin{array}{rl} 
\gamma  & \gets \frac{1}{L} \\
  w_t  & \gets  \displaystyle \argmin_{w \in \reals^d } \left( \vphantom{\V{w_{t-\frac{1}{2}}}(w)}
   \alpha_{t} 
\langle w,  \nabla f(w_{t-1}) \rangle \right.
 \\ &\qquad \qquad \quad
\left. + \frac{1}{\gamma} \V{w_{t-\frac{1}{2}}}(w) \right)
\\   w_{t+\frac{1}{2}} & \gets \displaystyle  \argmin_{w \in \reals^d }  \left(  \vphantom{\V{w_{t-\frac{1}{2}}}(w)} 
\alpha_{t}
\langle w,  \nabla f(w_t) \rangle \right. \\ &\qquad \qquad \quad + \left. \frac{1}{\gamma} \V{w_{t-\frac{1}{2}}}(w) \right)
      \end{array}
      }
    $
    & \small
    $\boxed{\small
    \begin{array}{rl}
      g(x,y) & := \langle x, y \rangle - f^*(y)\\
      \alpha_t & := 1\text{ for } t=1, \ldots, T\\
      \alg^X & := \OPTMD[\phi(\cdot), x_0, \frac{1}{2L}] \\ %\text{, or, } 
      \alg^Y & := \BR
%      \alg^X & := \BTRL[\frac 1 2 \| \cdot \|^2_2, \frac{1}{4L}] \\
    \end{array}
    }$
  \\
    \small Iterative Description & 
    \small FGNRD Equivalence
\end{tabular}
\end{center}
Output: $\bar w_T = \bar x_T$
\end{algorithm}

\begin{theorem} \label{thm:extra}
%Assume the constraint set $\K$ has a diameter $R < \infty$.
The output $\bar w_T = \bar{x}_{T}$ of Algorithm~\ref{alg:extra} satisfies 
\[
  f(\bar{w}_T) - \min_{w \in \reals^d } f(w) =  \frac{8 L \V{w_0}(w^*) +
\frac{1}{8L}  \| \nabla f(w_0) \|^2  
    }{T},
\]
where $w^{*}:= \argmin_{w \in \reals^d} f(w)$.
\end{theorem}

\begin{proof}[Proof of Theorem~\ref{thm:extra}]
The equivalence of the two displays, namely, for all $t$,
  \begin{eqnarray*}
     w_t  & =  & x_t \\
     w_{t-\frac{1}{2}} &= & x_{t-\frac{1}{2}} \\
     \nabla f(w_{t})  & = & y_t 
  \end{eqnarray*}
can be trivially shown by induction. Specifically, the first two relations hold in the beginning by letting the initial point $w_{0} = w_{-\frac{1}{2}} = x_{0} = x_{-\frac{1}{2}} \in \reals^{d}$. To show the last one, we apply the definition of \BR,
\begin{equation*}
y_t \gets \argmax_{y \in \YY}  \ell_t(y) 
=  \argmax_{y \in \YY} f^*(y) -\langle x_t, y \rangle 
= \nabla f( x_t).
\end{equation*}
By induction, we have $w_{t}=x_{t}$; consequently, $y_{t} = \nabla f(w_t)$.

% $\nabla f(\bar{w}_{t})  =  y_t$,  

By summing the regret bound of each player, i.e. Lemma~\ref{lem:OPTMD} and Lemma~\ref{regret:BTL}, we get
%\begin{equation*} %\label{eq:e1}
\begin{align}
& \avgregret{x}[\OPTMD] + \avgregret{y}[\BR] \notag
\\ & 
\overset{(a)}{\leq} \frac{1}{A_t}
\left \{   
\frac{1}{\gamma} \V{x_0}(x^*) 
+ \frac{\gamma}{2}   \sum_{t=1}^T \alpha_t^2 \| \nabla f (w_t ) -   \nabla f (w_{t-1} ) \|^2_*\right \} \label{ldiff}
\\ &
\overset{(b) }{\leq} \frac{1}{A_t}
\left \{   
\frac{1}{\gamma} \V{x_0}(x^*) 
+ \gamma  \sum_{t=1}^T \alpha_t^2  \left( \| \nabla f (w_t ) \|^2  + \|\nabla f (w_{t-1} ) \|^2  \right) \right \} \notag
\\ & 
\overset{(c) }{\leq} \frac{1}{A_t}
\left \{   
\frac{1}{\gamma} \V{x_0}(x^*) 
+ \gamma \alpha_1^2 \| \nabla f(w_0) \|^2
+ 4 \gamma  L \sum_{t=1}^T \alpha_t^2  \left( f(w_t) - f(w_*)  \right) \right
 \} \notag
 \\ & 
\overset{(d) }{\leq} \frac{1}{A_t}
\left \{   
\frac{1}{\gamma} \V{x_0}(x^*) 
+ \gamma \alpha_1^2 \| \nabla f(w_0) \|^2
+ 4 \gamma  L \sum_{t=1}^T \alpha_t^2  \langle \nabla f(w_t), w_t - w_* \rangle \right
 \} \notag
 \\ & 
\overset{(e) }{=} \frac{1}{T}
\left \{   
\frac{1}{\gamma} \V{x_0}(x^*) 
+ \gamma  \| \nabla f(w_0) \|^2
\right \}
+ 4 \gamma  L \underbrace{  \frac{1}{T} \sum_{t=1}^T  \langle \nabla f(w_t), w_t - w_* \rangle  
}_{ := \avgregret{x}[\OPTMD] }, \notag
%\end{split}
\end{align}
where in (a) we used the relation $m_{t} = y_{t-1} = \nabla f(w_{t-1})$ when
we applied Lemma~\ref{lem:OPTMD},
(b) is by the triangle inequality,
(c) is by Lemma~\ref{lem:smooth},
(d) uses the convexity, and
(e) uses $\alpha_{t}=1$.
%\| \nabla f(x) \|^2 \leq 2 L \left( f(x) - f(x^*)  \right), \forall x \in \reals^d.

Rearranging the above inequality and noting that $\avgregret{y}[\BR]=0$, we obtain
$$(1-4\gamma L)\avgregret{x}[\OMD] = \frac{\frac{1}{\gamma}\V{x_0}(x^*) + \gamma  \| \nabla f(w_0) \|^2 }{ T }.$$
Setting $\gamma = \frac{1}{8L}$, we obtain the result.

\iffalse
If $\K$ is a compact convex set, then we can bound the r.h.s.~of \eqref{eq:e1} as
\begin{equation} \label{eq:extra-1}
%\begin{align}
%\avgregret{x}[\OPTMD] + \avgregret{y}[\BR]
\text{r.h.s. of} \, \eqref{eq:e1}
\leq \frac{1}{T} \left( \frac{R^2}{\gamma} + \frac{\gamma T G^2 }{2}   \right) \leq O\left( \frac{GR}{\sqrt{T}}  \right), 
%\end{align}
\end{equation}
where we set $\alpha_{t}=1$, used $\V{x_0}(x^*)* \leq R^{2}$ and $\| \nabla f(w)  \|^{2} \leq G^2, \, \forall w \in \K$, and $\gamma = \frac{R}{G \sqrt{T} }$.
\fi

\end{proof}

We remark that we can get an $O(1/\sqrt{T})$ convergence rate when we replace $\reals^{d}$ with $\K \subset \reals^d$ in Algorithm~\ref{alg:extra}. 
Specifically, we can simply bounding the square norm of the difference of the gradients on the r.h.s.~of \eqref{ldiff}
by $G^2$, and then follow the proof of Theorem~\ref{thm:GD} to get a $O(1/\sqrt{T})$ rate by setting the parameters appropriately. Recovering the $O(1/T)$ rate for the constrained setting using the regret analysis is left open in this work.

\subsection{Cumulative gradient descent} \label{thm:GD2}

We describe yet another algorithm from the game for non-smooth optimization.
Algorithm~\ref{alg:GD2} below has the last-iteration guarantee,
while subgradient descent method, i.e. $w_{t} = w_{t-1} - \eta f_x(w_{t-1})$, does not enjoy the last-iteration guarantee, see e.g. Chapter 2 of \cite{S85} or \cite{S14}. This highlights the benefit of the proposed algorithm.

\begin{algorithm}[h] 
   \caption{Cumulative gradient descent} \label{alg:GD2}
Given: Convex $f(\cdot)$ and iterations $T$.\\
Init: $w_{0} = x_{0} \in \reals^{d}$.
\begin{center} 
\begin{tabular}{c c}
    $
      \boxed{\small
      \begin{array}{rl} 
        \eta & \gets \frac{(R/G)}{\sqrt{T} } \\
w_t & = (1-\frac{1}{t}) w_{t-1} - \frac{1}{t} \eta \sum_{s=1}^{t-1} \delta_s,
\\ & \qquad \text{where } \delta_s \in \partial f(w_{s} )
\\ & = (1- \frac{1}{t}) w_{t-1} - \frac{1}{t} \eta \delta_s
\\ & \quad - \frac{t-1}{t} \eta \left( w_{t-1} - (1-\frac{1}{t-1}) w_{t-2}  \right)
\end{array}
      }
    $
    & \small
    $\boxed{\small
    \begin{array}{rl}
      g(x,y) & := \langle x, y \rangle - f^*(y)\\
      \alpha_t & := 1\text{ for } t=1, \ldots, T\\
      \alg^X & := \OMD[\frac 1 2 \| \cdot \|^2_2, x_0, \frac{(R/G)}{\sqrt{T}}] \\ 
      \alg^Y & := \BTL
    \end{array}
    }$
  \\
    \small Iterative Description & 
    \small FGNRD Equivalence
\end{tabular}
\end{center}
Output: $w_T = \bar x_T$
\end{algorithm}

\begin{theorem} \label{thm:GD2}
Denote the constant $G = \max_{\delta\in \partial f(w), w \in \K} \| \delta \|^{2}$, and let $R$ be an upper bound on $\|w_0-w^*\|$.
The output $w_T = \bar{x}_{T}$ of Algorithm~\ref{alg:GD2} satisfies 
\[
  f(w_T) - \min_{w \in \K} f(w) =
  O\left( \frac{ GR  }{\sqrt{T}} \right), 
\]
where $w^{*}:= \argmin_{w \in \K} f(w)$.
\end{theorem}

\begin{proof}

Let $\gamma = \frac{(R/G)}{\sqrt{T}}$ and $\alpha_{t}=1$.
The update of \OMD is $x_{{t}} = x_{{t-1}} - \gamma y_{t-1}$,
while the update of $\BTL$
is $y_{t} = \argmax_{{y\in \YY}} \sum_{s=1}^t \alpha_s \left( f^{*}(y) - \langle x_s, y \rangle \right)  \in \partial f(\bar x_t)$.
Recursively expanding $\bar{x}_{t}$ leads to
\begin{equation*}
\begin{split}
\bar{x}_t & = \frac{(t-1) \bar{x}_{t-1} + x_t  }{t}
= \frac{ (t-1)\bar{x}_{t-1} + x_{t-1} -\gamma y_{t-1}  }{t} 
\\ &
= \left(1-\frac{1}{t}\right) \bar{x}_{t-1} - \frac{1}{t} \gamma \sum_{s=1}^{t-1} y_{s}.
\end{split}
\end{equation*}
%Since $y_{t-1} \in \partial f(\bar{x}_{t-1})$ and $\bar{x}_{{t-1}}=w_{{t-1}}$,
We now see the equivalence of the two displays in Algorithm~\ref{alg:GD2}.

By Lemma~\ref{lem:OMD}, the regret of the $x$-player is,
$   \regret{x} 
 \leq 
\frac{1}{\gamma} \V{x_0}(x^*) 
+ \frac{\gamma}{2}   \sum_{t=1}^T \| \alpha_t y_t \|^2_2.
$
On the other hand, by Lemma~\ref{lem:BRregret}, the regret of the $y$-player is not greater than $0$.
%By setting $\alpha_{t} = 1$ and 
Adding average regrets of both players, we get 
\begin{align*}
 \avgregret{x}[\OMD] + \avgregret{y}[\BTL]
& \leq \frac{1}{A_t}
\left \{   
\frac{1}{\gamma} \V{x_0}(x^*)
+ \sum_{t=1}^T \frac{\gamma}{2}  \| \alpha_t y_t \|^2 
\right \}
\\ &=
O\left( \frac{ GR  }{\sqrt{T}} \right). 
\end{align*}
\end{proof}

Let us comment that the same algorithm can be produced from Algorithm 1 in \cite{C19} with the online gradient descent strategy via a so called ``anytime online-to-batch conversion'', though the algorithm is not explicitly stated in the paper. 
We also note that while Algorithm~\ref{alg:GD2} depicts the equivalence when the problem is unconstrained, i.e., $\K = \reals^{d}$, we can also derive the projected version of the algorithm
for its iterative description when the constraint set is a compact convex set.
% We also emphasize that our analysis here is very simple.
Specifically, 
the left column (iterative description) in Algorithm~\ref{alg:GD2} would become
\begin{align*}
%x_t     & = \Pi_{\K}\left[ x_{t-1} - \frac{t}{8L}\nabla f(w_{t-1}) \right], \\
w_{{t}} & = \left( 1 - \frac{1}{t}  \right) w_{{t-1}}  + \frac{1}{t} x_t,
\text{ where }
x_t     = \Pi_{\K}\left[ x_{t-1} - \gamma \delta_{t-1} \right], \delta_{t-1} \in \partial f(w_{t-1}).
\end{align*}

\subsection{Accelerated methods for smooth convex optimization}

In this subsection, we are going to introduce several \emph{accelerated} algorithms. An optimization procedure is generally referred to as accelerated if it achieves a faster convergence rate relative to a vanilla method, e.g. gradient descent. The convergence of gradient descent on smooth convex optimization problems was known to be $O(1/T)$, and it was believed by many up until the early 1980s that this was the fastest rate achievable despite the only known lower bound of $\Omega(1/T^2)$. It was shown by Nesterov \cite{N83a} that indeed an upper bound of $O(1/T^2)$ is indeed possible with a slightly more complex update scheme. 

We will discuss \emph{Nesterov Accelerated Gradient Descent}, and its variants, in the context of our FGNRD framework. Nesterov's first acceleration method \cite{N83a,N83b} (see also \cite{SBC14}) for the unconstrained setting is frequently described as follows. First initialize $z_0$ and $w_0$ arbitrarily, and then iterate as follows:

\begin{algorithm}[H]
% \floatname{algorithm}{Protocol}
\caption{Unconstrained Nesterov Accelerated Gradient Descent}
\label{alg:nes_basic}
\begin{algorithmic}[1] \normalsize
\State \textbf{Input:} $w_0, z_0$ arbitrary in $\reals^{d}$.
\For{$t=1, 2, \ldots, $} 
    \State $\theta_t \gets \frac{t}{2(t+1)L}, \quad \beta_t \gets \frac{t-2}{t+1}$
    \State $w_t  \gets z_{t-1} - \theta_t \nabla f(z_{t-1})$
    \State  $z_t  \gets w_t + \beta_t (w_t - w_{t-1}) $
\EndFor
    \State \textbf{return} $ w_T $
\end{algorithmic}
\end{algorithm}

But to get the ball rolling let us first consider a related method of Polyak \cite{P64}.

\subsubsection{The Heavy Ball Method} % (fold)
\label{sub:heavy_ball}

In Algorithm~\ref{alg:heavyball} we describe the classical iterative version of the Heavy Ball algorithm in the box on the left. One observes that the update looks quite similar to vanilla gradient descent, but there's the addition of a so-called \emph{momentum} term $v_t$ which is the difference of the previous two iterates. 

\begin{algorithm}[h] 
   \caption{ Heavy Ball  } \label{alg:heavyball}
Given: $L$-smooth $f(\cdot)$, arbitrary $x_0 = w_0 = w_{-1} \in \reals^{d}$, iterations $T$.
\begin{center} 
\begin{tabular}{c c}
    $
      \boxed{\small
      \begin{array}{rl} 
        \eta_t & \gets \frac{t}{4(t+1)L}, \quad \beta_t \gets \frac{t-2}{t+1}\\
        v_t  & \gets w_{t-1} - w_{t-2} \\
        w_{t} & \gets w_{t-1} - \eta_t \nabla f(w_{t-1}) + \beta_t v_t
      \end{array}
      }
    $
    & \small
    $\boxed{\small
    \begin{array}{rl}
      g(x,y) & := \langle x, y \rangle - f^*(y)\\
      \alpha_t & := t\text{ for } t=1, \ldots, T\\
      \alg^Y & := \FTL[\nabla f(x_0)] \\
      \alg^X & := \MD[\frac 1 2 \| \cdot \|^2_2, x_0, \frac{1}{8L}] %\text{, or, } \\
%      \alg^X & := \BTRL[\frac 1 2 \| \cdot \|^2_2, \frac{1}{4L}] \\
    \end{array}
    }$
  \\
    \small Iterative Description & 
    \small FGNRD Equivalence
\end{tabular}
\end{center}
Output: $w_T = \bar x_T$
\end{algorithm}

We provide a formulation of Heavy Ball in the right box of Algorithm~\ref{alg:heavyball}, and let us begin by stating the equivalence. The proof is deferred to Appendix~\ref{app:thm:Heavy}.
%\jake{Should we give the proof here or in the appendix?}
\begin{proposition}
  The two interpretations of Heavy Ball described in Algorithm~\ref{alg:heavyball} are equivalent. %when $\K = \reals^{d}$.
\end{proposition}

\begin{proof}[Proof (sketch)]
  Similar to the proof of Proposition~\ref{thm:equivFW}, we need to show that the following three equalities are maintained throughout each iteration:
   % We emphasize that the objects on the left correspond to Alg.~\ref{alg:game} and those on the right to Alg.~\ref{alg:fw}.
    \begin{eqnarray}
       \nabla f(w_{t-1})  & = & y_t \nonumber \\
       \bar{x}_t & = & \bar{x}_{t-1} - \frac{\gamma \alpha_t^2}{A_t} \nabla f(\bar{x}_{t-1}) + (\bar{x}_{t-1} -  \bar{x}_{t-2} ) \left( \frac{\alpha_t A_{t-2} }{ A_t \alpha_{t-1} }   \right), \label{eq:HBupdate}\\
       w_t & = &\bar{x}_t \nonumber
    \end{eqnarray}
    The first and last lines follow along the same lines as the proof of Proposition~\ref{thm:equivFW}. The expression \eqref{eq:HBupdate} can be established when we see that the $\MD[\frac 1 2 \| \cdot \|^2_2, x_0, \frac{1}{8L}]$ update algorithm can be written as
    \[
      x_{t} = x_{t-1} - \alpha_t \gamma y_t = x_{t-1} - \frac {\alpha_t}{ 8 L} y_t.
    \]
    Plugging this expression into the definition of $\bar{x}_t$ establishes the update \eqref{eq:HBupdate}.
\end{proof}

Denote $\Pi_{\K}[\cdot]$ the projection onto $\K$.
We remark that for the case when $\K \subset \reals^{d}$,
the left column (iterative description) in Algorithm~\ref{alg:heavyball} would become
\begin{align*}
%x_t     & = \Pi_{\K}\left[ x_{t-1} - \frac{t}{8L}\nabla f(w_{t-1}) \right], \\
w_{{t}} & = \left( 1 - \frac{2}{t+1}  \right) w_{{t-1}}  + \frac{2}{t+1} x_t,
\text{ where }
x_t     = \Pi_{\K}\left[ x_{t-1} - \frac{t}{8L}\nabla f(w_{t-1}) \right].
\end{align*}

\begin{theorem} \label{thm:Heavy}
%Assume $\K \subset \reals^{d}$.
%and let $w^*$ be the unconstrained minimizer of $f$. \jake{I think $\K=\reals^n$ should be part of theorem}
The output $\bar{x}_{T}$ of Algorithm~\ref{alg:heavyball} satisfies 
\[
  f(\bar x_T) - \min_{x \in \K} f(x) =
 O\left(  \frac{ 4 L \| x_0 - x^*\|^2 +  \sum_{t=1}^T 4 L \left( \| \bar{x}_{t-1} - x_t  \|^2 -  \| x_{t-1} - x_t \|^2 \right)  }{T^2} \right).
%=  O\left(\frac{1}{T}\right).
\]
Furthermore, given a convex constraint set $\K \subset \reals^{d}$ with diameter $R$, we have $ f(\bar x_T) - \min_{w \in \K} f(w) = O(\frac{LR}{T})$.
\end{theorem}

We note that the first bound in Theorem~\ref{thm:Heavy} does not require $\K$ to be a compact convex set. However, when $\K = \reals^{d}$, unfortunately we cannot further upper-bound the trajectory-dependent term $\| \bar{x}_{t-1} - x_t  \|^2 -  \| x_{t-1} - x_t \|^2 $ by a constant. Whether one can use the regret analysis to show an $O(1/T)$ rate of unconstrained Heavy Ball is left open.

\subsubsection{Nesterov's Methods} % (fold)
\label{ssub:nestmethods}

%\jake{Need to fold this paragraph in.} 
Starting from 1983, Nesterov has proposed various accelerated methods for smooth convex problems
(i.e. \cite{N83a,N83b,N88,N05}), 
all of which are often described as accelerated gradient descent.
%In this section, 
Our goal now will be to try to understand Nesterov's various algorithms.
We will show that all the methods can be generated from \textit{Fenchel game} with some simple tweaks.

\begin{algorithm}[h] 
   \caption{ Nesterov's 1-memory method \cite{N88,T08} } \label{alg:Nes-1mem}
Given: $L$-smooth $f(\cdot)$, convex domain $\K \subseteq \reals^d $, arbitrary $x_0=v_0 \in \K$, 
%\junkun{I think the correct version is ``arbitrary $v_0 \in \K$''}
1-strongly convex distance generating function $\phi(\cdot)$, iterations $T$.
%\junkun{(consistency) distance generating function or Prox. function}
\begin{center} 
\begin{tabular}{c c}
    $
      \boxed{\small
      \begin{array}{rl}
        \beta_t & \gets \frac{2}{t+1}, \gamma_{t} \gets \frac{t}{4L}\\
        z_{t} & \gets (1 - \beta_t) w_{t-1} + \beta_t v_{t-1}\\
        v_{t} & \gets \underset{x \in \K}{ \argmin} 
          \gamma_t \langle  \nabla f(z_t), x \rangle  + \V{v_{t-1}}(x)
        \\
        w_{t} & \gets (1 - \beta_t) w_{t-1} + \beta_t v_{t}
      \end{array}
      }
    $
    & \small
    $\boxed{\small
    \begin{array}{rl}
      g(x,y) & := \langle x, y \rangle - f^*(y)\\
      \alpha_t & := t\text{ for } t=1, \ldots, T\\
      \alg^Y & := \OFTL \\ %\junkun{\text{It should be }\OFTL[\nabla f(v_0)] } \\
      \alg^X & := \MD[\phi(\cdot), x_0, \frac{1}{4L}] \\
    \end{array}
    }$
  \\
    \small Iterative Description & 
    \small FGNRD Equivalence
\end{tabular}
\end{center}
Output: $w_T = \bar x_T$,
\end{algorithm}

% \begin{algorithm}[H] 
%    \caption{ Nesterov's first acceleration method  \cite{N83b,N83a}
%     } \label{alg:NesV0}
% Given: $L$-smooth $f(\cdot)$, arbitrary $w_0 \in \reals^d$, iterations $T$
% \begin{center} 
% \begin{tabular}{c c}
%     $
%       \boxed{\small
%       \begin{array}{rl}
%         \theta & \gets \frac{t}{2(t+1)L}, \beta_t \gets \frac{t-1}{t+2}\\
%         w_t & \gets z_{t-1} - \theta \nabla f(z_{t-1})\\
%         z_t & \gets w_t + \beta_t (w_t - w_{t-1})
%       \end{array}
%       }
%     $
%     & \small
%     $\boxed{\small
%     \begin{array}{rl}
%       g(x,y) & := \langle x, y \rangle - f^*(y)\\
%       \alpha_t & := t\text{ for } t=1, \ldots, T\\
%       \alg^Y & := \OFTL[\nabla f(w_0)] \\
%       \alg^X & := \begin{cases}
%        \MD[\frac 1 2 \| \cdot \|^2_2, \frac{1}{4L}] & \text{, or, } \\
%       % \alg^X & := 
%       \BTRL[\frac 1 2 \| \cdot \|^2_2, \frac{1}{4L}] & \\
%       \end{cases}
%     \end{array}
%     }$
%   \\
%     \small Iterative Description & 
%     \small FGNRD Equivalence
% \end{tabular}
% \end{center}
% Output: $w_T = \bar x_T$
% \end{algorithm}

% Now let us consider that the $x$-player's action space is unconstrained.
% That is, $\K= \reals^n$. 
% We are going to show that 

% our framework can recover
%has a direct correspondence with 

Let us begin with the version known as the \emph{1-memory method}. 
%We describe this alongside our 
%We first consider recovering Nesterov’s (1988) 1-memory method  \cite{N88}
%and Nesterov’s (2005) $\infty$-memory method \cite{N05}.
To be precise, we adopt the presentation of Nesterov’s algorithm given in
Algorithm~1 
%and Algorithm~3 
of \cite{T08}.
% respectively.

\begin{proposition}\label{thm:Nes_constrained} 
 The two interpretations of Nesterov's $1$-memory method in Algorithm~\ref{alg:Nes-1mem} 
 % (Nesterov's $\infty$-memory method, as described in Algorithm~\ref{alg:Nes-1mem} (Algorithm~\ref{alg:Nes-infmem}, respectively),
 are equivalent. That is, the iteratively generated points $w_t$ coincide with the sequence $\bar x_t$ produced by the FGNRD dynamic.
% When both are run for exactly $T$ rounds, the output $\bar x_T$ of Algorithm~\ref{alg:game} with the weighting scheme $\{\alpha_t=t\}$ is identically the output $w_T$ of Algorithm~\ref{alg:Nes} 
 %\kfir{should be Alg.~\ref{alg:Nes}}
 %with option (A) (respectively, option (B))
 %as long as: \textbf{(I)} Initialize $y_1$ in Alg.~\ref{alg:game} equals $\nabla f(z_0)$ in Alg.~\ref{alg:fw}; \textbf{(II)} Alg.~\ref{alg:game} sets $\alg^Y := \OFTL$; \textbf{(III)}  Alg.~\ref{alg:game} sets $\alg^X := \MD$ with parameter $\gamma=\frac{1}{4L}$ (respectively, \BTRL with $\eta = \frac{1}{4L}$).
\end{proposition} 

\begin{proof}

%For convenience,
Allow us define some notations first.
\begin{eqnarray} 
  \nonumber \xav_t  & := &  \displaystyle \frac{1}{A_t}\sum_{s=1}^t \alpha_s x_s \\
  \label{def:tilde_x} \xof_t  & := &  \displaystyle \frac{1}{A_t}\left(\alpha_t x_{t-1} + \sum_{s=1}^{t-1} \alpha_s x_s \right),\\
  \nonumber \yftl_{t+1} & := & \argmin_{y \in \YY} \sum_{s=1}^t \alpha_s \left( f^*(y) - \langle x_s, y \rangle\right)
\end{eqnarray}
Let us make a few observations about these sequences.
%\kfir{The numbering below hides the equations}\\
% \kfir{Should define $\yof_t$, Isn't $\yof_t:=y_t$?}
\begin{eqnarray}
  \label{eq:haty}
  \yftl_{t+1} & = & \argmax_{y \in \YY}  \left \langle \xav_t, y \right \rangle - f^*(y) = \nabla f(\xav_{t}) \\
%  \label{eq:tildey}
%   \yof_t & = &
%\argmin_y \left \{
% A_T (f^*(y) - \langle \xof_t, y \rangle) \right \} =
%   \nabla f(\xof_t) , \label{eq:yopt} \\
  \label{eq:xdiffs} \xof_t - \xav_t & = & \frac{\alpha_t}{A_t}(x_{t-1} - x_{t}).
\end{eqnarray}
Equations~\eqref{eq:haty} and~\eqref{eq:xdiffs} follow from simple arithmetic as well as elementary properties of Fenchel conjugation and the Legendre transform \cite{R96}. Equation~\eqref{eq:xdiffs} is by a simple algebraic calculation.

%Let us recall the notations (\ref{def:tilde_x}),
%$\xav_t := \displaystyle \frac{1}{A_t}\sum_{s=1}^t \alpha_s x_s \text{ and } \xof_t := \displaystyle \frac{1}{A_t}\left(\alpha_t x_{t-1} + \sum_{s=1}^{t-1} \alpha_s x_s \right).$
We show, via induction, that the following equalities are maintained for every $t$. Note that three objects on the left correspond to the iterative description given in Algorithm~\ref{alg:accLinear} whereas the three on the right correspond to the FGNRD description.

  \begin{eqnarray}
%    \tilde{x}_t  & = & z_t \label{eq:ygradAcc}\\
    \nabla f(z_t) & = & y_t \label{eq:ygradAcc}\\
    v_t & = & x_t \label{eq:zAcc}\\
    w_t & = & \bar{x}_t \label{eq:xAcc}.%\\
%    z_t & = & \tilde{x}_t.
  \end{eqnarray}
  We first note that 
  the initialization ensures that \eqref{eq:ygradAcc} holds for $t=1$.
  Also, observe that we have
  $\beta_t = \frac{\alpha_t}{A_t}.$
  Therefore,
  \begin{equation} \label{eq:w_mixAcc2}
  \begin{split}
  w_t &  = \left(  1 - \beta_t \right) w_{t-1} +\beta_t v_t
=  \left( 1 - \frac{\alpha_t}{A_t} \right) w_{t-1} + \frac{\alpha_t}{A_t} v_t
=  \frac{A_{t-1}}{A_t} w_{t-1} + \frac{\alpha_t}{A_t} v_t
  \end{split}
  \end{equation}
From \eqref{eq:w_mixAcc2}, we see that \eqref{eq:zAcc} and $w_{{t-1}}=\bar{x}_{{t-1}}$ implies \eqref{eq:xAcc}, as $w_t$ is always an average of the updates $v_t$. It remains to establish \eqref{eq:ygradAcc} and \eqref{eq:zAcc} via induction.

Let us first show \eqref{eq:ygradAcc}.
But before that, we show %in \eqref{eq:tildey} 
$y_t = \nabla f(\tilde{x}_t)$ as follows.

In Algorithm~\ref{alg:Nes-1mem} the $y$-player is using the \OFTL algorithm to choose the sequence of $y_t$'s, which is given by
\begin{equation} \label{60}
\yof_t = \argmin_{{y \in \YY}} \left \{\displaystyle \alpha_t \ell_{t-1}(\cdot) +  \sum_{s=1}^{t-1} \alpha_{s} \ell_{s}(y) \right\},
\end{equation}
where the guess $m_{t}(\cdot)$ to be $m_{t}(\cdot) \gets \ell_{{t-1}}(\cdot)$, $\ell_s(y) := f^*(y) - \langle x_s, y\rangle$, and $\ell_{0}(x_0):= f^{*}(y) - \langle x_{0}, y \rangle$ so that $y_0=\nabla f(x_0)$. 
Then, we have 
\begin{equation} \label{eq:tildey}
   \yof_t  = 
\argmin_{y \in \YY} \left \{
 A_T (f^*(y) - \langle \xof_t, y \rangle) \right \} =\nabla f(\xof_t).
\end{equation}
%the optimization problem in \eqref{eq:yopt}.
%Hence, we have $y_t = \nabla f(\tilde{x}_t)$.
%We need to consider two parallel sequence of points for both the $x$- and $y$-players. 

Given that $y_t = \nabla f(\tilde{x}_t)$ shown in \eqref{eq:tildey},
it suffices to show that $\tilde{x}_t = z_t$ for establishing \eqref{eq:ygradAcc}.
We have 
\begin{equation*}
\begin{split}
z_t & = (1- \beta_t) w_{t-1} + \beta_t v_{t-1}
= (1 - \frac{\alpha_t}{A_t} ) \bar{x}_{t-1} + \frac{a_t}{A_t} x_{t-1} 
\\ &= \frac{A_{t-1}}{A_t} \frac{ \sum_{s=1}^{t-1} \alpha_s x_s }{A_{t-1}} + \frac{a_t}{A_t} x_{t-1}
=  \frac{ \sum_{s=1}^{t-1} \alpha_s x_s + \alpha_t x_{t-1}}{A_t} 
= \xof_{t}.
\end{split}
\end{equation*}
Hence, \eqref{eq:ygradAcc} holds.

To show (\ref{eq:zAcc}), observe that the update of $v_t$  is exactly equivalent to \MD
 for which $\gamma \leftarrow \frac{1}{4L}$, $\alpha_{t} \leftarrow t$, and $\ell_t(\cdot) \leftarrow \langle y_t , \cdot \rangle$. 
 We thus have completed proof.
\end{proof}

We remark that when the problem is unconstrained, i.e. $\K = \reals^{d}$, Algorithm~\ref{alg:Nes-1mem}
with the distance generating function $\phi(\cdot)=\frac{1}{2}\| \cdot\|^{2}_{2}$ can be rewritten as Algorithm~\ref{alg:nes_basic} --- the 
Nesterov’s first acceleration method.

\begin{lemma} \label{lem:yregretbound}
Consider the FGNRD implementation in Algorithm~\ref{alg:Nes-1mem}, where the function $f(\cdot)$ is convex and $L$-smooth with respect to the norm $\| \cdot \|$, whose dual norm is $\| \cdot \|_*$. Regardless of the sequence of points $x_1, \ldots, x_T$ generated from the protocol, we have
\begin{equation} \label{wregret_y}
\displaystyle
\regret{y}[\OFTL] \leq L \sum_{t=1}^T \frac{\alpha_t^2}{A_t} \|x_{t-1} - x_t \|^2.
\end{equation}
\end{lemma}

\begin{proof}
%In the following, we consider that the $y$-player uses \OFTL, i.e. $y_t:= \tilde{y}_t$.

By using Lemma~\ref{regret:Opt-FTL}
with the guess $m_t(\cdot) \leftarrow \ell_{t-1}(\cdot)$, $w_t \leftarrow \hat{y}_{t}$ of 
\eqref{eq:haty}, and $z_t \leftarrow y_t$ of \eqref{eq:tildey},
and that $
\alpha_t \left( \ell_t(y) - \ell_{t-1}(y) \right) = \alpha_t \langle x_{t-1} - x_{t}, y \rangle$ in Fenchel Game, we have
    \begin{eqnarray*}
         \displaystyle \sum_{t=1}^{T} \alpha_t \ell_t(\yof_t) - \alpha_t \ell_t(y^*) 
        & \leq &  \displaystyle \sum_{t=1}^T \alpha_t
     \left(   \ell_t( \yof_t ) - \ell_{t-1}( \yof_t )
      - \left( \ell_t(\yftl_{t+1}) - \ell_{t-1}(\yftl_{t+1})  \right) \right)\\
   %      \delta_t(\yof_t) - \delta_t(\yftl_{t+1}) = \displaystyle  \sum_{t=1}^T \alpha_t \langle x_{t-1} - x_{t}, \yof_t - \yftl_{t+1} \rangle \\
        \text{(Eqns. \ref{eq:haty}, \ref{eq:tildey})} \quad \quad
        & = & \displaystyle  \sum_{t=1}^T \alpha_t \langle x_{t-1} - x_{t}, \nabla f(\xof_t) - \nabla f(\xav_t) \rangle \\
        \text{(H\"older's Ineq.)} \quad \quad
        & \leq & \displaystyle  \sum_{t=1}^T \alpha_t \| x_{t-1} - x_{t}\| \| \nabla f(\xof_t) - \nabla f(\xav_t) \|_* \\
        \text{($L$-smoothness of $f$)} \quad \quad
        & \leq & \displaystyle   L \sum_{t=1}^T \alpha_t \| x_{t-1} - x_{t}\| \|\xof_t - \xav_t \| \\
        \text{(Eqn. \ref{eq:xdiffs})} \quad \quad
        & = & \displaystyle  L \sum_{t=1}^T \frac{\alpha_t^2}{A_t} \| x_{t-1} - x_{t}\| \|x_{t-1} - x_{t} \|
    \end{eqnarray*}
    as desired, where the first inequality is because that $m_t(\cdot)= \ell_{t-1}(\cdot)$.
 \end{proof}

\begin{theorem}\label{thm:metaAcc}
% Following the setting as Theorem~\ref{thm:metaAcc},
% if the $x$-player is $MD$ with a
% $1$-strongly convex distance generating function $\phi(\cdot)$ and
% the parameter $\gamma$ that
% satisfies
% $\frac{1}{CL} \leq \gamma \leq \frac{1}{4L}$ for some constant $C \geq 4$, then 
% \[
   
% \] where $\V{x_0}(x^*) \leq D$.
Let us consider the output $\xav_T \equiv w_T$ of Algorithm~\ref{alg:Nes-1mem}. Given that $f$ is $L$-smooth, $w_T$ satisfies
\[
  \displaystyle f(w_T) - \min_{w \in \K} f(w) 
  \leq \frac{8 L  D_{w_0}^{\phi}(w^*)} {T^2}.
\]
\end{theorem}
\begin{proof}
The equivalence $\bar x_T \equiv w_T$ was established in Proposition~\ref{thm:Nes_constrained}. We may now appeal to Theorem~\ref{thm:meta} to prove this result, which gives us that
\[
  f(w_T) - \min_{w \in \K} f(w) 
  \leq\frac{1}{A_T}(\regret{y}[\OFTL] + \regret{x}[\MD]).
\]
We have a bound for $\regret{x}$ from Lemma~\ref{lem:MD}, with parameters $\gamma = \frac 1 {4L}$ and $\beta = 1$, and a bound on $\regret{y}$ in Lemma~\ref{lem:yregretbound}. Combining these we obtain
\begin{equation} \label{eq:genericbound}
\displaystyle  f(w_T) - \min_{w \in \K} f(w)
  \leq \frac 1 {A_T} \left( 4L D_{w_0}^{\phi}(w^*)   
    + \sum_{t=1}^{T} \left(\frac{\alpha_t^2}{A_t} L - 2L \right)
         \| x_{t-1} - x_t \|^2    \right).
\end{equation}
Given the choice of weights $\alpha_t = t$, we have $A_t = \frac{t(t+1)}{2}$ and therefore $\frac{\alpha_t^2}{A_t} = \frac{2t^2}{t(t+1)} \leq 2$. With this in mind, the sum on the right hand side of \eqref{eq:genericbound} is non-positive. Noticing that $\frac{1}{A_T} \leq \frac 2 {T^2}$ completes the proof. 
% The choice of $\{\alpha_t,\gamma\}$ implies $\frac{D}{\gamma} \leq CLD$ and $\frac{L\alpha_t^2}{A_t} = \frac{2Lt^2}{t(t+1)} \leq 2L \leq \frac 1 {2 \gamma}$, which ensures that the summation term in \eqref{eq:genericbound} is negative. The rest is simple algebra.

% Similar calculations can be done for the bound \eqref{eq:genericbound2}, and hence omitted.

% We have already done the hard work to prove this theorem. Lemma~\ref{lem:fenchelgame} tells us we can bound the optimization error of $\xav_T$ by the $\epsilon$ error of the approximate equilibrium $(\xav_T, \yav_T)$. Theorem~\ref{thm:meta} tells us that the pair $(\xav_T, \yav_T)$ derived from Algorithm~\ref{alg:game} is controlled by the sum of averaged regrets of both players, $\frac{1}{A_T}(\regret{x}[\OFTL] + \regret{y}[\MD])$. But we now have control over both of these two regret quantities, from Lemmas~\ref{lem:yregretbound} of \OFTL and~\ref{lem:MD} of \MD. 
% \begin{equation}
% \displaystyle  f(\xav_T) - \min_{x \in \XX} f(x)
%   \leq \frac 1 {A_T} \left( \frac D {\gamma} 
%     + \sum_{t=1}^{T} \left(\frac{\alpha_t^2}{A_t} L - \frac{1}{2 \gamma} \right)
%          \| x_{t-1} - x_t \|^2    \right).
% \end{equation}
%  The right hand side of \eqref{eq:genericbound} is the sum of these bounds.
% blah
\end{proof}

% Theorem~\ref{thm:meta} is somewhat opaque without a specifying the sequence $\{\alpha_t\}$. But what we now show is that the summation term \emph{vanishes} when we can guarantee that $\frac{\alpha_t^2}{A_t}$ remains constant! This is where we obtain the following fast rate.

It is worth dwelling on exactly how we obtained the above result. A less refined analysis of the \MD algorithm would have simply ignored the negative summation term in Lemma~\ref{lem:MD}, and simply upper bounded this by 0. But the negative terms $\| x_t - x_{t-1} \|^2$ in this sum happen to correspond \textit{exactly} to the positive terms one obtains in the regret bound for the $y$-player, but this is true \textit{only as a result of} using the \OFTL algorithm. To obtain a cancellation of these terms, we need a $\gamma_t$ which is roughly constant, and hence we need to ensure that $\frac{\alpha_t^2}{A_t} = O(1)$. The final bound, of course, is determined by the inverse quantity $\frac 1 {A_T}$, and a quick inspection reveals that the best choice of $\alpha_t = \theta(t)$. This is not the only choice that could work, and we conjecture that there are scenarios in which better bounds are achievable for different $\alpha_t$ tuning. We show in Subsection~\ref{sub:acclinear} that a \emph{linear rate} is achievable when $f(\cdot)$ is also strongly convex, and there we tune $\alpha_t$ to grow exponentially in $t$ rather than linearly.

\paragraph{Infinite memory method.} We finish this section by mentioning one additional algorithm. Nesterov has proposed another first-order method with $O(1/T^2)$ convergence,
which requires maintaining a weighted average of gradients along the algorithm's path.
%, although the technique has the slight downside in that it requires maintaining the full sequence of gradients observed along the algorithm's path. But 
This method also has a natural and and simple interpretation in our FGNRD framework, and we include it here for completeness.

\begin{algorithm}[h] 
   \caption{ Nesterov's $\infty$-memory method \cite{N05,T08} } \label{alg:Nes-infmem}
Given: $L$-smooth $f(\cdot)$, convex domain $\K \subseteq \reals^d $, arbitrary $v_0=x_0 \in \K$, 
%\junkun{I think the correct version is ``arbitrary $v_0 \in \K$''}
1-strongly convex regularizer $R(\cdot)$, iterations $T$.
%\junkun{(consistency) regularization function or Prox. function}
\begin{center} 
\begin{tabular}{c c}
    $
      \boxed{\small
      \begin{array}{rl}
        \beta_t & \gets \frac{2}{t+1}, \gamma_{t} \gets \frac{t}{4L} \\\
        z_{t} & \gets (1 - \beta_t) w_{t-1} + \beta_t v_{t-1}\\
        v_{t} & \displaystyle \gets \argmin_{x \in \K} 
          \sum_{s=1}^{t} \gamma_s \langle \nabla f(z_s), x \rangle + R(x)
        \\
        w_{t} & \gets (1 - \beta_t) w_{t-1} + \beta_t v_{t}
      \end{array}
      }
    $
    & \small
    $\boxed{\small
    \begin{array}{rl}
      g(x,y) & := \langle x, y \rangle - f^*(y)\\
      \alpha_t & := t\text{ for } t=1, \ldots, T\\
      \alg^Y & := \OFTL    \\  %\junkun{\text{It should be }\OFTL[\nabla f(v_0)] } \\
      \alg^X & := \BTRL[R(\cdot), \frac{1}{4L}] \\
    \end{array}
    }$
  \\
    \small Iterative Description & 
    \small FGNRD Equivalence
\end{tabular}
\end{center}
Output: $w_T = \bar x_T$,
\end{algorithm}

We leave it to the reader to establish the equivalence described in Algorithm~\ref{alg:Nes-infmem}, and here we state the convergence rate whose proof is deferred to Appendix~\ref{app:thm:smoothAcc}. 
%\jake{Jun-Kun is this in the appendix?}.

\begin{theorem} \label{thm:smoothAcc}
Assume that $f(\cdot)$ is $L$-smooth convex. Nesterov's $\infty$-memory method (Algorithm~\ref{alg:Nes-infmem}) satisfies
\[
   \displaystyle f(w_T) - \min_{w \in \K} f(w) 
  \leq O\left( \frac{ L \big( R(w^*) - \min_{x\in \K} R(x) \big) } {T^2} \right).
\]
\end{theorem}

% On the other hand, following the same setting, if $\alg^x$ is chosen as \BTRL with a $\beta$-strongly convex regularizer $R(\cdot)$ and a parameter $\eta$. Then the point $\xav_T$ satisfies
% \begin{equation} \label{eq:genericbound2}
% \displaystyle  f(\xav_T) - \min_{x \in \XX} f(x)
%   \leq \frac 1 {A_T} \left( \frac{ R(x^*) - R(\hat{x})}{\eta }  
%     + \sum_{t=1}^{T} \left(\frac{\alpha_t^2}{A_t} L - \frac{\beta}{2 \eta} \right)
%          \| x_{t-1} - x_t \|^2    \right),
% \end{equation}
% where $R(\hat{x}) = \min_{x \in \XX} R(x)$.

% \jake{proof segment:} On the other hand, if the $y$-player is $\OFTL$ and the $x$-player is $\BTRL$, then, by Lemma~\ref{lem:yregretbound} of \OFTL and Lemma~\ref{regret:BTRL} of \BTRL with $\mu=0$ (as the $x$-player sees linear loss functions), we have 
% \begin{equation} 
% \displaystyle  f(\xav_T) - \min_{x \in \XX} f(x)
%   \leq \frac 1 {A_T} \left( \frac{ R(x^*) - R(\hat{x})}{\eta }  
%     + \sum_{t=1}^{T} \left(\frac{\alpha_t^2}{A_t} L - \frac{\beta}{2 \eta} \right)
%          \| x_{t-1} - x_t \|^2    \right),
% \end{equation}
% where $R(\hat{x}) = \min_{x \in \XX} R(x)$.

% Similarly,
% if the $x$-player is $\BTRL$ 
% with a $1$-strongly convex regularizer $R(\cdot)$,
% and the parameter $\gamma$
% satisfies
% $\frac{1}{CL} \leq \eta \leq \frac{1}{4L}$ for some constant $C \geq 4$
% then
% \[
%    \displaystyle f(\xav_T) - \min_{x \in \XX} f(x) 
%   \leq \frac{2 C L \big( R(x^*) - R(\hat{x}) \big) } {T^2},
% \]
% where $R(\hat{x}) = \min_{x \in \XX} R(x)$.

% \jake{END infinite memory stuff}

\subsubsection{Accelerated proximal methods}

\begin{algorithm}[H] 
   \caption{ Accelerated proximal method.} \label{alg:AccProx}
Given: $L$-smooth $f(\cdot)$, arbitrary $v_0 = x_0 \in \reals^d$, iterations $T$.
%$g(x,y):= \langle x, y \rangle - f^*(y) + \psi(x)$.
\begin{center} 
\begin{tabular}{c c}
    $
      \boxed{\small
      \begin{array}{rl}
        \beta_t & \gets \frac{2}{t+1}, \gamma_{t} \gets \frac{t}{4L}\\
        z_{t} & \gets (1 - \beta_t) w_{t-1} + \beta_t v_{t-1}\\
        v_{t} & \gets
\textbf{prox}_{t \gamma \psi} ( x_{t-1}- t \gamma \nabla f(z_t) ) 
%         \underset{x \in \K}{ \argmin} 
%          \gamma_t \langle  \nabla f(z_t), x \rangle  + \V{v_{t-1}}(x)
        \\
        w_{t} & \gets (1 - \beta_t) w_{t-1} + \beta_t v_{t}
      \end{array}
      }
    $
    & \small
    $\boxed{\small
    \begin{array}{rl}
      g(x,y) & := \langle x, y \rangle - f^*(y)+ \psi(x)\\
      \alpha_t & := t\text{ for } t=1, \ldots, T\\
      \alg^Y & := \OFTL \\
      \alg^X & := \MD[\phi(\cdot)= \frac{1}{2} \| \cdot \|^2_2, \frac{1}{4L}] \\
    \end{array}
    }$
  \\
    \small Iterative Description & 
    \small FGNRD Equivalence
\end{tabular}
\end{center}
Output: $w_T = \bar x_T$
\end{algorithm}

In this section, we consider solving composite optimization problems
\begin{equation}
\min_{x \in \reals^d} f(x) + \psi(x),
\end{equation}
where $f(\cdot)$ is smooth convex but $\psi(\cdot)$ is possibly non-differentiable convex. Examples of $\psi(\cdot)$ include $\ell_1$ norm ($\|\cdot\|_1$), $\ell_{{\infty}}$ norm ($\| \cdot\|_{{\infty}}$), and elastic net ($\| \cdot \|_{1} + \gamma \| \cdot \|^2$ for $\gamma > 0$).
%\jake{Jun-Kun: can you give one example of such a function? What is the standard example here?}
We would like to show that the game analysis still applies to this problem.
We just need to change the payoff function $g$ to account for $\psi(x)$.
Specifically, we consider the following two-player zero-sum game,
\begin{equation} \label{eq:gnew}
 \min_{x \in \reals^d} \max_{y \in \YY} \left \{ g(x,y):= \langle x, y \rangle - f^*(y) + \psi(x)  \right\}.
\end{equation} 
Notice that the minimax value of the game is $\min_{x \in \reals^d} f(x) + \psi(x)$, which is exactly the  optimum value of the composite optimization problem.
Let us denote the proximal operator 
\footnote{It is known that for some $\psi(\cdot)$, their corresponding proximal operations have closed-form solutions, see e.g. \cite{PB14} for details.}
as 
$\textbf{prox}_{\lambda \psi} (v) := \argmin_{x \in \reals^d} \big( \psi(x)+ \frac{1}{2\lambda} \| x - v \|^2_2 \big).$

\iffalse
\begin{algorithm}[H] 
   \caption{Accelerated Proximal Method} \label{alg:AccProx}
\begin{algorithmic}[1]
\State In the weighted loss setting of Protocol~\ref{alg:game} (let $\alpha_t = t$ and $\gamma = \frac{1}{4L}$):
\State \quad $y$-player uses \OFTL as $\alg^y$: $y_t = \nabla f( \xof_{t})$.
\State \quad $x$-player uses \MD  with $\psi(x):= \frac{1}{2} \|x \|^2_2$ in Bregman divergence
as $\alg^x$: \\ \State \qquad 
$x_{t} = \argmin_{x}   \gamma ( \alpha_t h_t(x) ) + V_{x_{t-1}}(x)
= \argmin_{x}   \gamma ( \alpha_t \{  \langle x , y_t \rangle + \psi(x)   \} ) + \V{x_{t-1}}(x)$
\\ \State \qquad
$ = \argmin_{x} \phi(x) + \frac{1}{2 \alpha_t \gamma} ( \| x \|^2_2 + 2 \langle \alpha_t \gamma y_t - x_{t-1} , x \rangle )  = \textbf{prox}_{\alpha_t \gamma \psi} ( x_{t-1}- \alpha_t \gamma \nabla f( \xof_{t}) ) $
\end{algorithmic}
\end{algorithm}
\fi

We remark that the FGNRD interpretation in Algorithm~\ref{alg:AccProx} is essentially the same as Algorithm~\ref{alg:Nes-1mem}, except here the payoff function $g$ is defined differently \eqref{eq:gnew}. The weighting scheme $\alpha_t$ and the players' strategies remain the same, and we thus omit the equivalence between the two interpretations of Algorithm~\ref{alg:AccProx}.

In this new game, the $x$-player plays $\MD$ with the distance generating function $\phi(x)= \frac{1}{2} \| x \|^2_2$
and receives the loss functions $\alpha_t h_t(x) := \alpha_t \{  \langle x , y_t \rangle + \psi(x)   \} $, which leads to the following update,
\begin{equation*}
\begin{split}
x_{t} & = \argmin_{x \in \reals^d}   \gamma ( \alpha_t h_t(x) ) + D_{x_{t-1}}^{\phi}(x)
= \argmin_{x \in \reals^d}   \gamma ( \alpha_t \{  \langle x , y_t \rangle + \psi(x)   \} ) + \V{x_{t-1}}(x)
%\\ \State \qquad
\\ & = \argmin_{x \in \reals^d } \psi(x) + \frac{1}{2 \alpha_t \gamma} ( \| x \|^2_2 + 2 \langle \alpha_t \gamma y_t - x_{t-1} , x \rangle ) 
\\ & = \textbf{prox}_{\alpha_t \gamma \psi} \left( x_{t-1}- \alpha_t \gamma \nabla f( \xof_{t}) \right),
\end{split}
\end{equation*}
where the last equality follows that
$y_t = \nabla f(\tilde{x}_t)$ and the definition of the operator $\textbf{prox}(\cdot)$.
One can view Algorithm~\ref{alg:AccProx} as a variant of the so called ``Accelerated Proximal Gradient''in \cite{BT09}.
Yet, the design and analysis of our algorithm is simpler than that of \cite{BT09}.

\begin{theorem} \label{thm:proximal}
%Suppose that $D_{x_0}^{\phi}(x^*) \leq D$.
The update $w_T = \bar{x}_T$ in Algorithm~\ref{alg:AccProx}
satisfies 
\[
f(w_T) - \min_{w \in \reals^d} f(w) \leq O\left(\frac{ L \| w_0 - w^*\|^2 }{T^2}\right).
\] 
\end{theorem}

\begin{proof}
Even though the payoff function $g(\cdot,\cdot)$ is a bit different,
the proof still essentially follows the same line as
Theorem~\ref{thm:metaAcc}, as $y$-player plays $\OFTL$ and the $x$-player plays $\MD$. 
\end{proof}

\subsubsection{Related works}
%\junkun{I put related works regarding acceleration here.}
In recent years, there are growing interest in giving new interpretations of Nesterov’s accelerated algorithms or proposing new varaints. For example, \cite{T08} gives a unified analysis for some Nesterov’s accelerated algorithms
\cite{N88,N04,N05}, using the standard techniques and analysis in optimization literature.
\cite{LRP16,HL17} connects the design of accelerated algorithms with dynamical systems and control theory. \cite{BLS15,DFR18} gives a geometric interpretation of the Nesterov’s method for unconstrained optimization, inspired by the ellipsoid method.
\cite{FB15} studies the Nesterov’s methods and the Heavy Ball method for quadratic non-strongly convex problems by analyzing the eigen-values of some linear dynamical systems. 
%\cite{SRBA17} studies acceleration via the lens of numerical analysis an
\cite{AO17} proposes a variant of accelerated algorithms by mixing the updates of gradient descent and mirror descent and showing the updates are complementary.
%\junkun{We should be careful in describing \cite{LZ17,L20} as well as \cite{DO18,DR70} and \cite{CST20}}
\cite{DO18,DO19} propose a primal-dual view that recovers several first-oder algorithms with careful discretizations of a continuous-time dynamic, which also leads to a new accelerated extra-gradient descent method. 
\cite{CST20} show a simple acceleration proof of mirror prox \cite{Nemi04} 
and dual extrapolation \cite{N07} based on solving the Fenchel game.
%\junkun{ How to properly cite our papers?
%\cite{AW17,ALLW18,WA18} }
\cite{SBC14,wibisono2016variational,SDJS18,KBB15,SRBA17,WRJ21,ACPR18} analyze the acceleration algorithms via the lens of ordinary differential equations and the numerical analysis.
See also a monograph \cite{ADP21} for various methods of acceleration,
%\cite{nemirovski1983}
and a recent work \cite{wang2022provable} that shows acceleration of Heavy Ball compared to standard gradient descent for minimizing a class of Polyak-\L{}ojasiewicz functions when the non-convexity is averaged-out. 
We also note an independent work \cite{LZ17,L20} provide a game interpretation of Nesterov's accelerated method but does not have the idea of the regret analysis. In our work, we show a deeper connection between online learning
and optimization, and propose a modular framework that is not limited to Nesterov's method. 
%We also note that
%in recent years there has emerged a lot of work where learning problems are treated as repeated games, and many researchers have been studying the relationship between game dynamics and provable convergence rates (see e.g.\cite{abernethy2008optimal,balduzzi2018mechanics,gidel2018negative,daskalakis2017training,NeurIPS2013_5148}).

%\junkun{Todo: cite more works}
Finally, 
%it would not be complete if we do not describe Nesterov's smoothing technique \cite{N05}.
the accelerated algorithms in this work require the function to be smooth.
For minimizing general non-smooth convex functions, the lower-bound of the convergence rate is $O(1/\sqrt{T})$ for any first-order gradient-based methods, see e.g., Subsection 3.2.1 in \cite{N04}.
However, Nesterov \cite{N05} show that when a function has a specific structure, an $O(1/T)$ rate can be obtained via a smoothing technique.
Specifically, the class of non-smooth functions is in the following form:
\begin{equation} \label{fff}
f(x) =  \max_{z \in \Theta} \langle M x, z \rangle - \psi(z),
\end{equation}
where
$\Theta$ is a compact convex set, $\psi(\cdot): \Theta \to \reals$ is convex, and $M: \K \to \Theta$ is a linear operator.  
The smoothing technique that Nesterov introduced is first by constructing a smooth approximation of $f(\cdot)$ in \eqref{fff}:
\begin{equation}
f_{\nu}(x) = \max_{z \in \Theta} \langle M x, z \rangle - (\psi(z) + \nu d(z) ),
\end{equation}
where $d(\cdot)$ is $1$-strongly convex on $\Theta$.
Since $\psi(z) + \nu d(z)$ is a strongly convex function, one can therefore show that
the approximated function $f_{\nu}(\cdot)$ is smooth.
Then, one can apply those accelerated methods for smooth convex functions to $f_{{\nu}}(\cdot)$.
By tuning the parameter $\nu$ to trade-off the smoothness constant of $f_{\nu}(\cdot)$ and the approximation error of $f_{{\nu}}(\cdot)$ to $f(\cdot)$ appropriately, one can get an $O(1/T)$ rate for minimizing the original function $f(\cdot)$. We refer the reader to the original paper \cite{N05} for the details.
We also note that there are other smoothing techniques in the literature, e.g., Moreau envelope  \cite{beck2012smoothing}.

\subsection{Proving Accelerated Linear Rates} \label{sub:acclinear}

\begin{algorithm}[h] 
   \caption{ Accelerated Gradient with Linear Convergence } \label{alg:accLinear}
Given: $L$-smooth, $\mu$-strongly convex $f(\cdot)$, convex domain $\K \subseteq \reals^d $,
 iterations $T$, and a distance generating function $\phi(\cdot)$ that is
$1$-strongly convex, $L_{\phi}$-smooth, and differentiable.
Finally, let $R(x) = \phi(x)$. \\
Init: $w_0=v_0 = x_0 = \argmin_{x \in \K} \phi(x)$.
%\jake{Need to work on GD parameters}
\begin{center} 
\begin{tabular}{c c}
    $
      \boxed{\small
      \begin{array}{rl}
        \beta & \gets \frac{1}{2} \sqrt{ \frac{\mu}{L(1+L_{\phi})} } , \gamma_{t} \gets \alpha_t \\
        z_{t} & \gets (1 - \beta) w_{t-1} + \beta v_{t-1}\\
        \Phi_t(x) & := \displaystyle \sum_{s=1}^{t} \gamma_s \big( \langle \nabla \tilde{f}(z_s), x \rangle  + \mu \phi(x) \big) \\
        v_{t} & \gets \argmin_{x \in \K} R(x) + \Phi_t(x)
         \\ & 
        \\
        w_{t} & \gets (1 - \beta) w_{t-1} + \beta v_{t}
      \end{array}
      }
    $
    &
    $\boxed{\small
    \begin{array}{rl}
      \tilde f(x) & := f(x) - \mu \phi(x)  \\
      g(x,y) & := \langle x, y \rangle - \tilde f^*(y) + \mu \phi(x) \\
      \alpha_1 & = \frac{1}{2L(1+L_{\phi})}  \\
      \frac{\alpha_t}{A_t} & = \frac{1}{2} \sqrt{ \frac{\mu}{L(1+L_{\phi})} }\\
      & \text{ for } t=2, \ldots, T
       \\
 %     \alpha_t & := \frac{ \theta}{1-\theta} A_{t-1} 
      %\frac{ (\alpha_1 + \cdots + \alpha_{t-1})}{\sqrt{8L/\mu} - 1} 
  %    \text{ for } t=2, \ldots, T,   \\ \text{ where  } \theta  & := \frac{1}{2} \sqrt{ \frac{\mu}{L(1+L_{\phi})} }\\
      \alg^Y & := \OFTL \\
      \alg^X & := \BTRL[\phi, 1] \\
    \end{array}
    }$
  \\
    \small Iterative Description & 
    \small FGNRD Equivalence
\end{tabular}
\end{center}
Output: $w_T = \bar x_T$
\end{algorithm}

Nesterov observed that for strongly convex smooth functions,
%, when $f(\cdot)$ is both $\mu$-strongly convex and $L$-smooth, 
one can achieve an accelerated linear rate (e.g. page 71-81 of \cite{N04}). 
%a rate that is exponentially decaying in $T$ 
%(e.g. page 71-81 of \cite{N04}). 
It is natural to ask if the zero-sum game and regret analysis in the present work
also recovers this faster rate in the same fashion. We answer this in the affirmative.
%Denote $\kappa := \frac{L}{\mu}$.
%In the following,
%we assume
%that the function $f(\cdot)$ is $L$-smooth 
%with respect to some norm $\| \cdot \|$
%and there exists a differentiable function $r(\cdot)$ that is $L_{\phi}$-smooth and $1$-strongly convex with respect to the same norm $\| \cdot \|$.
%Furthermore, 
Assume $f(\cdot)$ is $\mu$-strongly convex in the following sense (see also Section 3.3 of \cite{L20}),
\begin{equation} \label{new_sc}
f(z) \geq f(x) + \langle \nabla f(x), z - x \rangle + \mu \V{x}(z), 
\end{equation}
for all $z, x \in \K$, where $\V{x}(z)$ is the Bregman divergence.
% with the distance generating function $r(\cdot)$, i.e. $V_x^{r}(z):= r(z) - r(x) + \langle \nabla \phi(x), z - x \rangle$.
%In the case that the norm is the $l_2$ norm, i.e. $\| \cdot \|_2$, we can define
When the distance generating function is
$\phi(x):= \frac{1}{2} \| x \|^2_2$, 
%(and hence $L_{\phi}=1$), and
the strong convexity condition (\ref{new_sc}) becomes
\begin{equation}
f(z) \geq f(x) + \langle \nabla f(z), z - x \rangle + \frac{1}{2} \| z - x \|^2_2.
\end{equation}
%In the following, we will denote $\kappa := \frac{L}{\mu}$.

It is known that the function $\tilde{f}(x):= f(x) - \mu \phi(x)$ is a convex function for all $x \in \K$ (see e.g. \cite{LFN18}). Based on this property, we consider a new game
\begin{equation} \label{split2}
\displaystyle \tilde g(x,y):= \langle x, y \rangle - \tilde{f}^*(y) + \mu \phi(x),
\end{equation}
where the minimax vale of the game is
%\begin{equation}
$V^* := \min_{x \in \K} \max_{y \in \YY} \tilde g(x, y) = \min_{x \in \K} \tilde{f}(x) + \mu \phi(x) = \min_{x \in \K} f(x)$.
We begin by showing the following in this new game.

\begin{proposition}\label{prop:Nes_linrate_equiv} 
 The two interpretations of Algorithm~\ref{alg:accLinear} 
 are equivalent. That is, the iteratively generated points $w_t$ coincide with the sequence $\bar x_t$ produced by the FGNRD dynamic.
\end{proposition}

\begin{proof}
The proof follows the same lines as that of Proposition~\ref{thm:Nes_constrained} until the last line
by setting $\beta_{t}= \beta = \frac{1}{2} \sqrt{ \frac{\mu}{L(1+L_{\phi})} }   = \frac{\alpha_t}{A_t}$. 
We only need to replace the last line of the proof. Specifically, the update of $v_t$ is exactly equivalent to \BTRL
on the loss sequences $\alpha_t \ell_t(\cdot) \leftarrow \alpha_t \left(  \langle \nabla \tilde{f}(z_t) , \cdot \rangle  + \mu \phi(\cdot) \right) $. Hence, the proof is completed. 
\end{proof}

We now prove the accelerated linear convergence rate of Algorithm~\ref{alg:accLinear}. First, we denote the \emph{condition number} $\kappa := \frac{L}{\mu}$.
\begin{theorem} \label{thm:acc_linear2}
Assume we are given a norm $\| \cdot \|$, and that we have a distance-generating function $\phi(\cdot)$ that is differentiable, $L_{\phi}$-smooth, and $1$-strongly convex with respect to $\| \cdot \|$.
Assume we also have a function $f(\cdot)$ that is $L$-smooth with respect to
$\| \cdot \|$ and $\mu$-strongly convex in the sense of (\ref{new_sc}).
% Define the game
% $\tilde g(x,y):= \langle x, y \rangle - \tilde{f}^*(y) + \mu  \phi(x)$.
% If the $y$-player plays \OFTL: $y_t \leftarrow \nabla \tilde{f}(\tilde{x}_t)$
% and the $x$-player plays \BTRL:
% $
% x_t \leftarrow \argmin_{{x \in \XX}} \sum_{s=1}^t \alpha_{s} h_{s}(x) + R(x),
% $
% where $R(x):= \alpha_{0} \mu  \phi(x)$, then 
% the weighted average points $(\xav_T, \yav_T)$ would be an 
Then the output of Algorithm~\ref{alg:accLinear} satisfies
\[
  f(w_T) - \min_{w \in \K} f(w) \leq c \exp\left(-\frac{1}{2\sqrt{\kappa}\sqrt{1 + L_\phi}} T \right) \left( \phi(w^*) - \phi(w_0) \right) 
\]
for some constant $c := 2 L (1+L_{\phi}) > 0$.
\end{theorem}
\begin{proof}
%The proof basically follows the same line as that of Theorem~\ref{thm:acc_linear}.
As the proof of Lemma~\ref{lem:yregretbound},
we first bound the regret of the $y$-player as follows. 
    \begin{eqnarray*}
         \displaystyle \sum_{t=1}^{T} \alpha_t \ell_t(\yof_t) - \alpha_t \ell_t(y^*) 
        %& \leq &  \displaystyle \sum_{t=1}^T \delta_t(\yof_t) - \delta_t(\yftl_{t+1}) \\
        & \leq & \displaystyle  \sum_{t=1}^T \alpha_t \langle x_{t-1} - x_{t}, \yof_t - \yftl_{t+1} \rangle \\
        \text{(Eqns. \ref{eq:haty}, \ref{eq:tildey})} 
        & = & \displaystyle  \sum_{t=1}^T \alpha_t \langle x_{t-1} - x_{t}, \nabla \tilde{f}(\xof_t) - \nabla \tilde{f}(\xav_t) \rangle \\
        \text{(H\"older's Ineq.)}  
        & \leq & \displaystyle  \sum_{t=1}^T \alpha_t \| x_{t-1} - x_{t}\| \| \nabla \tilde{f}(\xof_t) - \nabla \tilde{f}(\xav_t) \|_* \\
        & = & \displaystyle  \sum_{t=1}^T \alpha_t \| x_{t-1} - x_{t}\| \| \nabla f(\xof_t) -\mu \nabla \phi( \xof_t ) - \nabla f(\xav_t) + \mu \nabla \phi( \xav_t) \|_* \\
                        \text{\footnotesize (Triangle Ineq. $\&$ Smooth $\phi$)}   
        & \leq & \displaystyle  \sum_{t=1}^T \alpha_t \| x_{t-1} - x_{t}\| ( \| \nabla f(\xof_t) - \nabla f(\xav_t) \|_* +  \mu L_{\phi} \| \xav_t - \xof_t \| ) \\
                \text{\footnotesize (Smoothness of $f$ $\&$ $\mu \leq L$)} 
        & \leq & \displaystyle  L (1+L_{\phi}) \sum_{t=1}^T \alpha_t \| x_{t-1} - x_{t}\| \|\xof_t - \xav_t \| \\
        \text{(Eqn. \ref{eq:xdiffs})}  
        & = & \displaystyle L (1+L_{\phi}) \sum_{t=1}^T \frac{\alpha_t^2}{A_t} \| x_{t-1} - x_{t}\| \|x_{t-1} - x_{t} \|.
    \end{eqnarray*}
So the regret satisfies
\begin{equation} \label{reg_y_accH-2}
\displaystyle
\regret{y} \leq L(1+L_{\phi}) \sum_{t=1}^T \frac{\alpha_t^2}{A_t} \|x_{t-1} - x_t \|^2.
\end{equation}
For the $x$-player, 
according to Lemma~\ref{regret:BTRL}, its regret is 
\begin{equation} \label{reg_x_acc0-2}
\begin{aligned}{}
\displaystyle
\regret{x} 
\leq R(x^*) - R(x_0) - \sum_{t=1}^T \left( \frac{\mu A_{t-1} }{2} + \frac{1}{2}  \right)\|x_{t-1} - x_t \|^2_2,
\end{aligned}
\end{equation} 
where $x_{0} = \argmin_{x} R(x)$.
Summing (\ref{reg_y_accH-2}) and (\ref{reg_x_acc0-2}), we have
\begin{equation*}
\begin{aligned}
& \regret{y}[\OFTL] + \regret{x}[\BTRL]
\\ & \leq R(x^*) - R(x_0) + \sum_{t=1}^T \left( \frac{L(1+L_{\phi}) \alpha_t^2}{A_t} - \frac{\mu A_{t-1} + 1}{2} \right) \|x_{t-1} - x_t \|^2_2.
\end{aligned}
\end{equation*}
By choosing the weight $\{\alpha_t\}$ to satisfy $\alpha_1 = \frac{1}{ 2 L (1+L_{\phi})}$ and that for $t \geq 2$, $\frac{\alpha_t}{A_t} = \frac{1}{2} \sqrt{ \frac{\mu}{L(1+L_{\phi})} }$, the coefficient of the distance terms will be non-positive for all $t$, i.e.
$(\frac{ L(1+L_{\phi}) \alpha_t^2}{A_t} - \frac{\mu A_{t-1}}{2} ) \leq 0 $, which means that the distance terms will cancel out.
%To see this, let $\frac{\alpha_t}{A_t} = \theta$ for some constant $\theta >0$, we have 
%\begin{equation}
%\begin{split}
%\frac{L(1+L_{\phi}) \alpha_t^2}{A_t} - \frac{\mu A_{t-1} + 1 }{2}
%& = L(1+L_{\phi}) \theta^2 A_t - \frac{\mu}{2} ( A_t(1-\theta) + 1 )
%\\ & \leq A_t \left( L(1+L_{\phi}) \theta^2 - \frac{\mu}{2}(1-\theta) \right).
%\end{split}
%\end{equation}
%So it suffices to have that $L(1+L_{\phi}) \theta^2 - \frac{\mu}{2}(1-\theta) \leq 0$, which can be guaranteed by choosing $\theta = \frac{1}{2} \sqrt{ \frac{\mu}{L(1+L_{\phi})} }$.
Therefore, the optimization error $\epsilon$ after $T$ iterations satisfies that
\begin{equation*}
\begin{aligned}
& \epsilon \leq
\frac{ \regret{y} + \regret{x} }{A_T} 
\leq \frac{1}{A_1} \frac{A_1}{A_2} \cdots \frac{A_{T-1}}{A_T}
( R(x^*) - R(x_0)  ) 
\\ & =  \frac{1}{A_1} \left(1 - \frac{\alpha_2}{A_2} \right) \cdots \left(1 - \frac{\alpha_T}{A_T} \right)
( R(x^*) - R(x_0)  ) 
\\ & \leq \frac{1}{A_1} \left(1 - \frac{\alpha_2}{\tilde{A}_2} \right) \cdots \left(1 - \frac{\alpha_T}{\tilde{A}_T} \right)
( R(x^*) - R(x_0)  ) 
\\ & \leq \left(1- \frac{1}{2 \sqrt{1+L_{\phi}} \sqrt{\kappa}} \right)^{T-1} \frac{R(x^*) - R(x_0) }{A_1}
= O\left( \exp\left( - \frac{1}{2 \sqrt{1+L_{\phi}} \sqrt{\kappa}} T \right) \right).
\end{aligned}
\end{equation*}
%which is $O\left( \big(1 - \frac{1}{2 \sqrt{1+L_{\phi}} \sqrt{\kappa}} \big)^{T}\right) = O\left( \exp\big( - \frac{1}{2 \sqrt{1+L_{\phi}} \sqrt{\kappa}} T \big) \right)$. 
\end{proof}

\section{New algorithms}
\label{sec:NewAlgs}

In Section~\ref{sec:ExistingAlgs} we established that the FGNRD framework can be used to describe a set of known algorithms for convex optimization under various assumptions. But we have argued that the framework is indeed quite generic, given that we have flexibility to choose different no-regret update algorithms, to change the game's payoff function, and to modify the protocol. In the present section we use FGNRD to describe two novel algorithms, and to provide strong convergence rates using the same methodology.

\subsection{Boundary Frank-Wolfe}

In original formulation of the Fenchel Game No-Regret Dynamics, Protocol~\ref{alg:game}, we specified that the two players act in a particular order, first the $y$-player followed by the $x$-player. This ordering is important, as it allows us to consider scenarios in which $x_t$ can be chosen in response to $y_t$, as was the case in nearly all of the algorithms proposed in Section~\ref{sec:ExistingAlgs}. But we can instead reverse the order of the two players, allowing the $x$-player to act before the $y$-player, while otherwise following the same protocol. This will give us a new algorithm for constrained convex optimization, \emph{Boundary Frank-Wolfe}, which has the unusual property that it only computes gradients of $f$ on the boundary of the constraint set $\K$. We will show the algorithm has a $O(\log T/T)$ convergence rate for any convex Lipschitz $f$, without any smoothness nor strong-convexity requirement for $f$, yet under two conditions: (1) the true minimizer of $f$ in $\K$ must exist on the boundary, and $\K$ must be a \emph{strongly convex set}. The result relies on recent work in online convex optimization under strongly convex constraints \cite{HLGS16}.
(Notably, there is prior work that develops a Frank-Wolfe-like algorithm that  strongly convex decision sets \cite{D15} with  $O( 1 / T^2 )$ convergence, but this result requires that $f$ is both smooth and strongly convex as well.)

\begin{algorithm}[H] 
   \caption{ Boundary Frank-Wolfe } \label{alg:new_fw}
Given: possibly non-smooth convex $f(\cdot)$, convex domain $\K \subset \reals^d $, arbitrary $w_0 = z_0$, iterations $T$.
\begin{center} 
\begin{tabular}{c c}
    $
      \boxed{\small
      \begin{array}{rl}
 z_t & \gets \argmin_{z \in \K} \frac{1}{t-1} \sum_{s=1}^{t-1} \langle z , \delta_s  \rangle,
\delta_s \in \partial f(z_s)
  \\
 w_t & \gets  \frac{(t-1)w_{t-1} + z_t}{t}
      \end{array}
      }
    $
    & \small
    $\boxed{\small
    \begin{array}{rl}
      g(x,y) & := \langle x, y \rangle - f^*(y)\\
      \alpha_t & \gets 1\\
%      (B): \alpha_t & \gets 1\\
      \alg^X & := \FTL[z_0] \\
      \alg^Y & := \BR \\
    \end{array}
    }$
  \\
    \small Iterative Description & 
    \small FGNRD Equivalence
\end{tabular}
\end{center}
Output: $w_T := \bar x_T$.
\end{algorithm}

\iffalse
\begin{algorithm}[H] 
   \begin{algorithmic}[1]
   \caption{Boundary Frank-Wolfe}\label{alg:new_fw}
\State \textbf{Input:} Init. $x_{1} \in \K$.
\FOR{$t=2, 3 \dots , T$}
\State $ x_t \leftarrow \argmin_{x \in \K} \frac{1}{t-1} \sum_{s=1}^{t-1} \langle x , \partial f(x_s)  \rangle$
\ENDFOR
\State Output: $\bar{x}_T  =  \frac{1}{T} \sum_{t=1}^T x_t$
\end{algorithmic}
\end{algorithm}

\begin{algorithm}[h] 
   \caption{Modified meta-algorithm, swapped roles}
   \label{alg:example2}
\begin{algorithmic}[1]
%\State Let the weighting scheme be $\alpha_t = 1$ for all $t$.
\FOR{$t= 1, 2, \dots, T$}
\State \quad   $x_t := \alg^X( g(\cdot,y_1), \ldots, g(\cdot,y_{t-1}) )$
\State \quad   $y_t :=  \alg^Y( g(x_1, \cdot), \ldots, g(x_{t-1}, \cdot), g(x_t, \cdot))$
\ENDFOR
\State Output: $\bar x_T = \frac 1 T \sum_{t=1}^T  x_t$ and $\bar y_T := \frac 1 T \sum_{t=1}^T y_t$ 
\end{algorithmic}
\end{algorithm}
\fi

The equivalence between the two forms of Algorithm~\ref{alg:new_fw} can be proven in the usual fashion, and so we omit the proof of the following proposition.
\begin{proposition}\label{prop:alg_new_fw_equiv} 
 The two interpretations of Algorithm~\ref{alg:new_fw} 
 are equivalent. That is, the iteratively generated points $w_t$ coincide with the sequence $\bar x_t$ produced by the FGNRD dynamic.
\end{proposition}

\begin{theorem} \label{thm:equiv2}
%  Algorithm~\ref{alg:new_fw} is a instance of Algorithm~\ref{alg:example2} if \textbf{(I)} Init. $x_1$ in Alg~\ref{alg:new_fw} equals $x_1$ in Alg.~\ref{alg:example2}; \textbf{(II)} Alg.~\ref{alg:example2} sets $\alg^X := \FTL$; and \textbf{(III)} Alg.~\ref{alg:example2} sets $\alg^Y := \BR$.
%  Furthermore, when 
Assume that the constraint set $\K$ is $\lambda$-strongly convex, and that $\sum_{s=1}^t \delta_s $, where $\delta_{s}\in \partial f(x_s)$ has non-zero norm, then
Boundary Frank-Wolfe (Algorithm~\ref{alg:new_fw}) has
  \[
  f(w_T) - \min_{w \in \K} f(w) = O\left(\frac{ M \log T}{ \lambda L_T T}\right),
  \]
  where $M:= \sup_{\delta \in \partial f(x), x \in \K} \| \delta \|$, 
$\Theta_t := \frac{1}{t} \sum_{s=1}^t   \delta_s$, $\delta_s \in \partial f( x_s)$,
and  $L_T:= \min_{1\leq t \leq T} \| \Theta_t \|$.
\end{theorem}

\begin{proof}
% The equivalence of the two displays shown on Algorithm~\ref{alg:new_fw}
% can be easily derived in an inductive manner. Specifically, we have
% %  \begin{eqnarray}
% $    \partial f(x_t)  =  y_t $. %\label{eq:ygradAcc}\\
% Note that we have chosen the weighting scheme be $\alpha_t = 1$ for all $t$.
Since $y$-player plays \BR, its regret $\leq 0$.
For the $x$-player, we use Lemma~\ref{thm:FTL},
which grants $\regret{x} \leq O(\frac{ M \log T}{ \lambda L_T})
$.
Using Theorem~\ref{thm:meta} we have
\begin{eqnarray*}
 f(\bar x_T) - \min_{x \in \K} f(x) & \leq & \frac{1}{T} ( \regret{x}[\FTL] + \regret{y}[\BR] ) \\ 
 & = & O\left(\frac{ M \log T}{ \lambda L_T T}\right)
\end{eqnarray*}
as desired.
\end{proof}

%\jake{JunKun this paragraph is a little confusing, can you explain?}

Note that the rate depends crucially on $L_T$, which is the smallest averaged-gradient norm computed during the optimization.
According to Theorem~\ref{th:cvx}, the closure of the gradient space is a convex set. This implies the following lemma that ensures $L_{T}$ is bounded away from $0$.
%We acknowledge that the value of $L_{T}$ 
\begin{lemma}  
If $\mathbf 0 \notin \text{closure}\left(\{y : y \in \partial f(x), x \in \K \}\right)$, 
then the average of gradient norm $L_{T}$ satisfies $L_T > 0$.
\end{lemma}

\begin{proof}
By Theorem~\ref{th:cvx}, we know $\Theta_{t}$, which is the average of some sub-gradients, is itself a subgradient. Moreover, $\Theta_{t}$ is not $\mathbf{0}$ given that $\mathbf{0}$ 
is not in the gradient space, hence we have $L_T>0$.

Alternatively, as $\mathbf{0}$ is not in the closure of the gradient space, which is a closed convex set by Theorem~\ref{th:cvx}, the separation theorem implies that any point in the closure must be bounded away from $\mathbf{0}$. Hence, $L_T>0$.
\end{proof}
We remark that if all the subgradients at the points in the constraint set $\K$ have their sizes be lower-bounded by a constant $G>0$, then we have that $L_{T}$ is lower-bounded by $G$, since the (closure) of the union of the subdifferentials is a convex set (Theorem~\ref{th:cvx}) and therefore $\Theta_{t}$, which is the average of some sub-gradients, is itself a subgradient.
%Now let us discuss when the boundary FW works; namely, the condition that causes the cumulative gradient being nonzero. If a linear combination of gradients is $\mathbf 0$ then clearly $\mathbf 0$ is in the convex hull of subgradients $\partial f(x)$ for boundary points $x$. Since
%the closure of $\{ \nabla f(x)| x \in \K \}$ is convex, according to Theorem~\ref{th:cvx}, this implies that $\mathbf 0$ is in $\{ \nabla f(x)| x \in \K \}$. If we know in advance that $\mathbf 0 \notin \text{closure}(\{ \nabla f(x)| x \in \K \})$ we are assured that the cumulative gradient will not be $\mathbf 0$. 
%Hence,

The proposed algorithm may only be useful when it is known, a priori, that the solution $w^*$ will occur not in the interior but on the boundary of $\K$. It is indeed an odd condition, but it does hold in many typical scenarios.
One may add a perturbed vector to the gradient and show that with high probability, $L_T$ is a non-zero number. The downside of this approach is that it would generally grant a slower convergence rate.

\subsection{Gauge Frank-Wolfe}

Using the FGNRD framework further, we propose a new algorithm that is ``\FW-like''  in the following sense: we can optimize a smooth objective function using only linear optimization oracle queries. This method requires the constraint set to be ``suitably round'' but achieves a rate of $O(1/T^2)$ without additional assumptions on $f$. The main trick is to utilize a regularization function that is adapted to the constraint set $\K$. With this in mind, define the ``gauge function'' of $\K$ \cite{F87,FMP14} as 
\begin{equation}
\g(x) := \inf \{ c \geq 0: \frac{x}{c} \in \K \}.
\end{equation}
Notice that, for a closed convex $\K$ that contains the origin, one has 
$\K = \{ x \in \reals^d: \g(x)\leq 1 \}$. Furthermore, the boundary points on $\K$ satisfy $\g(x)=1$.
\smallskip
\smallskip
\\
\noindent
\textbf{Definition: ($\lambda$-Gauge set)} We call a $\frac{\lambda}{2}$-strongly convex set that contains the origin in its interior a $\lambda$-Gauge set.

Examples of the $\lambda$-Gauge set are given in Appendix~\ref{app:betagauge}.

\begin{theorem} (Theorem~4 in \cite{P96} and Theorem 2 in \cite{M20})
Let $\K$ be a $\lambda$-Gauge set.
%a $\frac{\lambda}{2}$-strongly convex set that contains the origin in its interior. 
Then, its squared gauge function, i.e. $\g^2(\cdot)$, is $\lambda$-strongly-convex. 
%We call the constraint set a $\lambda$-Gauge set.
\end{theorem}
%\jake{Don't we have published research now that says that strongly convex sets are guage sets? Can't we call on this equivalence, and we don't even need to define ``gauge set''?}

Algorithmically, we will utilize \BTRL for the $x$-player with regularization $\g^2(\cdot)$ in the Fenchel Game. Naturally, this requires us to solve problems of the form
\begin{equation}\label{eq:gauge_FTRL}
    x = \argmin_{x \in \K}  \langle \zeta , x \rangle + \g^2(x),
\end{equation}
for any arbitrary vector $\zeta$. While at first glance this does not appear to be a linear optimization problem, let us reparametrize the problem as follows. Denote $\text{bndry}(\K)$ as the boundary of the constraint set $\K$, and notice that \eqref{eq:gauge_FTRL} is equivalent to solving
\begin{equation}\label{eq:gauge_FTRL2}
     \min_{\rho \in [0,1], z \in \text{bndry}(\K)}  \langle \zeta , \rho z \rangle + \g^2(\rho z)
     ~=~
     \min_{\rho \in [0,1]} \left[ \left( \min_{z \in \text{bndry}(\K)} \langle \zeta , z \rangle \right) \rho  + \rho^2\right],
\end{equation}
where we are able to remove the dependence on the gauge function since it is homogeneous, i.e.,
%$\g^2(\rho z ) = 
$\g(\rho z ) = \rho \g(z)$, and is identically 1 on the boundary of $\K$.
Notice that the inner minimization is now a linear problem, and the outer problem is a 1-parameter constrained quadratic problem which can be solved directly. In summary, we have
\begin{equation}\label{eq:guage_closed_form}
    \argmin_{x \in \K}  \langle \zeta , x \rangle + \g^2(x) = \rho^* z^* \quad \text{where} \; \begin{cases}
    z^* = \argmin_{z \in \text{bndry}(\K)} \langle \zeta , z \rangle , & \\
    \rho^* = \max(0, \min(1, -\langle \zeta, z^* \rangle/2) ). &
  \end{cases}
\end{equation}

% In the following, we introduce a family of \BTRL algorithms that rely solely on a linear oracle, and we believe this is a novel approach to online linear optimization problems. 
% The restriction we require is that the regularizer $R(\cdot)$ is chosen as the \emph{square gauge function} $\g^2(\cdot)$ for the decision set $\K$ of the learner. Here we will assume
% for every $t$ that $\ell_t(\cdot) = \langle l_t, \cdot \rangle$ for some vector $l_t$, hence \BTRL (\ref{update:BTRL}) reduces to
% \begin{equation}\label{eq:gauge_FTRL}
%     x_t = \argmin_{x \in \K} \eta \langle L_{t} , x \rangle + \g^2(x),
% \end{equation}
% where $L_{t} = l_1 + \ldots + l_{t}$. 
% Denote $\text{bndry}(\K)$ as the boundary of the constraint set $\K$.
% We can reparametrize the above optimization, by observing that any point $x \in \K$ can be written as $\rho z$ where $z \in \text{bndry}(\K)$, and $\rho \in [0,1]$. Hence we have 
% \begin{equation}\label{eq:gauge_FTRL}
%      \min_{\rho \in [0,1]} \min_{z \in \text{bndry}(\K)}  \eta \langle L_{t} , \rho z \rangle + \g^2(\rho z)
%      ~=~
%      \min_{\rho \in [0,1]} \left( \min_{z \in \text{bndry}(\K)} \eta \langle L_{t} , z \rangle \right) \rho  + \rho^2.
% \end{equation}
% We are able to remove the dependence on the gauge function since it is homogeneous, $\g(\rho x) = |\rho| \g(x)$, and is identically 1 on the boundary of $\K$. The inner minimization reduces to the linear optimization $z^* := \argmin_{z \in \K} \langle L_{t}, z \rangle$, and the optimal $\rho$ is $\max(0, \min(1, -(\eta/2) \langle L_{t}, z^* \rangle) )$.

\begin{algorithm}[h] 
   \caption{ Gauge Frank-Wolfe (smooth convex $f(\cdot)$) %\jake{Can we write solution of $\rho,\hat x$ in form of \eqref{eq:guage_closed_form}?}
   } \label{alg:gaugeFW}
Given: $L$-smooth convex $f(\cdot)$, arbitrary $v_0 = x_0 \in \K \subset \reals^d $, iteration $T$.
%\junkun{I think the correct version is ``arbitrary $v_0 \in \K$''}
%1-strongly convex distance generating function $\phi(\cdot)$, iterations $T$
%\junkun{(consistency) distance generating function or Prox. function}
\begin{center} 
\begin{tabular}{c c}
    $
      \boxed{\small
      \begin{array}{rl}
        \beta_t & \gets \frac{2}{t+1}, \gamma_{t} \gets \frac{t}{4L}\\
        z_{t} & \gets (1 - \beta_t) w_{t-1} + \beta_t v_{t-1}\\
     z_t^* & \gets {\displaystyle \argmin_{z \in \text{bndry}(\K)}} \langle \sum_{s=1}^t \gamma_s \nabla f(z_s) , z \rangle \\
    \gamma_t &  \gets \frac12 \langle \sum_{s=1}^t \gamma_s \nabla f(z_s), z^* \rangle\\
    \rho_t^* & \gets \max(0, \min(1, -\gamma_t) )\\   v_t & \gets \rho_t^* z_t^* \\                 
%       (\hat{x}_t, \rho_t) & \gets \underset{x \in \K, \rho \in[0,1] }{\argmin} \sum_{s=1}^t \gamma_s \rho \langle  x,  \nabla f(z_s) \rangle  +  \rho^2 \\
%       v_t & \gets \rho_t \hat{x}_t \\                 
        w_{t} & \gets (1 - \beta_t) w_{t-1} + \beta_t v_{t}
      \end{array}
      }
    $
    & \small
    $\boxed{\small
    \begin{array}{rl}
      g(x,y) & := \langle x, y \rangle - f^*(y)\\
      \alpha_t & := t\text{ for } t=1, \ldots, T\\
      \alg^Y & := \OFTL \\ %\junkun{\text{It should be }\OFTL[\nabla f(v_0)] } \\
      \alg^X & := \BTRL[\g^2(\rho x), \frac{1}{4L}] \\
    \end{array}
    }$
  \\
    \small Iterative Description & 
    \small FGNRD Equivalence
\end{tabular}
\end{center}
Output: $w_T = \bar x_T$,
\end{algorithm}

\iffalse
\begin{algorithm}[t] 
  \caption{Gauge Frank-Wolfe (smooth convex $f(\cdot)$)}
  \label{alg:gaugeFW}
  \begin{algorithmic}[1]
%    \State Input: step size $\eta$.
    \State Let $\{\alpha_t = t \}$ be a $T$-length weight sequence.
    \FOR{$t= 1, 2, \dots, T$}
    \State  The $y$-player plays \OFTL: $y_t = \nabla f(\tilde{x}_t)$.
    \State  The $x$-player plays BTRL:
    \State  Compute $(\hat{x}_t, \rho_t) = \underset{x \in \K, \rho \in[0,1] }{\argmin} \sum_{s=1}^t \rho \langle  x, \alpha_s y_s \rangle  +  \frac{1}{\eta} \rho^2$\quad \text{ and } set $x_t = \rho_t \hat{x}_t$.                 
    \ENDFOR
    \State Output $\bar{x}_{T} := \frac{ \sum_{s=1}^T \alpha_s x_s  }{ \sum_{s=1}^T \alpha_{t} }$.
  \end{algorithmic}
\end{algorithm}
\fi

\begin{theorem} \label{thm:gaugeFW}
Suppose the constraint set $\K$ is a $\lambda$-Gauge set.
Assume that the function $f(\cdot)$ is $L$-smooth convex with respect to the induced gauge norm $\g(\cdot)$. 
Suppose that the step size $\eta$ satisfies
$\frac{1}{CL} \leq \frac{\eta}{\lambda} \leq \frac{1}{4L}$ for some constant $C \geq 4$.
Then, the output $w_T$ of Algorithm~\ref{alg:gaugeFW} satisfies
\[
   \displaystyle f(w_T) - \min_{w \in \K} f(w) 
%  \leq \frac{2 C L \g^2(x^*)  } {\lambda T^2},
  \leq \frac{2 C L } {\lambda T^2}.
\]
\end{theorem} %\jake{Jun-Kun: isn't $\g^2(x^*) $ just bounded by 1, or the diameter of the set? Can we remove it from the above theorem?}

\begin{proof}
We have just shown that line 4-5 is due to that the $x$-player plays
\BTRL
with the square of the guage function as the regularizer.
So 
Algorithm~\ref{alg:gaugeFW} is an instance of Algorithm~\ref{alg:Nes-infmem},
and we can invoke Theorem~\ref{thm:smoothAcc} to obtain the convergence rate,
noting that $\g^2(x^*) \leq 1$.
\end{proof}

We want to emphasize again that our analysis does not need the function $f(\cdot)$ to be strongly convex 
to show $O(1/T^2)$ rate. On the other hand, 
\cite{D15} shows the $O(1/T^2)$ rate under the additional assumption that the function is strongly convex.
\cite{D15} raises an open question if an accelerated linear rate is possible under the assumption that the function is smooth and strongly convex. We partially answer the open question through Algorithm~\ref{alg:gaugeFW2} and Theorem~\ref{thm:gaugeFW-linear}.

\begin{algorithm}[h] 
   \caption{ Gauge Frank-Wolfe (smooth and strongly convex $f(\cdot)$) } \label{alg:gaugeFW2}
Given: $L$-smooth and $\mu$-strongly convex $f(\cdot)$, 
arbitrary $w_0 = v_0 = x_0 \in \K \subset \reals^d $, iterations $T$.
%and a distance generating function $\phi(\cdot)$ that is
%$1$-strongly convex, $L_{\phi}$-smooth, and differentiable.
%Finally, set an initial weight $\alpha_0 >0$.

%\jake{Need to work on GD parameters}
\begin{center} 
\begin{tabular}{c c}
    $
      \boxed{\small
      \begin{array}{rl}
        \beta & \gets \frac{1}{2} \sqrt{ \frac{\mu}{L(1+L_{\phi}/\lambda)} } , \gamma_{t} \gets \alpha_t \\
        z_{t} & \gets (1 - \beta) w_{t-1} + \beta v_{t-1}\\
%       (\hat{x}_t, \rho_t) & \gets \underset{x \in \K, \rho \in[0,1] }{\argmin} \big( \sum_{s=1}^t \gamma_s \rho \langle  x,  \nabla \tilde{f}(z_s) \rangle \\  & \qquad \qquad \qquad \quad +     \left( \frac{1}{\mu} + \sum_{s=1}^t \gamma_s \right) \frac{\mu}{\lambda} \rho^2 \big)  \\
     z_t^* & \gets \underset{ z \in \text{bndry}(\K)}{ \argmin} \langle \sum_{s=1}^t \gamma_s \nabla \tilde{f}(z_s) , z \rangle  \\
     \gamma_t & \gets \frac{\langle \sum_{s=1}^t \gamma_s \nabla \tilde{f}(z_s), z^* \rangle}{ 2 (\frac{1}{\mu}+\sum_{s=1}^t \gamma_s ) \frac{\mu}{\lambda}  } \\
    \rho_t^* & \gets \max\left(0, \min\left(1, - \gamma_t \right)\right) \\   v_t & \gets \rho_t^* z_t^* \\                 
        w_{t} & \gets (1 - \beta) w_{t-1} + \beta v_{t}
      \end{array}
      }
    $
    & \small
    $\boxed{\small
    \begin{array}{rl}
      g(x,y) & := \langle x, y \rangle - \tilde f^*(y) + \frac{\mu}{\lambda} \g^2(\cdot) \\
             & \text{where } \tilde f(x) := f(x) - \frac{\mu}{\lambda} \g^2(\cdot)  \\
      \alpha_1 & := \frac{1}{2L(1+L_{\phi}/\lambda)}  \\
      \frac{\alpha_t}{A_t} & := \frac{1}{2} \sqrt{ \frac{\mu}{L(1+L_{\phi}/\lambda)} }, \\ %\forall t \geq 2  \\
      %\frac{ (\alpha_1 + \cdots + \alpha_{t-1})}{\sqrt{8L/\mu} - 1} 
      & \text{ for } t=2, \ldots, T   \\ 
      \alg^Y & := \OFTL \\
      \alg^X & := \BTRL[ \frac{1}{\lambda}\g^2(\cdot), 1] \\
    \end{array}
    }$
  \\
    \small Iterative Description & 
    \small FGNRD Equivalence
\end{tabular}
\end{center}
Output: $w_T = \bar x_T$,
\end{algorithm}

\iffalse
\begin{algorithm}[t] 
  \caption{Gauge Frank-Wolfe (smooth and strongly convex $f(\cdot)$)}
  \label{alg:gaugeFW2}
  \begin{algorithmic}[1]
    \State Let $\{\alpha_t \}_{t=0}^T$ be a $T+1$-length weight sequence satisfying Theorem~\ref{thm:gaugeFW-linear}.
    \FOR{$t= 1, 2, \dots, T$}
    \State  The $y$-player plays \OFTL: $y_t = \nabla \tilde{f}(\tilde{x}_t)$,
    where $\tilde{f}(\cdot):= f(\cdot) - \frac{\mu}{\lambda} \g^2(\cdot) $.
    \State  The $x$-player plays BTRL:
    \State  Compute $(\hat{x}_t, \rho_t) = \underset{x \in \K, \rho \in[0,1] }{\argmin} \sum_{s=1}^t \rho \langle  x, \alpha_s y_s \rangle  +   \left( \sum_{s=0}^t \alpha_s \right) \frac{\mu}{\lambda} \rho^2$
    \State Set $x_t = \rho_t \hat{x}_t$.                 
    \ENDFOR
    \State Output $\bar{x}_{T} := \frac{ \sum_{s=1}^T \alpha_s x_s  }{ \sum_{s=1}^T \alpha_{t} }$.
  \end{algorithmic}
\end{algorithm}
\fi
%\junkun{Note the proof is not yet completed! We have a problem highlighted in the following!}

\begin{theorem} \label{thm:gaugeFW-linear}
Suppose the constraint set $\K$ is a $\lambda$-Gauge set 
and the induced square of the gauge function $\g^2(\cdot)$ is differentiable and $L_{\phi}$-smooth. 
Assume that the function $f(\cdot)$ is $L$-smooth with respect to the induced gauge norm $\g(\cdot)$ 
and $f(\cdot)$ is $\mu$-strongly convex satisfying (\ref{new_sc}) with $\frac{1}{\lambda}\g^2(\cdot)$ being the distance generating function. 
Then, the output $w_T$ of Algorithm~\ref{alg:gaugeFW2} satisfies
\[
   \displaystyle f(w_T) - \min_{w \in \K} f(w) 
  = O\left( \exp\left(-\frac{ T }{2 \sqrt{1+L_{\phi}/\lambda} \sqrt{\kappa}}\right) \right).
\]
\end{theorem}

%\jake{JunKun do we need anything here?} 
Lemma 2.2 and Theorem A.2 in \cite{KRDP21} implies that if a strongly convex set $\K$ which is centrally-symmetric at the origin satisfies the notion of a smooth set
\footnote{
For a compact convex set $\K$, 
it is smooth if and only if  
$ \# \{ N_{\K}(z) \cap  \partial \K^{\circ} \}  = 1$ for any point $z$ on the boundary of $\K$, i.e. $z \in \partial \K$,
where $N_{\K}(z):= \{ w \in \reals^d : \langle z - x, w \rangle \geq 0, \forall x \in \K\}$
is the normal cone at $z$ and $\K^{\circ} :=\{ w \in \reals^d : \langle w, x \rangle \leq 1 , \forall x \in \K  \}$ is the polar of $\K$, see e.g. Definition 2.1 in \cite{KRDP21}.},
then the induced square of the gauge function $\g^2(\cdot)$ is $L_{\phi}$-smooth and differentiable for some constant $L_{\phi}>0$;
%Let us make a brief remark about 
consequently, the condition of Theorem~\ref{thm:gaugeFW-linear} holds.
Examples include $l_p$ balls and Schatten $p$ balls for $1 < p \leq 2$ \cite{KRDP21}.

\begin{proof}

The equivalence of the two displays in Algorithm~\ref{alg:gaugeFW2} can be shown by following that of the proof of Proposition~\ref{prop:Nes_linrate_equiv},
where
% the correspondence $v_t = x_t$ can be seen as follows.
the update \BTRL: 
%with the regularizer $R(\cdot):= \frac{1}{\lambda} \g^2(\cdot)$ 
\[
x_t  \leftarrow \argmin_{{x \in \K}} \sum_{s=1}^t \alpha_{s} h_{s}(x) + R(x)
=  \argmin_{{x \in \K}} \sum_{s=1}^t \alpha_{s} \left( \langle y_s, x \rangle  + \frac{\mu}{\lambda} \g^2(x)  \right)+ \frac{1}{\lambda}  \g^2(x)
\]
is equivalent to solving
%\] which is equivalent to
\begin{equation*}
\begin{aligned}
%    & \min_{\rho \in [0,1]} \min_{z \in \text{bndry}(\K)}  \langle L_{t} , \rho z \rangle + \left( \frac{1}{\mu} + \sum_{s=1}^t \alpha_s \right) \frac{\mu}{\lambda} \g^2(\rho z)
%\\ & \equiv 
(\rho_t^*, z_t^*) \gets  
\argmin_{\rho \in [0,1]} \left( \argmin_{z \in \text{bndry}(\K)} \langle \sum_{s=1}^t \alpha_s \nabla\tilde{f} (z_s) , z \rangle \right) \rho  + \left( \frac{1}{\mu} + \sum_{s=1}^t \alpha_s \right) \frac{\mu}{\lambda} \rho^2,
\end{aligned}
\end{equation*}
and setting $v_{t} = \rho_t^* z_t^*$.
%where we denote $L_t := \sum_{s=1}^T \alpha_s y_s = \sum_{s=1}^T \alpha_s \nabla \tilde{f}(\tilde{x}_s)$.
Specifically,
Algorithm~\ref{alg:gaugeFW2} is an instance of Algorithm~\ref{alg:accLinear} with the regularization of \BTRL being the square of the gauge function, i.e. $\phi(\cdot) = \frac{1}{\lambda} \g^2(\cdot)$. 
%where the scaling of $\frac{1}{\lambda}$ ensures that 
%$\phi(\cdot)$ is 1-strongly convex, 
Hence,
the convergence rate can be shown by following the proof of Theorem~\ref{thm:acc_linear2}, where we bound the Lipschitzness of $\nabla \tilde{f}(\cdot)$ as
$ \displaystyle  \| \nabla \tilde{f}(\xof_t) - \nabla \tilde{f}(\xav_t) \| \leq \| \nabla f(\xof_t) -\frac{\mu}{\lambda} \nabla \phi( \xof_t ) - \nabla f(\xav_t) + \frac{\mu}{\lambda} \nabla \phi( \xav_t) \|_* 
\leq \displaystyle   \| \nabla f(\xof_t) - \nabla f(\xav_t) \|_* +  \mu \frac{ L_{\phi} }{\lambda} \| \xav_t - \xof_t \| \leq L\left( 1 + \frac{L_{\phi}}{\lambda}\right)\| \xav_t - \xof_t \| $.
\end{proof}

\subsection{Optimistic mirror descent, with weighted averaging} \label{subsec:OPTACC}

We give another accelerated algorithm for smooth convex optimization in the following. Algorithm~\ref{alg:OPTACC} is obtained when we swap the ordering of the player and let the $x$-player plays first using \OPTMD and then the $y$-player plays \BTL. 
%It is noted that the algorithm here only requires a single gradient evaluation in each iteration, while the extra-gradient methods in the literature, e.g. \cite{KLBC19,DO18,Nemi04} need computing two gradients at different points in every iteration. 
The reader might find some similarity of Algorithm~\ref{alg:OPTACC} and some single-call variants of the extra-gradient methods, see e.g., \cite{popov1980modification,hsieh2019convergence,chambolle2011first,gidel2018variational,daskalakis2017training,peng2020training}.
However, the gradient is taken at some weighted-average points in Algorithm~\ref{alg:OPTACC} and the resulting converge rate is an accelerated rate $O(1/T^2)$, which seems to be different from the single-call extra-gradient methods in the literature, to the best of our knowledge. We note that
the algorithm of Chambolle and Pock (Section 5 in \cite{chambolle2016ergodic}) has a $O(1/T^2)$ for a class of relevant problems with a single gradient call in each iteration, however, their algorithm
needs the distance generating function of the Bregman divergence to be the square $l_{2}$ norm and its update is different from Algorithm~\ref{alg:OPTACC}.

Theorem~\ref{thm:OPTACC} below shows an accelerated rate $O(1/T^2)$. 
Comparing Algorithm~\ref{alg:extra} 
and Algorithm~\ref{alg:OPTACC} here, one can see that the difference is that the $y$-player plays $\BTL$ in Algorithm~\ref{alg:OPTACC}
instead of $\BR$ in Algorithm~\ref{alg:extra}.
This is crucial to get the accelerated rate, as the negative distance terms in the regret bound of $\BTL$ cancel out that of $\OPTMD$.
%our proof below.
On the other hand, 
the regret bound of $\BR$ cannot cancel out the distance terms from that of $\OPTMD$,
which results in a slower $O(1/T)$ rate of Algorithm~\ref{alg:extra}.
%in the proof of the $O(1/T)$ rate in Theorem~\ref{thm:extra}, 

\begin{algorithm}[t] 
   \caption{ Optimistic mirror descent, with weighted averaging  } \label{alg:OPTACC}
Given: $L$-smooth $f(\cdot)$, a $1$-strongly convex $\phi(\cdot)$, iterations $T$ \\
Init: arbitrary $w_{-\frac{1}{2}} = w_{0}  = x_0 = x_{-\frac{1}{2} }\in \K \subseteq \reals^d  $.
\begin{center} 
\begin{tabular}{c c}
    $
      \boxed{\small
      \begin{array}{rl} 
\gamma  & \gets \frac{1}{L} \\
  w_t  & \gets  \displaystyle \argmin_{w \in \K} \left( \vphantom{\V{w_{t-\frac{1}{2}}}(w)}
   \alpha_{t} 
\langle w,  \nabla f(\bar{w}_{t-1}) \rangle \right.
 \\ &\qquad \qquad \quad
\left. + \frac{1}{\gamma} \V{w_{t-\frac{1}{2}}}(w) \right)
\\   w_{t+\frac{1}{2}} & \gets \displaystyle  \argmin_{w \in \K}  \left(  \vphantom{\V{w_{t-\frac{1}{2}}}(w)} 
\alpha_{t}
\langle w,  \nabla f(\bar{w}_t) \rangle \right. \\ &\qquad \qquad \quad + \left. \frac{1}{\gamma} \V{w_{t-\frac{1}{2}}}(w) \right)
      \end{array}
      }
    $
    & \small
    $\boxed{\small
    \begin{array}{rl}
      g(x,y) & := \langle x, y \rangle - f^*(y)\\
      \alpha_t & := t\text{ for } t=1, \ldots, T\\
      \alg^X & := \OPTMD[\phi(\cdot), x_0, \frac{1}{2L}] \\ %\text{, or, } 
      \alg^Y & := \BTL
%      \alg^X & := \BTRL[\frac 1 2 \| \cdot \|^2_2, \frac{1}{4L}] \\
    \end{array}
    }$
  \\
    \small Iterative Description & 
    \small FGNRD Equivalence
\end{tabular}
\end{center}
Output: $\bar w_T = \bar x_T$
\end{algorithm}

\begin{theorem} \label{thm:OPTACC}
%Assume the constraint set $\K$ has a diameter $R < \infty$.
The output $\bar w_T = \bar{x}_{T}$ of Algorithm~\ref{alg:OPTACC} satisfies 
\[
  f(\bar{w}_T) - \min_{w \in \K} f(w) =  \frac{2 L \V{w_0}(w^*) + 2 L \| w_0 - w_1   \|^2 }{T^2},
\]
where $w^{*}:= \argmin_{w \in \K} f(w)$.
\end{theorem}

\begin{proof}[Proof of Theorem~\ref{thm:OPTACC}]
The equivalence of the two displays, namely, for all $t$,
  \begin{eqnarray*}
     w_t  & =  & x_t \\
     w_{t-\frac{1}{2}} &= & x_{t-\frac{1}{2}} \\
     \nabla f(\bar{w}_{t})  & = & y_t 
  \end{eqnarray*}
can be trivially shown by induction. Specifically, the first two relations hold in the beginning by letting the initial point $w_{0} = w_{-\frac{1}{2}} = x_{0} = x_{-\frac{1}{2}} \in \K$. To show the last one, we apply the definition of \BTL,
\begin{equation*}
y_t \gets \argmax_{y \in \YY} \sum_{s=1}^t \alpha_s \ell_s(y) 
= \argmax_{y \in \YY} \sum_{s=1}^t \left( \alpha_s f^*(y) - \alpha_s \langle x_s, y \rangle \right)
= \nabla f(\bar{x}_t).
\end{equation*}
By induction, we have $w_{t}=x_{t}$. So we have $\bar{w}_{t}=\bar{x}_{t}$ and consequently $y_{t} = \nabla f(\bar{w}_t)$.

% $\nabla f(\bar{w}_{t})  =  y_t$,  

By summing the regret bound of each player, i.e. Lemma~\ref{lem:OPTMD} and Lemma~\ref{regret:BTL}, we get
\begin{equation*}
\begin{split}
& \avgregret{x}[\OPTMD] + \avgregret{y}[\BTL]
\\ & 
\overset{(a) }{\leq} \frac{1}{A_t}
\left \{   
\frac{1}{\gamma} \V{x_0}(x^*) 
+ \frac{\gamma}{2}   \sum_{t=1}^T \alpha_t^2 \| \nabla f (\bar{w}_t ) -   \nabla f (\bar{w}_{t-1} ) \|^2_*  - \sum_{t=1}^T \frac{A_{t-1}}{L} \| \nabla f(\bar{w}_{t}) - \nabla f(\bar{w}_{t-1})    \|^2_* \right \}
\\ & \overset{(b)}{\leq} \frac{ L \V{x_0}(x^*) + L \| \bar{w}_0 - \bar{w}_1  \|^2 }{ A_T } 
\leq \frac{2 L \V{x_0}(x^*) +2  L \| w_0 - w_1   \|^2 }{T^2}, 
\end{split}
\end{equation*}
where in (a) we used the relation $m_{t} = y_{t-1} = \nabla f(\bar{w}_{t-1})$ when
we applied Lemma~\ref{lem:OPTMD} and the fact that the $y$-player's loss function $\ell_t(y):= f^*(y) - \langle x , y_t \rangle$ is $\frac{1}{L}$-strongly convex w.r.t the dual norm $\| \cdot\|_{*}$ so that we have $\mu = \frac{1}{L}$ when we invoked Lemma~\ref{regret:BTL},
and in (b) we used that $\frac{ \gamma \alpha_{t}^{2} }{2} \leq \frac{A_{t-1}}{L}$ for $\gamma \leq \frac{1}{2L}$ so that the distance terms cancel out except the one at $t=1$.

\end{proof}

\begin{appendices}

\section{Proof of Theorem~\ref{th:cvx} } \label{app:show_cvx}

\begin{proof}
This is a result of the following lemmas.\\

\noindent
\textbf{Definition:} \textit{[Definition 12.1 in \cite{R98}] A mapping $T$ : $\mathcal{R}^n \rightarrow \mathcal{R}^n$ is called monotone if it has the property that }
\[
\langle v_1 - v_0, x_1 - x_0 \rangle \geq 0 \text{ whenever } v_0 \in T(x_0) , v_1 \in T(x_1).
\]
\textit{Moreover, $T$ is maximal monotone if there is no monotone operator that properly contains it.}\\
\noindent
\textbf{Lemma 2:} \textit{[Theorem 12.17 in \cite{R98}] For a proper, lower semi-continuous, convex function $f$, $\partial f$ is a maximum monotone operator.}\\
\noindent
\textbf{Lemma 3:} \textit{[Theorem 12.41 in \cite{R98}] For any maximal monotone mapping $T$, the set ``domain of $T$`` is nearly convex, in the sense that there is a convex set $C$ such that $ C \subset \text{domain of } $T$ \subset cl(C) $. The same applies to the range of $T$.}

Therefore, the closure of the union of the sets
$\bigcup_{x \in \XX} \partial f(x)$
is convex
%, because we can define another proper, lsc, convex function $\hat{f}(x)$ such that it is $\hat{f}(x) = f(x)$ if $x \in \XX$; otherwise, $\hat{f}(x) = \infty$. Then, the sub-differential of $\hat{f}(x)$ is equal to $ \partial f(x)$. So, we can apply 
by the above lemmas. 
%to get the result.
\end{proof}

\section{Proof of Lemma~\ref{lem:OMD} and~\ref{lem:OPTMD}} \label{app:lem:MD}

\begin{proof}[Proof of Lemma~\ref{lem:OMD}]
Let $\delta_{t} \in \partial \ell_t(z_t)$.
The per-round regret with respect to any comparator $z^* \in \ZZ$ can be bounded as
\begin{equation} \label{33}
\begin{split}
 \alpha_t \ell_t(z_t)  - \alpha_t \ell_t(z^*)
& \leq 
\langle \alpha_t \delta_t , z_t - z^* \rangle 
\\ & =
\langle \frac{1}{\gamma} \left( \nabla \phi(z_{t+1}) - \nabla \phi(z_t) \right), z^* - z_{t+1} \rangle
+ \langle \alpha_t \delta_t, z_t - z_{t+1} \rangle 
\\ & \quad +\langle \frac{1}{\gamma} \left( \nabla \phi(z_t) - \nabla \phi(z_{t+1} ) \right) - 
\alpha_t \delta_t, z^* - z_{t+1} \rangle   
\\ & \leq 
\langle \frac{1}{\gamma} \left( \nabla \phi(z_{t+1}) - \nabla \phi(z_t) \right), z^* - z_{t+1} \rangle
+ \langle \alpha_t \delta_t, z_t - z_{t+1} \rangle 
\\ & =
 \frac{1}{\gamma} \left( \V{z_t}(z^*) - \V{z_{t+1}}(z^*)
- \V{z_{t+1}}(z_t)
  \right) + \langle \alpha_t  \delta_t, z_t - z_{t+1} \rangle
\\ & \leq
 \frac{1}{\gamma} \left( \V{z_t}(z^*) - \V{z_{t+1}}(z^*) \right) + \frac{\gamma}{2 \beta}  \| \alpha_t \delta_t \|^2_*,
\end{split}
\end{equation}
where the second inequality uses the optimality of 
$z_{{t+1}}$ for 
$\min_{z \in \ZZ}   \alpha_{t} \ell_{t}(z) +  \frac{1}{\gamma} \V{z_{t}}(z)$
and the last one uses that
$\langle \alpha_t  \nabla \ell_t(z_t), z_t - z_{t+1} \rangle \leq \frac{\gamma}{2 \beta} \| \alpha_t \delta_t \|^2_* + \frac{\beta}{2 \gamma} \| z_t -z_{t+1}\|^2 $ and that $\phi(\cdot)$ is $\beta$ strongly convex.

Summing \eqref{33} from $t=1$ to $T$ leads to the bounds.

\end{proof}

\begin{proof}[Proof of Lemma~\ref{lem:OPTMD}]

The proof in the following is based on the proof in \cite{RS13}.
We let $\delta_{t} \in \partial \ell_t(z_t)$ in the following.

By applying Lemma~\ref{lem:newMD} to the update
$$z_t \leftarrow \argmin_{z \in \ZZ}   \left( \alpha_{t} 
\langle z,  m_t \rangle + \frac{1}{\gamma} \V{z_{t-\frac{1}{2}}}(z)\right),$$
we have
\begin{equation} \label{opt1}
\alpha_t \langle  m_t, z_t - z_{t+\frac{1}{2}} \rangle \leq \frac{1}{\gamma}
\left( \V{z_{t-\frac{1}{2}}}(z_{t+\frac{1}{2}}) - \V{z_t}(z_{t+\frac{1}{2}}) - \V{z_{t-\frac{1}{2}}}(z_{t})  \right).
\end{equation}
Also, by applying Lemma~\ref{lem:newMD} to 
$z_{t+\frac{1}{2}} \leftarrow   \argmin_{z \in \ZZ}   \left( \alpha_{t}
\langle z,  \delta_t \rangle +  \frac{1}{\gamma} \V{z_{t-\frac{1}{2}}}(z)\right)$, we get
\begin{equation} \label{opt2}
\alpha_t \langle  \nabla \ell_t(z_t)  ,z_{t+\frac{1}{2}} - z^* \rangle \leq \frac{1}{\gamma}
\left( \V{z_{t-\frac{1}{2}}}(z_{*}) - \V{z_{t+\frac{1}{2}}}(z^*) - \V{z_{t-\frac{1}{2}}}(z_{t+\frac{1}{2}})  \right).
\end{equation}
Then, we can bound the per-round regret with respect to the comparator
$z^*$ as follows:
\begin{equation*}
\begin{split}
& \alpha_t \ell_t(z_t)  - \alpha_t \ell_t(z^*)
 \leq \alpha_t \langle \delta_t, z_t - z^* \rangle  
\\ & = \alpha_t \langle \delta_t, z_{t+\frac{1}{2}} - z^* \rangle
+ \alpha_t \langle m_t, z_{t} - z_{t+\frac{1}{2}} \rangle
+ \alpha_t \langle \delta_t - m_t , z_t - z_{t+\frac{1}{2}} \rangle
\\ &
\leq 
\frac{1}{\gamma}
\left( \V{z_{t-\frac{1}{2}}}(z_{*}) - \V{z_{t+\frac{1}{2}}}(z^*)
- \V{z_t}(z_{t+\frac{1}{2}}) - \V{z_{t-\frac{1}{2}}}(z_{t}) 
   \right)
 + \alpha_t \|  \delta_t - m_t \|_* \| z_t - z_{t+\frac{1}{2}} \|  
\\ &
\leq 
\frac{1}{\gamma}
\left( \V{z_{t-\frac{1}{2}}}(z_{*}) - \V{z_{t+\frac{1}{2}}}(z^*)
- \frac{\beta}{2} \| z_t - z_{t+\frac{1}{2}} \|^2 - \frac{\beta}{2} 
\| z_t - z_{t-\frac{1}{2}} \|^2
   \right)
\\ &   \quad
 + \frac{ \alpha_t^2 \gamma }{2 \beta} \|  \delta_t - m_t \|_*^2 + \frac{\beta}{2 \gamma} \| z_t - z_{t+\frac{1}{2}} \|^2,
\end{split}
\end{equation*}
where the second inequality is due to \eqref{opt1} and \eqref{opt2}
and the last one is by the $\beta$-strong convexity of $\phi(\cdot)$.

Summing the above from $t=1$ to $t=T$
and noting that $ z_{-\frac{1}{2}}= z_0  \in \K$, we obtain the bound.

\end{proof}

\section{Proof of Theorem~\ref{thm:linearFW} } \label{app:linearFW}

Note that in the algorithm that we describe below the weights $\alpha_t$ are not predefined but rather depend on the queries of the algorithm. These adaptive weights are explicitly defined in Algorithm~\ref{alg:SC-AFTL} which is used by the $y$-player.
Note that Algorithm~\ref{alg:SC-AFTL} is equivalent to performing FTL updates over the following loss sequence:
 $\left\{\tilde{\ell}_t(y) :=\alpha_t \ell_t(y) \right\}_{t=1}^T.$ 
The $x$-player plays best response, which only involves the linear optimization oracle.

\begin{algorithm}[h] 
  \caption{\small Adaptive Follow-the-Leader (AFTL)}
  \label{alg:SC-AFTL}
  \begin{algorithmic}[1]
  \For{$t= 1, 2, \dots, T$}
    \State Play $y_t \in \YY$
    \State Receive a strongly convex loss function $\alpha_t \ell_{t}(\cdot)$ with $\alpha_t = \frac{1}{\| \nabla \ell_t(y_t) \|^2} $.
    \State Update $y_{t+1} = \min_{y\in \YY} \sum_{s=1}^t \alpha_y \ell_s(y) $  
   \EndFor 
  \end{algorithmic}
\end{algorithm}
%Before proving the theorem,
%let us focus on the strategy of the $y$-player which we describe in Algorithm~\ref{alg:SC-AFTL}.

\begin{proof}
We can easily show the equivalence of the two displays shown on Algorithm~\ref{alg:Ada-fw} by using the induction and noting that
$\nabla \ell_t(y_t) = x_t - \nabla f^*(y_t)  =  x_t - \bar{x}_{t-1}$, where in the last equality we used that $y_t = \nabla f(\bar{x}_{t-1})$ if and only if 
$\bar{x}_{t-1} = \nabla f^* (y_t)$. 

Now we switch to show the convergence rate.
Since the $x$-player plays best response, $\avgregret{x}[\BR]=0$, we only need to show that 
%the $y$-player's regret satisfies
$\avgregret{y}[\textsc{SC-AFTL}] = O(\exp(-\frac{\lambda B}{L} T))$, which we do next.

We start by defining a function $s(y) := \max_{x \in \XX} - x^\top y + f^*(y)$, which is a strongly convex function.
We are going to show that $s(\cdot)$ is also smooth.
We have 
\begin{equation*}
\begin{aligned}
& \| \nabla_w s(\cdot) - \nabla_z s(\cdot) \| = 
\| \argmax_{x \in \XX} ( - w^\top w + f^*(w) ) -  \argmax_{x \in \XX} ( - z^\top x + f^*(z) ) \|
\\ &= \| \argmax_{x \in \XX} ( - w^\top x) - (\argmax_{x' \in \XX}  - z^\top x') \|
\leq  \frac{ 2 \| w - z \|}{ \lambda ( \| w\| + \| z\| ) }
\leq  \frac{ \| w - z \|}{ \lambda B},
\end{aligned}
\end{equation*}
where the second to last inequality uses Lemma~\ref{lm:lip} regarding $\lambda$-strongly convex sets,
and the last inequality is by assuming the gradient of $\| \nabla f(\cdot)\| \geq B$
and the fact that $w,z \in \YY$ are gradients of $f$.
This shows that $s(\cdot)$ is a smooth function with smoothness constant $L':=\frac{1}{\lambda B}$.
\begin{equation*}
\begin{aligned}
T &= \sum_{t=1}^T \frac{\| \nabla \ell_t(y_t)^2 \|}{\| \nabla \ell_t(y_t)\|^2}
\overset{Proposition~\ref{sameGrad}}{ = }\sum_{t=1}^T \frac{\| \nabla s(y_t)\|^2}{\| \nabla \ell_t(y_t) \|^2}
\overset{Lemma~\ref{lem:smooth}}{ \leq} \sum_{t=1}^T \frac{L'}{\| \nabla \ell_t(y_t)\| ^2} (s(y_t) - s(y^*))\\
&\le \sum_{t=1}^T \frac{L'}{\| \nabla \ell_t(x_t)\|^2}  (\ell_t(y_t) - \ell_t(y^*))\label{eq:normal},
\end{aligned}
\end{equation*}
where we denote $y^* := \argmin_{y \in \YY} s(y) $ and the last inequality follows from the fact that $s(y_t) := \ell_t(y_t)$ and $\ell_t(y) = - g(x_t,y) \leq  - g(x_y , y) = s(y)$ for any $y$. %We can apply Lemma~\ref{lem:GSmooth} because 
%$\argmin_{x \in \reals^d} \{ \max_{y \in \YY} - x^\top y + f^*(x) \} \in \XX$, as $\XX$ is the gradient space. 

In the following, we will denote $c$ a constant such that $\| \nabla \ell_t(y_t) \| = \| x_t - \nabla f^*(y_t) \| = \| x_t - \bar{x}_{t-1} \|\leq c$.
We have
\begin{equation} \label{eq:ExpRate}
\begin{aligned}
T & \leq
\sum_{t=1}^T \frac{{ L' }}{\| \ell_t( y_t) \|^2} ( \ell_t(y_t)-\ell_t(y^*) )  \nonumber \\
&\overset{(a)}{=}
\sum_{t=1}^T L' ( \tilde{\ell}_t(y_t)-\tilde{\ell}_t(y^*) )  \nonumber \\
&\overset{(b)}{\le}
\frac{L\cdot L'}{2} \sum_{t=1}^T \frac{\| \nabla \ell_t( y_t) \|^{-2}}{\sum_{s=1}^t \|\nabla \ell_s(y_t) \|^{-2}}\nonumber \\
&\overset{(c)}{\le}
\frac{L\cdot L'}{2}\left( 1+\log(c^2 \sum_{t=1}^T\|\nabla \ell_t( y_t)\|^{-2}) \right)~,
\end{aligned}
\end{equation}
where (a) is by the definition of $\tilde{\ell}_t(\cdot)$,
and (b) is shown using Lemma~\ref{regret:FTL} with strong convexity parameter of $\ell_t(\cdot)$ being $\frac{1}{L}$, and (c) is by Lemma~\ref{lem:Log_sum} so that
\begin{align*}
\sum_{t=1}^T \frac{\| \ell_t( y_t) \|^{-2}}{\sum_{s=1}^t \| \ell_s( y_s) \|^{-2}}
&=
\sum_{t=1}^T \frac{c^2\| \ell_t( y_t) \|^{-2}}{\sum_{s=1}^t c^2\| \ell_s( y_s) \|^{-2}}
\leq{}
{1+\log(c^2 \sum_{t=1}^T\| \ell_t( y_t)\|^{-2})}.
\end{align*}
Thus, we get
\begin{align} \label{eq:exp}
c^2 \sum_{t=1}^T \| \nabla \ell_t(y_t) \|^{-2} = O(  e^{\frac{1}{L\cdot L'}T}) =  O(  e^{\frac{\lambda B}{L}T}).
\end{align}
So
%Here, we consider the weight to be
%$\alpha_t = \| \nabla \ell_t(x_t)^{-2} \|$. 
\begin{equation*}
\begin{aligned}
& \frac{ \regret{y} }{  A_T } :=
\frac{\sum_{t=1}^T \alpha_t(\ell_t (y_t) - \ell_t(y^*))}{A_T} 
\leq \frac{L}{2 A_T} \sum_{t=1}^T \frac{\|\nabla \ell_t(y_t)\|^{-2}}{\sum_{\tau = 1}^t \| \nabla \ell_t(y_\tau)\|^{-2}}\\
&\overset{(a)}{\le} \frac{L c^2\pr{1 + \log \pr{ c^2 \sum_{t=1}^T \| \nabla \ell_t(y_t)\|^{-2}}}}{2 c^2 \sum_{t=1}^{T} \|\nabla \ell_t(y_t)\|^{-2}}
\overset{(b)}{\le} O( \frac{ L c^2\pr{1 + \pr{\frac{\lambda B T}{L}}}}{ e^{\frac{\lambda B}{L}T}} ) = O\pr{ L c^2 e^{-\frac{\lambda B}{L}T}},
\end{aligned}
\end{equation*}
where $(a)$ is by Lemma~\ref{lem:Log_sum}, $(b)$ is by (\ref{eq:exp}) and the fact that $\frac{1+\log z}{z}$ is monotonically decreasing for $z \ge 1$. 
Thus, we have
\[
    f(w_T) - \min_{w \in \K} f(w) \leq \avgregret{x}[\BR] + \avgregret{y}[\textsc{AFTL}] = O\pr{ L c^2 e^{-\frac{\lambda B}{L}T}}.
  \]
This completes the proof.
\end{proof}

\begin{proposition} \label{sameGrad}
For arbitrary $y$, let $\ell(\cdot) := - g(x_y, \cdot )$. Then $- \nabla_{y} \ell(\cdot) \in \partial_{y} s(\cdot)$, where $x_y$ means that the $x$-player plays $x$ by \BR after observing the $y$-player plays $y$.
\end{proposition}
\begin{proof}
Consider any point $w \in \YY$,
\begin{equation*}
\begin{aligned}
 s(w) - s(y) & = g( x_y, y) - g(x_w,w)   
\\ & = g(x_y, y) - g(x_y, w) + g(x_y, w) -  g(x_w, w)  
\geq   g(x_y, y) - g(x_y, w) + 0
\\ & \geq  \langle g_y( x_y, y) , w - y\rangle = \langle - \nabla_{y} \ell(y) , w - y \rangle,
\end{aligned}
\end{equation*}
where the first inequality is because that $x_w$ is the best response to $w$, the second inequality is due to the concavity of $g(x_y, \cdot)$.
The overall statement implies that $-\nabla_{y} \ell(y)$ is a subgradient of $s$ at $y$.
\end{proof}

\iffalse
\begin{lemma} \label{lem:GSmooth}
For any $L$-smooth convex function $\ell(\cdot):\reals^d \mapsto \reals$, if $x^* =\argmin_{x\in \reals^d} \ell(x)$, then 
$$ \| \nabla \ell(x)\|^2 \le 2 L \left( \ell(x) - \ell(x^*)\right), \quad \forall x\in \reals^d~.$$
\end{lemma} 
\fi

\begin{lemma}  \label{lm:lip}
\footnote{\cite{P96} discuss the smoothness of the support function on strongly convex sets.
Here, we state a more general result.}
Denote  
$x_p = \argmax_{x \in \K} \langle p, x\rangle $ and $x_q = \argmax_{x \in \K} \langle q, x\rangle $, where $p,q \in \reals^d$ are any nonzero vectors.  
If a compact set $\K$ is a $\lambda$-strongly convex set,
then 
\begin{equation*}
    \| x_p - x_q \| \leq \frac{2 \|p - q\|}{\lambda ( \| p \| + \|q \| )}.
\end{equation*}
\end{lemma}

\begin{proof}
%Recall 
\cite{P96}
show that a strongly convex set $\K$ can be written as intersection of some Euclidean balls. %(c.f. definition 3 in Section~\ref{Pre}).
Namely,
    \[ \K = \underset{u: \| u\|_2 = 1}{\cap} B_{\frac{1}{\lambda}} \left( x_u - \frac{u}{\lambda} \right) ,\]
where $x_u$ is defined as $x_u = \argmax_{x \in \K} \langle \frac{u}{\|u\|}, x\rangle$.

    Let  ${x_p = \argmax_{x \in \K} \langle \frac{p}{\|p\|}, x\rangle }$ and ${x_q = \argmax_{x \in \K} \langle \frac{q}{\|q\|}, x \rangle}$.
    Based on the definition of strongly convex sets, we can see that
    $x_q \in B_{ \frac{1}{\lambda}  } ( x_p - \frac{p}{\lambda \| p \|})$ and $x_p \in B_{\frac{1}{\lambda}  } ( x_q - \frac{q}{\lambda \| q \|} )$.
    Therefore, 
    \[
        \left \| x_q - x_p -  \frac{p}{\lambda \| p\|} \right \|^2 \leq \frac{1}{\lambda^2},
    \]
    which leads to
    \begin{equation}
        \label{eqn:ineqSumP1}
     \|p \| \cdot  \| x_p - x_q \|^2 \leq \frac{2}{\lambda} \langle x_p - x_q,  p \rangle.
    \end{equation}
    Similarly,
    \[
        \left \| x_p - x_q - \frac{q}{\lambda \|q\|} \right \|^2 \leq \frac{1}{\lambda^2}, 
    \]
    which results in
    \begin{equation}
        \label{eqn:ineqSumP2}
      \| q \| \cdot  \| x_p - x_q \|^2 \leq \frac{2}{\lambda} \langle x_q - x_p,  q \rangle.
    \end{equation}
    Summing (\ref{eqn:ineqSumP1}) and (\ref{eqn:ineqSumP2}), one gets
    $(\|p\| + \| q \|) \| x_p - x_q \|^2 \leq \frac{2}{\lambda} \langle x_p - x_q, p-q \rangle$.
    Applying the Cauchy-Schwarz inequality completes the proof.
\end{proof}

\begin{lemma} (\cite{L17}) \label{lem:Log_sum} 
For any non-negative real numbers $a_1,\ldots, a_n\geq 1$,
\begin{align*}
\sum_{i=1}^n \frac{a_i}{\sum_{j=1}^i a_j} 
\le 
1+\log\left( \sum_{i=1}^n a_i\right) ~.
\end{align*}
\end{lemma}
\section{Proof of Theorem~\ref{thm:equivStoFW} } \label{app:equivStoFW}

\begin{proof}
The equivalence of the updates follows the proof of Theorem~\ref{thm:equivFW}. Specifically, 
 we have that the objects on the left in the following equalities  correspond to Alg.~\ref{alg:game} and those on the right to Alg.~\ref{alg:newStoFW}.
  \begin{eqnarray*}
%    y_t  & = & \nabla f(w_{t-1}) \\
    x_t  & = & v_t    \\
    \bar{x}_t  & = & w_t .
  \end{eqnarray*}
To analyze the regret of the $y$-player, we define $\{ \hat{y}_t \}$ as the points if the $y$-player would have played \FTL, which is
\begin{equation*}
\begin{split}
\hat{y}_t &:= \argmin_{y \in \YY} \frac{1}{t-1} \sum_{s=1}^{t-1} \ell_t(y)
= \argmax_{y \in \YY} \frac{1}{t-1} \sum_{s=1}^{t-1} \langle x_s, y \rangle - f^*(y)
\\ & = \nabla f(\bar{x}_t)
 = \frac{1}{n} \sum_{i=1}^n \nabla f_i(\bar{x}_t).
\end{split}
\end{equation*}
On the other hand, we call the actual strategy used by the $y$-player \textsc{LazyFTL} as it only compute the gradient of a single component.
%In the following, denote $\nu:= L_0 + D$. 
We can show the average regret of \textsc{LazyFTL} as follows.
\begin{equation*}
\begin{split}
\avgregret{y} & = \frac{1}{T} \sum_{t=1}^T  \left( \ell_t(\hat{y}_t) - \ell_t(y_*) \right) + \frac{1}{T}  \sum_{t=1}^T \left( \ell_t(y_t) - \ell_t(\hat{y}_t) \right)
\\ & \overset{(a)}{\leq} 
\frac{4 L R \log T}{T} + \frac{1}{T}  \sum_{t=1}^T \left( \ell_t(y_t) - \ell_t(\hat{y}_t) \right)
\\ & = \frac{4 L R \log T}{T} + \frac{1}{T} 
\sum_{t=1}^T  \left( f^*(y_t) - f^*(\hat{y}_t) +  \langle x_t, \hat{y}_t - y_t \rangle \right) 
\\ & \overset{(b)}{\leq} \frac{4 L R \log T}{T} +
\sum_{t=1}^T \frac{1}{T} (L_0 + r) \| y_t - \hat{y}_t \|
\\ & = \frac{4 L R \log T}{T} + \frac{1}{T} (L_0 + r)
\sum_{t=1}^T \left \| \frac{1}{n} \sum_{i=1}^n g_{i,t} - \frac{1}{n} \sum_{i=1}^n \nabla f_i(\bar{x}_t) \right \|
\\ & = \frac{4 L R \log T}{T} + \frac{1}{T} (L_0 + r)
\sum_{t=1}^T \left \| \frac{1}{n} \sum_{i \neq i_t}^n \big( g_{i,t}  - \nabla f_i(\bar{x}_t) \big)\right \|
\\ & \overset{(c)}{\leq} \frac{4 L R \log T}{T} + \frac{L(L_0 + r)}{Tn} 
\sum_{t=1}^T \sum_{i \neq i_t}^n \| \bar{x}_{\tau_t(i)}  - \bar{x}_t \|
\\ & \overset{(d)}{\leq} \frac{4 L R \log T}{T} + \frac{L(L_0 + r)}{Tn} 
\sum_{t=1}^T \sum_{i \neq i_t}^n \frac{2nr}{t}
% \leq 
%\frac{4 L D \log T}{T} +  \frac{L n (L_0 \sqrt{D} + D) \log T }{T}
\\ & = O\left( \frac{ \max\{ LR , L (L_0+r) n r \} \log T}{T} \right), 
\end{split}
\end{equation*}
where (a) is due to the regret of \FTL (Lemma~\ref{regret:FTL}),
\[
\begin{aligned}
\frac{1}{T} \big( \sum_{t=1}^T \ell_t(\hat{y}_t) - \ell_t(y_*) \big)
%leq
%    \avgregret{\hat{y}}[\FTL] 
      & \leq \frac{1}{T}  \sum_{t=1}^T \frac{2  \| \nabla \ell_t(\hat{y}_t) \|^2}{\sum_{s=1}^{t}  (1/L)} 
   = \frac{4LR \log T}{T},
\end{aligned}
\]
and that $\| \nabla \ell_t(\hat{y}_t) \|^2 = \| x_t - \nabla f^*(\hat{y}_t) \|^2 = \| x_t - \bar{x}_{t-1} \|^2 \leq R$,
%, i.e. \eqref{eq:FWy1} in Theorem~\ref{thm:fwconvergence}, 
(b) we assume that the conjugate is $L_0$-Lipschitz and that $\max_{x \in \K} \| x \| \leq r$,
(c) we denote $\tau_t(i) \in [T]$ as the last iteration that $i_{th}$ sample's gradient is computed at $t$, and (d) is because that
\begin{equation*}
\begin{split}
\| \bar{x}_{\tau_t(i)} - \bar{x}_t \| & =
\left \| \frac{1}{\tau_t(i)}  \sum_{s=1}^{\tau_t(i)} x_s - \frac{1}{t}  \sum_{s=1}^t x_s \right \| \leq \left \| \sum_{s=1}^{\tau_t(i)} x_s \left( \frac{1}{\tau_t(i)} - \frac{1}{t} \right) \right \|
+ \left \| \frac{1}{t} \sum_{s=\tau_t(i) + 1}^{t} x_s \right \|
\\ & 
= \frac{t - \tau_t(i) }{t} \| \bar{x}_{\tau_t(i)} \|
+ \left \| \frac{1}{t} \sum_{s=\tau_t(i) + 1}^{t} x_s \right \|
\\ & 
\leq \frac{n r }{t} 
+ \left \| \frac{1}{t} \sum_{s=\tau_t(i) + 1}^{t} x_s \right \| = 
\frac{n r }{t} 
+  \frac{t - \tau_t(i)}{t} \left \| \frac{1}{t - \tau_t(i)} \sum_{s=\tau_t(i) + 1}^{t} x_s\right \|
\\ &
\leq \frac{2 n r}{t}.
\end{split}
\end{equation*}
For the $x$-player, since it plays \BR, the regret is non-positive, 
i.e. $\avgregret{x}[\BR] = 0$. By combining the regrets of both players, we have
\[
\begin{aligned}
    f(w_T) - \min_{w \in \K} f(w) & \leq \avgregret{x}[\BR] + \avgregret{y}[\textsc{LazyFTL}] \\ & =O\left( \frac{ \max\{ LR , L (L_0+r) n r \} \log T}{T} \right).
\end{aligned}    
  \]
\end{proof}

\section{Proof of Theorem~\ref{thm:Heavy}} \label{app:thm:Heavy}

%\textbf{Theorem~\ref{thm:Heavy}}
%\textit{
%Let $\alpha_{t}=t$. Assume $\K = \reals^{d}$. Also, let $\gamma =O(\frac{1}{L})$.
%The output $\bar{x}_{T}$ of Algorithm~\ref{alg:heavyball} is an $O(\frac{1}{T})$-approximate optimal solution of $\min_{x} f(x)$.
%}
%\junkun{We can consider $\alpha_t=1$.}
\begin{proof}
First, we can bound the norm of the gradient as
  \[
    \| \nabla \ell_t(y_t) \|^2 = \| x_t - \nabla f^*(y_t) \|^2 = \| x_t - \bar{x}_{t-1} \|^2. %\leq D.
  \]
  Combining this with Lemma~\ref{regret:FTL} we see that
  \begin{align*} 
    \avgregret{y}[\FTL] 
      & \leq \frac{1}{A_T}  \sum_{t=1}^T \frac{2 \alpha_t^2 \| \nabla \ell_t(y_t) \|^2}{\sum_{s=1}^{t} \alpha_s (1/L)} 
  = \frac{1}{A_T}  \sum_{t=1}^T \frac{2 \alpha_t^2 \| x_t - \bar{x}_{t-1} \|^2}{\sum_{s=1}^{t} \alpha_s (1/L)}.
%  = O( \sum_{\tau=1}^T \frac{ L \| \bar{x}_{t-1} - x_t  \|^2 }{A_T}).
   % = \frac{8L}{T(T+1)}
  %\sum_{t=1}^T \frac{t^2 D}{t(t+1)} \leq \frac{8LD}{T+1}.
  \end{align*}
On the other hand, the $x$-player plays $\MD$, according to Lemma~\ref{lem:MD},
its regret satisfies
\begin{equation*}
\avgregret{x}  
\leq  \frac{ \frac{1}{\gamma} \V{x_0}(x^*) - \sum_{t=1}^{T}  4L  \| x_{t-1} - x_t \|^2 }{A_T}.
\end{equation*}
By adding the average regret of both players, we get
\begin{equation*}
\begin{split}
 &   \avgregret{y}[\FTL]  + \avgregret{x}[\MD]
    \leq 
\\ & \qquad \qquad \qquad \qquad \qquad
    \frac{ 4 L \| x_0 - x^*\|^2 + \sum_{t=1}^T 4 L \left( \| \bar{x}_{t-1} - x_t  \|^2 -  \| x_{t-1} - x_t \|^2 \right)  }{A_T},
\end{split}
\end{equation*}
where we used that $\V{x_0}(x^*) = \frac{1}{2}\|x_0 - x^*\|^{2}$.
Since the distance terms may not cancel out, one can only bound the differences of the distance terms by a constant, which leads to the non-accelerated $O(1/T)$ rate.
\iffalse
\junkun{
\noindent
\textbf{Continue:}
We can recursively expand 
$\bar{x}_{t-1} - x_t$ as follows.
\begin{equation}
\begin{split}
\bar{x}_{t-1} - x_t
& = x_{t-1} - x_t + \bar{x}_{t-1} - x_{t-1}
\\ & = x_{t-1} - x_t + \left( 1 - \frac{\alpha_{t-1}}{A_{t-1}} \right) \left( \bar{x}_{t-2} - x_{t-1} \right)
\\ & = x_{t-1} - x_t + \left( 1 - \frac{\alpha_{t-1}}{A_{t-1}} \right) \left( 
\left( x_{t-2} -x_{t-1} + \left( 1 - \frac{\alpha_{t-2}}{A_{t-2}} \right) ( \bar{x}_{t-3} - x_{t-2})   \right)
 \right)
\\ & = \sum_{s=1}^t \frac{A_{s-1}}{A_{t-1}} \left( x_{s-1} - x_s \right),
\end{split}
\end{equation}
where we used the notation that $\frac{A_{0}}{A_{0}}:= \frac{0}{0} = 1$.
In the case of $\alpha_t = 1$, the above reduces to
\begin{equation}
\bar{x}_{t-1} - x_t = \sum_{s=1}^t \frac{s-1}{t-1} \left( x_{s-1} - x_s \right).
\end{equation}
Denote $B_{t} = \sum_{{s=1}}^{t} (s-1)$.
So we can bound $\bar{x}_{t-1} - x_t$ as
\begin{equation}
\begin{split}
\| \bar{x}_{t-1} - x_t \|^2 & = \| \sum_{s=1}^t \frac{s-1}{t-1} \left( x_{s-1} - x_s \right) \|^2 = \frac{B_t^2}{(t-1)^2} \| \sum_{s=1}^t \frac{s-1}{B_{t}} \left( x_{s-1} - x_s \right) \|^2
\\ & 
\leq
 \frac{B_t}{(t-1)^2} \sum_{s=1}^t (s-1) \|  x_{s-1} - x_s  \|^2
\end{split}
\end{equation}
The regret does not cancel out...
}
\fi
\end{proof}

\section{Proof of Theorem~\ref{thm:smoothAcc}} \label{app:thm:smoothAcc}

\begin{proof}
%The equivalence $\bar x_T \equiv w_T$ was established in Proposition~\ref{thm:Nes_constrained}. We may now appeal to Theorem~\ref{thm:meta} to prove this result, which gives us that

The proof is similar to that of Theorem~\ref{thm:metaAcc}. The only difference is that
the $x$-player is \BTRL instead of \MD. 
By using a bound on $\regret{y}$ in Lemma~\ref{lem:yregretbound}
and the bound in
Lemma~\ref{regret:BTRL} of \BTRL with $\mu=0$ (as the $x$-player sees linear loss functions), we have
\begin{equation} \label{eq:genericbound2}
\begin{split}
  f(w_T) - \min_{w \in \K} f(w) 
&  \leq\frac{1}{A_T}(\regret{x}[\OFTL] + \regret{y}[\BTRL])
\\ & 
  \leq \frac 1 {A_T}
\left( \frac{ R(x^*) - \min_{x \in \K} R(x)}{\eta }  
     + \sum_{t=1}^{T} \left(\frac{\alpha_t^2}{A_t} L - \frac{\beta}{2 \eta} \right)
          \| x_{t-1} - x_t \|^2    \right).
\end{split}
\end{equation}
Choosing $\eta = \frac 1 {4L}$, $\beta = 1$,
and the weight
 $\alpha_t = t$, we have $A_t := \frac{t(t+1)}{2}$ and therefore $\frac{\alpha_t^2}{A_t} = \frac{2t^2}{t(t+1)} \leq 2$. With this in mind, the sum on the right hand side of \eqref{eq:genericbound2} is non-positive. Noticing that $\frac{1}{A_T} \leq \frac 2 {T^2}$ completes the proof. 
% The choice of $\{\alpha_t,\gamma\}$ implies $\frac{D}{\gamma} \leq CLD$ and $\frac{L\alpha_t^2}{A_t} = \frac{2Lt^2}{t(t+1)} \leq 2L \leq \frac 1 {2 \gamma}$, which ensures that the summation term in \eqref{eq:genericbound} is negative. The rest is simple algebra.

% Similar calculations can be done for the bound \eqref{eq:genericbound2}, and hence omitted.

% We have already done the hard work to prove this theorem. Lemma~\ref{lem:fenchelgame} tells us we can bound the optimization error of $\xav_T$ by the $\epsilon$ error of the approximate equilibrium $(\xav_T, \yav_T)$. Theorem~\ref{thm:meta} tells us that the pair $(\xav_T, \yav_T)$ derived from Algorithm~\ref{alg:game} is controlled by the sum of averaged regrets of both players, $\frac{1}{A_T}(\regret{x}[\OFTL] + \regret{y}[\MD])$. But we now have control over both of these two regret quantities, from Lemmas~\ref{lem:yregretbound} of \OFTL and~\ref{lem:MD} of \MD. 
% \begin{equation}
% \displaystyle  f(\xav_T) - \min_{x \in \XX} f(x)
%   \leq \frac 1 {A_T} \left( \frac D {\gamma} 
%     + \sum_{t=1}^{T} \left(\frac{\alpha_t^2}{A_t} L - \frac{1}{2 \gamma} \right)
%          \| x_{t-1} - x_t \|^2    \right).
% \end{equation}
%  The right hand side of \eqref{eq:genericbound} is the sum of these bounds.
% blah
\end{proof}

\section{Examples of strongly convex sets}\label{app:betagauge}
\cite{D15} lists three known classes of $\lambda$-strongly convex sets as follows,
which are all gauge sets. 
\begin{enumerate}
  \item{$\ell_p$ balls: $\| x \|_p \leq r, \forall p \in (1,2]$. The strong convexity of the set is $\lambda=\frac{p-1}{r}$ and its sqaure of gauge function is $\frac{1}{r^2} \| x \|_{p}^{2}$, which is a $\frac{2(p-1)}{r^2}$-strongly convex function with respect to norm $\| \cdot \|_p$ by Lemma 4 in \cite{D15}. }
  \item{Schatten $p$ balls: $\| \sigma(X) \|_p \leq r$ for $p \in (1,2]$, where $\sigma(X)$ is the vector consisting of singular values of the matrix $X$. 
    The strong convexity of the set is $\lambda=\frac{p-1}{r}$
    and the sqaure gauge function is $\frac{1}{r^2} \| \sigma(X) \|_{p}^{2}$, which is a $\frac{2(p-1)}{r^2}$-strongly convex function with respect to norm $\| \sigma(\cdot) \|_p $ by Lemma 6 in \cite{D15}.}
  \item{  Group (s,p) balls: $\| X \|_{s,p}  = \| (\| X_1\|_s, \| X_2\|_s, \dots, \| X_m\|_s)  \|_p \leq r$
    where $X \in \reals^{{m \times n}}$, $X_{j}$ represents the $j$-th row of $X$, and $s,p \in (1,2]$.
    The strong convexity of the set is $\lambda=\frac{2(s-1)(p-1)}{r( (s-1)+(p-1) ) }$ and its sqaure gauge function is $\frac{1}{r^2} \| X \|_{s,p}^{2}$, which is a $\frac{2(s-1)(p-1)}{r^2( (s-1)+(p-1) ) }$-strongly convex function with respect to norm $\| \cdot\|_{s,p}$ by Lemma 8 in \cite{D15}.}
\end{enumerate}

\end{appendices}


\bibliography{Fenchel_Journal}% common bib file
%\bibliographystyle{unsrtnat}


\end{document}